\documentclass[11pt]{article}

\usepackage{graphicx}
\usepackage{subfigure} 
\usepackage{amssymb}
\usepackage{pifont}% http://ctan.org/pkg/pifont

\usepackage{mathrsfs}

\usepackage{fullpage}
\usepackage[protrusion=true,expansion=true]{microtype}

\usepackage{amsmath}
\usepackage{amsthm}

\usepackage{algorithm}
\usepackage{algorithmic}

\DeclareMathOperator*{\argmin}{arg\,min}

\usepackage{booktabs}       % professional-quality tables

\DeclareMathOperator*{\E}{\mathbb{E}}
\DeclareMathOperator*{\W}{\mathbb{W}}

\usepackage{multirow} 
\usepackage[table,dvipsnames]{xcolor}
\usepackage{colortbl}

\usepackage{times}
\usepackage{natbib}
\usepackage{bm}

\usepackage[colorlinks = true,
            linkcolor = blue,
            urlcolor  = blue,
            citecolor = blue,
            anchorcolor = blue]{hyperref}
\usepackage{cleveref}

\newtheorem{theorem}{Theorem}

\newtheorem{assum}{Assumption}
\newtheorem{lemma}{Lemma}
\newtheorem{proposition}{Proposition}

\newtheorem{corollary}{Corollary}
\newtheorem{definition}{Definition}

\allowdisplaybreaks

\title{Communication-Efficient Federated Hypergradient Computation via Aggregated Iterative Differentiation}

\author
{
Peiyao Xiao and Kaiyi Ji
\vspace{0.3cm}\\ Department of Computer Science and Engineering \vspace{0.15cm}\\ University at Buffalo\vspace{0.15cm}\\{\tt \{peiyaoxi,kaiyiji\}@buffalo.edu}
}

\begin{document}

\maketitle	

\begin{abstract}%
Federated bilevel optimization has attracted increasing attention due to emerging machine learning and communication applications. The biggest challenge lies in computing the gradient of the upper-level objective function (i.e., hypergradient) in the federated setting due to the nonlinear and distributed construction of a series of global Hessian matrices. In this paper, we propose a novel communication-efficient federated hypergradient estimator via aggregated iterative differentiation (AggITD). AggITD is simple to implement and significantly reduces the communication cost by conducting the federated hypergradient estimation and the lower-level optimization simultaneously. We show that the proposed AggITD-based algorithm achieves the same sample complexity as existing approximate implicit differentiation (AID)-based approaches with much fewer communication rounds in the presence of data heterogeneity. Our results also shed light on the great advantage of ITD over AID in the federated/distributed hypergradient estimation. This differs from the comparison in the non-distributed bilevel optimization, where ITD is less efficient than AID. Our extensive experiments demonstrate the great effectiveness and communication efficiency of the proposed method. 
\end{abstract}

\section{Introduction}
Bilevel optimization has drawn significant attention from the machine learning (ML) community due to its wide applications in ML including meta-learning~\citep{finn2017model,rajeswaran2019meta}, automated hyperparameter optimization \citep{franceschi2018bilevel,feurer2019hyperparameter}, reinforcement learning~\citep{konda1999actor,hong2020two}, adversarial learning~\citep{zhang2022revisiting,liu2021investigating}, signal processing~\citep{kunapuli2008classification} and AI-aware communication networks~\citep{ji2023network}. Existing studies on bilevel optimization have mainly focused on the single-machine scenario. However, due to computational challenges such as the second-order hypergradient computation and 
the increasing scale of problem models (e.g., deep neural networks), learning on a single machine turns out to be inefficient and unscalable. In addition, data privacy has also arisen as a critical concern in the single-machine setting recently~\citep{mcmahan2017communication}. These challenges have greatly motivated the recent development of federated bilevel optimization, with emerging applications such as federated meta-learning~\citep{tarzanagh2022fednest}, hyperparameter tuning for federated learning~\citep{huang2022federated}, resource allocation over edges~\citep{ji2022network} and graph-aided federated learning~\citep{xing2022big} etc.   

Mathematically, federated bilevel optimization takes the following formulation with $m$ clients. 
\begin{align}\label{object function}
&\min_{x\in\mathbb{R}^{d_1}}\ \  \ f(x)=\frac{1}{m}\sum_{i=1}^mf_i(x,y_{(x)}^\ast) \nonumber\\
&\text{subject to}\ \ 
y_{(x)}^\ast\in\argmin_{y\in\mathbb{R}^{d_2}}\frac{1}{m}\sum_{i=1}^{m}g_i(x,y),
\end{align}
where the upper- and lower-level functions  $f_i(x,y)=\mathbb{E}_{\xi_i} F_i(x,y;\xi_i)$ and $g_i(x,y)=\mathbb{E}_{\zeta_i} G_i(x,y;\zeta_i)$ for each client $i$ are jointly continuously differentiable. To efficiently solve the distributed nested problem in \cref{object function}, the biggest challenge lies in computing the gradient of the upper-level objective, i.e., the hypergradient $\nabla f(x)$, due to the approximation of a global Hessian inverse matrix and the client drift induced by the data  heterogeneity~\citep{karimireddy2020scaffold,hsu2019measuring}. To overcome these issues, existing approaches all focus on the AID-based federated hypergradient
estimation~\citep{huang2022federated,tarzanagh2022fednest}. However, the AID-based approaches naturally contain two consecutive loops at each outer iteration, each of which contains a large number of communication rounds,  for minimizing the lower-level objective and constructing the federated hypergradient estimate, separately, as shown in the left illustration in \Cref{Fig.process}. This heavily complicates the implementation 
and increases the communication cost.

\begin{figure*}[t]
% \vspace{-0.1cm}
    \centering
    \vspace{-0.1cm}
    \includegraphics[scale=0.32]{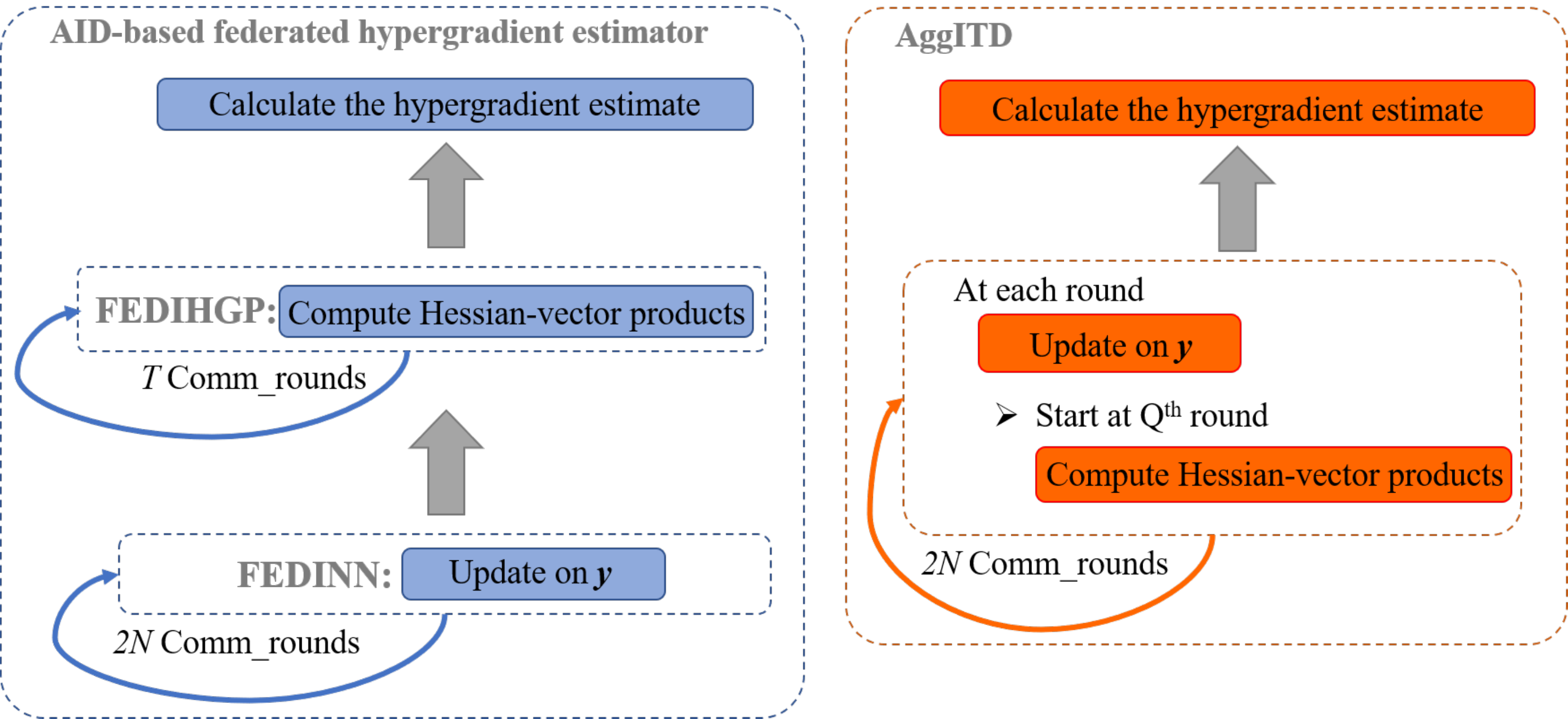}
    \hspace{0.2cm}
    \includegraphics[scale=0.2]{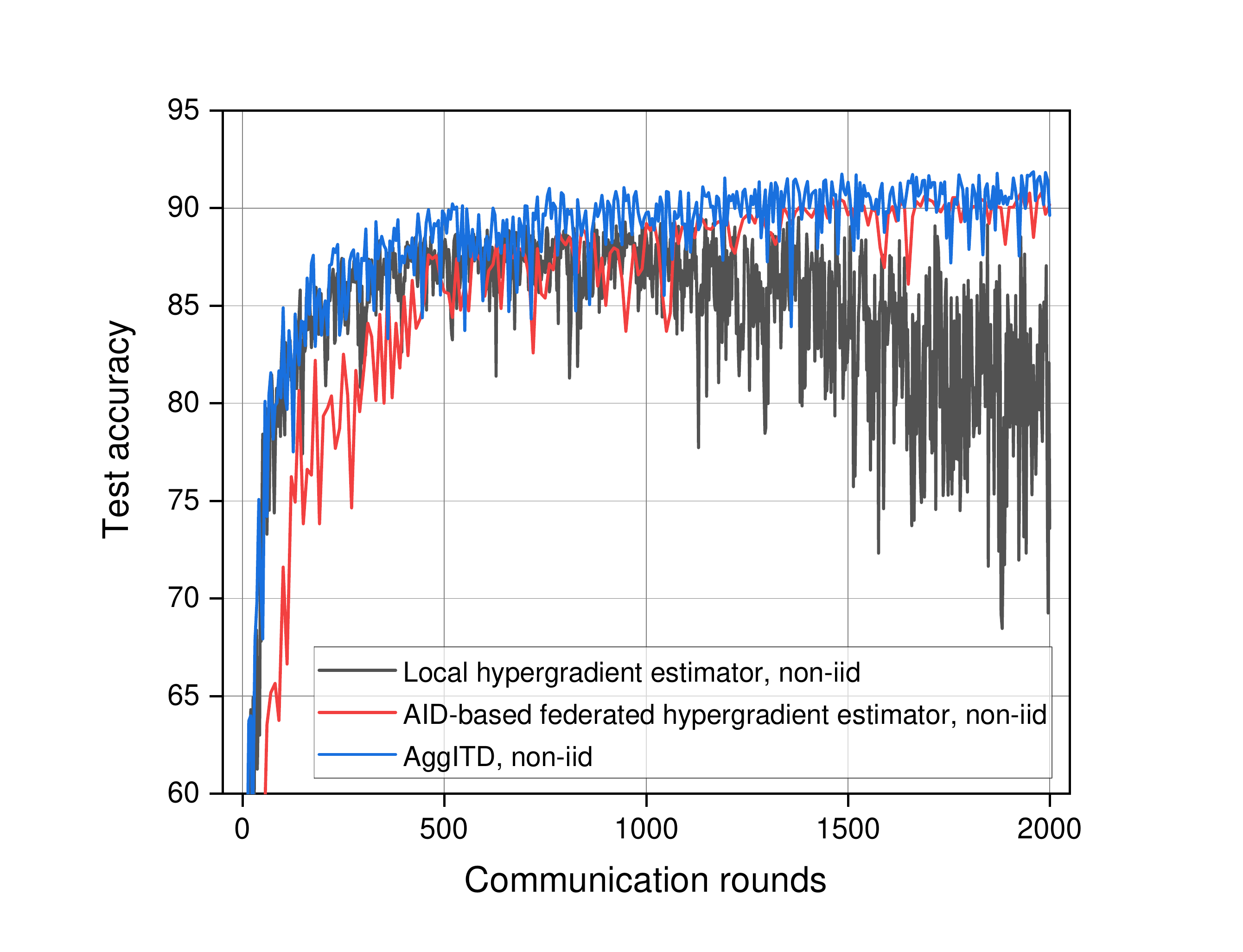}
        \vspace{-0.2cm}
    \caption{Comparison between AID-based FHE (left)~in FedNest~\citep{tarzanagh2022fednest} and our proposed AggITD estimator (middle). The right plot compares the performance  among fully local hypergradient estimator (i.e., using only local information), AID-based FHE and AggITD in federeated hyper-representation learning  in the presence of data heterogeneity.} 
    \label{Fig.process}
    % \vspace{-0.35cm}
\end{figure*}
% \vspace{-0.15cm}
\subsection{Main contributions}
In this paper, we propose a new federated hypergradient estimator (FHE) via aggregated iterative differentiation, which we refer to  as AggITD. 
% Our estimator has two features. First,  
% AggITD
As shown in \Cref{Fig.process}, our AggITD estimator 
leverages  intermediate iterates of the lower-level updates on $y$ for  the federated hypergradient estimation rather than the last iterate as in AID-based methods, and hence admits a simpler implementation  and much fewer communication rounds by conducting the lower-level updates on $y$ and the Hessian-vector-based hypergradient estimation simultaneously within the same communication loop. 
Our detailed contributions are summarized as below. 
% for the performance gain.     

{\bf A new ITD scheme.} We first show that existing ITD-based approaches in the non-distributed  setting~\citep{franceschi2018bilevel,grazzi2020iteration,ji2021bilevel} rely on the accomplishment of the lower-level updates on $y$ for the matrix-vector-based hypergradient estimation, and hence still requires two long communication loops for the federated hypergradient estimation (see \Cref{sec:aggaid} for more details). In contrast, we propose a new iterative differentiation process suitable for the efficient distributed implementation, which starts the matrix-vector based hypergradient estimation at a randomly sampled intermediate lower-level iterate, as illustrated in \Cref{Fig.process}. We anticipate that our estimator can be of independent interest to other distributed settings such as decentralized or asynchronous bilevel optimization. 
% and enable the simultaneous lower-level updating and hypergradient 

\begin{table*}[!t]
\renewcommand{\arraystretch}{1.5}
\centering
\small
\begin{tabular}{|c|c|c|c|} \hline
 \textbf{Hypergradient estimators} & Comm\_rounds$/$Outer\_itr & Comm\_loops$/$Outer\_itr & Sample complexity
 \\ \hline 
 AID-based FHE~\citep{tarzanagh2022fednest} & $2N+T+3$   & $2$  & {\small $\mathcal{\widetilde O}(\epsilon^{-2})$}  
 \\\cline{1-4}
AggITD (this paper) \cellcolor{blue!15}&  \cellcolor{blue!15} $2N+3$ &  \cellcolor{blue!15} 1& \cellcolor{blue!15}
{\small $\mathcal{\widetilde O}(\epsilon^{-2})$} \\ \hline
\end{tabular}
\caption{Comparison of AID-based FHE and the proposed AggITD in the presence of data heterogeneity. Communication round: the procedure that "\textbf{for} $i \in S$, in parallel \textbf{do}", where the participating clients send their local information (gradients or Hessian-vector products) to the server for aggregation, and the aggregated information is then broadcast back to clients. $N$ and $T$: the number of iterations for optimizing the lower-level objective and approximating the global Hessian-inverse-vector product, respectively.  
Sample complexity: the total number of samples to achieve an $\epsilon$-accurate stationary point.  
$\mathcal{\widetilde O}$: hide $\log$ factors.}\label{tab:compare}
\vspace{-0.3cm}
\end{table*}
 
{\bf Communication-efficient bilevel optimization.} Building on the proposed AggITD, we further develop a federated bilevel optimization algorithm named FBO-AggITD, which incorporates the technique of federated variance reduction into the lower- and upper-level updates on $y$ and $x$ to mitigate the impact of the client drift on the hypergradient estimation accuracy.  FBO-AggITD contains only a single communication loop, where only efficient matrix-vector products rather than Hessian or Hessian-inverse matrices are computed and communicated for the global Hessian-inverse-vector approximation. 

{\bf New theoretical analysis.} We provide a novel error and convergence analysis for the proposed AggITD estimator and FBO-AggITD algorithm, respectively. The analysis addresses two major challenges. First, differently from the AID-based estimator, the proposed AggITD depends on less accurate intermediate iterates $y^{t},t=Q+1,...,N$ at a random index  $Q$, which may introduce uncontrollable estimation errors due to the client drift. Second, the randomness from stochastic Hessian matrices and gradients further complicates the analysis. In fact, there has been no analysis even for non-distributed stochastic ITD-based estimators. To this end, a tighter recursion type of analysis is developed by decoupling the errors induced by the lower-level updates and the global Hessian-inverse-vector approximation. As shown in \Cref{tab:compare}, AggITD achieves the same sample complexity of $\mathcal{\widetilde O}(\epsilon^{-2})$ as the AID-based FHE~\citep{tarzanagh2022fednest}, with much fewer communication rounds. 

{\bf Strong empirical performance.} As shown in the right plot of \Cref{Fig.process}, AggITD admits a much faster convergence rate w.r.t.~communication rounds and better test accuracy than AID-based FHE. In addition, compared to the fully local hypergradient estimator (which is computed using only local client data), AggITD achieves a much higher test accuracy with a comparable rate and is much more stable with lower variance. This demonstrates the importance of aggregation under the lower-level heterogeneity. 
Such comparisons are also observed in the experiments in \Cref{sec:exp}.

\vspace{-0.1cm}
\subsection{Related work}
\vspace{-0.1cm}
\textbf{Bilevel optimization.} A large body of bilevel optimization methods have been proposed since the work in~\citealt{bracken1973mathematical}. 
For example,  \citealt{hansen1992new, gould2016differentiating, shi2005extended, sinha2017review} reduced the bilevel problem to the single-level constraint-based problem. Gradient-based methods have drawn more attention in machine learning recently, which can be generally categorized into 
AID~\citep{domke2012generic, pedregosa2016hyperparameter, liao2018reviving,arbel2022amortized} and ITD~\citep{maclaurin2015gradient, franceschi2017forward,finn2017model, shaban2019truncated, grazzi2020iteration} based methods. Various stochastic bilevel optimizers have also been developed via momentum~\citep{yang2021provably,huang2021biadam,guo2021randomized}, variance reduction~\citep{yang2021provably,dagreou2022framework}, Neumann series~\citep{chen2021single,ji2021bilevel}. 
Theoretically, the convergence of bilevel optimization has been analyzed by \citealt{franceschi2018bilevel,shaban2019truncated,liu2021value,ghadimi2018approximation,ji2021bilevel,hong2020two}. More results and details can be found in the survey by  \citealt{liu2021investigating}. 
In this paper, we propose a new stochastic ITD-based hypergradient estimator, which is further extended to the federated setting.  

% mainly including reducing to a single-level problem and the gradient-based methods. In general, the former is relatively difficult to perform 
% \cite{hansen1992new, gould2016differentiating, shi2005extended, sinha2017review} but the latter is simple and effective, which will be used in this research. 
% The gradient-based method can be divided into two approaches,  approximate implicit differentiation (AID), \citep{domke2012generic, pedregosa2016hyperparameter, liao2018reviving, ghadimi2018approximation} and iterative differentiation (ITD), \citep{maclaurin2015gradient, franceschi2017forward, shaban2019truncated, grazzi2020iteration},  and both of them essentially provide an estimation of the hypergradient $\nabla f(x)$. 
% For example, under nonconvex-strongly-convex conditions, \cite{ghadimi2018approximation} proposed their complexity analysis for AID-based algorithm and \cite{ji2021bilevel} gave a non-asymptotic theorem for both AID-based and ITD-based solutions. Here in this paper, we used the ITD approach to solve the problem.
% AID-based methods give solutions by implicit differentiation and solving a linear system (\cite{franceschi2018bilevel}) while ITD-based methods solve problems by automatic differentiation (\cite{ghadimi2018approximation}, \cite{ji2021bilevel})

\noindent \textbf{Federated learning.} Federated Learning was firstly introduced to allow different clients to train a model collaboratively without sharing data \citep{konevcny2015federated, shokri2015privacy,mohri2019agnostic}. As one of the earliest methods, FedAvg has been shown to effectively reduce the communication cost~\citep{mcmahan2017communication}. An increasing number of variants of FedAvg have been further proposed to address the issues such as the slow convergence and client drift via regularization~\citep{li2020federated,acar2021federated}, variance reduction~\citep{mitra2021linear,karimireddy2020scaffold}, proximal splitting~\citep{pathak2020fedsplit} and adaptive optimization~\citep{reddi2020adaptive}. In the homogeneous setting, FedAvg is relevant to local SGD, and has been analyzed in  \citealt{stich2019local,wang2018cooperative,stich2019error,basu2019qsparse}. 
In the heterogeneous setting,  \citealt{li2020federated,wang2020zeroth, mitra2021linear,li2019convergence,khaled2019first} provided the convergence analysis of their methods.

\noindent \textbf{Federated bilevel optimization.} Recent works~\cite{gao2022convergence,li2022local} focused on the homogeneous setting, and proposed momentum-based methods with fully local hypergradient estimators. The most relevant work  \cite{tarzanagh2022fednest} proposed FedNest using an AID-based FHE, and further provided its convergence rate guarantee despite the data heterogeneity. This paper proposes a simple and communication-efficient method via an ITD-based FHE.

% As demonstrated above many problems present nested structures, FL should be expanded to federated bilevel problems. For example, \citealt{gao2022convergence} invented momentum-based algorithms and \citep{xing2022big} put forward a graph-aided approach for optimizing the problems. The most relevant work is done by \cite{tarzanagh2022fednest} focusing on the federated minimax problem with their algorithm, FedNest, which used the AID approach and provided non-asymptotic analysis under nonconvex-strongly-convex conditions.

% \noindent \textbf{Decentralized bilevel optimization:} 

Bilevel optimization has also been studied in other distributed setups such as decentralized bilevel optimization~\citep{chen2022decentralized,yang2022decentralized,lu2022decentralized} and  asynchronous bilevel optimization over directed network~\citep{yousefian2021bilevel}. 
We anticipate that our proposed ITD-based estimator can be also applied to these scenarios.   

% Decentralized optimization is attracting lots of attention in recent years because it has been proven that under certain situations, decentralized optimization can have its own advantages \cite{lian2017can}. However, very few research studied decentralized bilevel optimization. As far as we know, \cite{lu2022decentralized} studied this question under a personalized bilevel optimization scenario and \cite{chen2022decentralized} provided their solutions assuming data homogeneity. We suspect that our proposed ITD-based construction can also provide a simple and communication-efficient implementation in decentralized bilevel optimization.

\noindent {\bf Notations.} We use $\partial  f(x,y^*_{(x)})/\partial x$ to denote the gradient of $f$ as a function of $x$, and $\nabla_xf$ and $\nabla_yf$ are partial derivatives of $f$ with respect to $x$ and $y$. 
% for the differentiable function $f(x,y)$: $\mathbb{R}^{d_1}\times\mathbb{R}^{d_2}\rightarrow\mathbb{R}\text{ and }y=y(x):\mathbb{R}^{d_1}\rightarrow\mathbb{R}^{d_2}$. 
For any vector $v$ and matrix $M$, we denote $\|v\|$ and $\|M\|$ as Euclidean and spectral norms, respectively. We let $f(x,y)=\frac{1}{m}\sum_{i=1}^mf_i(x,y)$ and $g(x,y)=\frac{1}{m}\sum_{i=1}^m g_i(x,y)$ denote the averaged upper- and lower-level objective functions across all clients $i$. 
Finally, let $S=\{1,...,m\}$ denote the set of all clients.

\vspace{-0.2cm}
\section{Federated Hypergradient Computation}
\subsection{Federated Hypergradient and Existing Approach}
\noindent {\bf Federated hypergradient.} The biggest challenge of federated bilevel optimization lies in computing the aggregated hypergradient 
$\nabla f(x)=\frac{1}{m}\sum_{i=1}^m\frac{\partial f_i(x,y_{(x)}^\ast)}{\partial x}$ due to the implicit dependence of the global lower-level solution $y^*_{(x)}$ on  $x$. 
Using the implicit function theorem~\citep{griewank2008evaluating} and  if $ g(\cdot)$ is twice differentiable and 
$\nabla_y^2g(x,y^*_{(x)})$ is invertible, an explicit form of $\nabla f(x)$ is 
\begin{align}\label{eq:implicitiHg}
\nabla f(x) =& \frac{1}{m}\sum_{i=1}^m\Big(\nabla_x f_i(x,y^*_{(x)}) - \nabla_x\nabla_yg(x,y^*_{(x)}) \nonumber
\\&\times\big[\nabla_y^2g(x,y^*_{(x)})\big]^{-1}\nabla_y f_i(x,y^*_{(x)})\Big),
\end{align}
where the first and second terms on the right side are direct and indirect parts of the federated hypergradient. 
As shown by \cref{eq:implicitiHg}, two challenges arise in the federated hypergradient computation. First, the second-order derivatives  $\nabla_x\nabla_yg(x,y^*_{(x)})$ and $\nabla_y^2g(x,y^*_{(x)})$ are all {global} information that is not accessible to each client $i$.    
 This greatly complicates the design of an unbiased estimate of $\nabla f(x)$. For example, it can be seen that a straightforward estimator by replacing such two global quantities with their local counterparts, i.e.,  $\nabla_x\nabla_yg_i(x,y^*_{(x)})$ and $\nabla_y^2g_i(x,y^*_{(x)})$ is a {biased} approximation of $\nabla f(x)$ due to the client drift.   
Second, it is highly infeasible to compute and communicate second-order information (such as Hessian inverse or even Hessian/Jacobian matrices) due to the restrictive computing and communication resource. 

\noindent {\bf AID-based FHE.} To address these challenges, \citealt{tarzanagh2022fednest} recently proposed  
a matrix-vector-based FHE  building on a non-federated AID-based estimate used in \citealt{ghadimi2018approximation}, which takes the form of 
\begin{align*}
 \hat h^I(x)=\frac{1}{m}\sum_{i=1}^m \big[ &\nabla_x F_i(x,y^N;\xi_i) -\nabla_x\nabla_yG_i(x,y^N)p_{T^\prime} \big]
 % \vspace{-0.3cm}
\end{align*}
where $y^N$ is first obtained to estimate the global $y^*_{(x)}$ via a FedSVRG~\citep{mitra2021linear,konevcny2016federated} type of method with $2N$ communication rounds and an  aggregated Hessian-inverse-vector (HessIV) estimate
\begin{align*}
   &p_{T^\prime} = \prod_{t=1}^{T^\prime}\big(I-\lambda\frac{1}{|S_t|}\sum_{i\in S_t}\nabla_y^2G_i(x,y^N;\zeta_{i,t})\big) p_0 
   \\&\text{with}\quad p_0= \frac{\lambda T}{|S|}\sum_{i\in S}\nabla_y F_i(x,y^N;\xi_{i,0})
\end{align*}
is then constructed based on the inner output $y^N$ using extra $T^\prime$ communication rounds as, for $t=1,...,T^\prime$
\begin{align*}
&\text{Local client $i$:} \quad p_{i,t} = (I-\lambda\nabla_y^2G_i(x,y^N;\zeta_{i,t}))p_{t-1}\nonumber
\\& \text{Server aggregates:} \quad 
p_t = \frac{1}{|S_t|}\sum_{i\in S_t}p_{i,t},
\end{align*}
where $T^\prime$ is chosen from $\{0,...,N-1\}$ uniformly at random. However, several challenges still remain, as elaborated in the next \Cref{sec:aggaid}.

\subsection{Our Method: Aggregated Iterative Differentiation}\label{sec:aggaid} 
\noindent {\bf Challenges in AID-based FHC.} Note that at each outer iteration $k$, AID-based FedIHGP includes two major communication loops, i.e., $2N$ rounds for inner  $y$ updates and $T^\prime$ rounds for outer  FHC, which introduce two challenges in practice. First, the construction of an AID-based hypergradient estimate is built on the output $y^N$ is inherently separated from the inner $y$ updating loop, and the resulting two communication and optimization loops complicate the implementation in practice. Second, the separate $T^\prime$ (which can be large at an order of $\kappa\log \frac{1}{\epsilon}$ in the worst case~\citep{tarzanagh2022fednest} communication rounds for the HessIV estimation can add a non-trivial communication burden on the FL systems due to the limited communication bandwidth and resource (e.g., in wireless setting). Then, an important question here is: {\em Can we develop a new FHE that can address these implementation and communication challenges simultaneously, while achieving better communication and computational performance in theory and in practice?} In this section, we provide an affirmative answer to this question by developing a novel aggregated iterative differentiation (AggITD)  for communication-efficient FHC.

% \begin{algorithm}[H]
% \caption{$\hat{h},\;y^{N}=\textbf{AggITD-S-C}(x,y,\alpha,\beta)$}   
% \begin{algorithmic}[1]\label{algorithm2}
% \STATE{\textbf{Clients calculate the following values}}
% \FOR{$t=0,1,2,...,N-1$}
% \FOR{$i\in S$ in parallel}
% \STATE{$q^t_i=\nabla_yG_i(x,y^t;\zeta_{i,t})$}
% \ENDFOR
% \STATE{$q^t=\frac{1}{|S|}\sum_{i \in S}q_i^t$}
% \STATE{$y^{t+1}=\textbf{FedLower}(x,y^t, q^t, \beta)$}
% \STATE{\textbf{Store aggregated  $q^t,y^{t+1}$ in server}}
% \ENDFOR
% \FOR{$i\in S$ in parallel}
% \STATE{$r^N_i=\nabla_yF_i(x,y^N;\xi_{i})$}
% \ENDFOR
% \STATE{Aggregate $r^N=\frac{1}{|S|}\sum_{i \in S}r^N_i$}
% % \STATE{$q^N=\frac{1}{|S|}\sum_{i \in S}\nabla_yG_i(x,y^N;\zeta_{i,N})$}
% \STATE{\textbf{Server calculates the following}}
% \FOR{$j=N,N-1,N-2,...,1$}
% % \FOR{$i \in S$}
% % \STATE{$v^{i,j-1}=\frac{\partial\langle y^j-\alpha q^{i,j},v^j\rangle}{\partial y^j}$ via automatic differentiation}
% % \ENDFOR
% % \STATE{$v^{j-1}=\frac{1}{|S|}\sum_{i\in S}v^{i,j-1}$}
% \STATE{$r^{j-1}=\frac{\partial\langle y^j-\alpha q^{j},r^j\rangle}{\partial y^j}$ via autograd}
% \ENDFOR
% \STATE{$r=\sum_{j=0}^{N}r^{j}$}
% \FOR{$i\in S$ in parallel}
% \STATE{$\hat{h}_i=\nabla_xF_i(x, y^N;\xi_{i,N})-\frac{\alpha\partial\langle\nabla_yG_i(x, y^N;\chi_i), r\rangle}{\partial x}$ via autograd}
% \ENDFOR
% \STATE{Aggregate $\hat{h}=\frac{1}{|S|}\sum_{i\in S}\hat{h}_i$}
% % \STATE{Return $\hat{h},\;y^{N}$}
% \end{algorithmic}
% \end{algorithm}

\begin{algorithm}[t]
% \small
\caption{$\widetilde{h},\;y^N=\textbf{AggITD}(x,y,\beta)$}
% \quad $\rhd$ Aggregated ITD-Based Hypergradient Estimator}
\begin{algorithmic}[1]\label{algorithm2}
\STATE{Set $y^0=y$ and choose $Q$ from $\{0,...,N\}$ UAR}
\FOR{$t=0,1,2,...,N$}
\FOR{$i\in S$ in parallel 
% and {$t\neq0$}
}
\STATE{
% $p^i=\frac{\partial [y^t-\alpha \nabla_yG_i(x,y^t;\zeta_{i,t})]}{\partial y^t}$\\
Compute $q_i^t=\nabla_yG_i(x,y^t;\zeta_{i,t})$  for $y$ updates}
\STATE{{if} $t=Q$, compute $r_i^t =\nabla_yF_i(x,y^t;\xi_{i,t}) $
}
\STATE{{if} $t\geq Q+1$ and $t\leq N$,  compute $z_i^t=z^{t-1}-\frac{\partial \langle\nabla_yG_i(x,y^t;u_{i,t}),z^{t-1}\rangle}{\partial y^t}$ via autograd}
\ENDFOR
\IF{$t\leq N-1$}
\STATE{Server aggregates and broadcasts
$q^t=\frac{1}{|S|}\sum_{i \in S}q_i^t$}
\STATE{$y^{t+1}=\textbf{One-Round-Lower}(x,y^t, q^t, \beta)$}
\ENDIF
\STATE{{if} $t=Q$, aggregate $z^t:=r^t=\frac{1}{|S|}\sum_{i \in S}r_i^t$} 
% and broadcast $z^t=r^t$}
\STATE{{if} $t \geq Q+1$ and $t\leq N$, aggregate $z^{t}=\frac{1}{|S|}\sum_{i\in S}z_i^t$ }
% \STATE{Server computes $p=\frac{\partial \langle y^t-\alpha q^t,r^t\rangle}{\partial y^t}$ via autograd} 
% \FOR{$j=0,1,2,...,t-1$ \& $t\neq0$}
% \STATE{$r^{j}=pr^{j}$}
% \ENDFOR
\ENDFOR
\STATE{{\small$p=\lambda(N+1)z^{N-1}$} if {\small$Q<N$} or {\small$\lambda(N+1)z^N$} otherwise.}
% \STATE{$r=\sum_{t=0}^{N}r^t$}
\FOR{$i\in S$ in parallel}
% \STATE{$\widetilde{h}^I_i=\frac{\partial\langle\nabla_yG_i(x, y^N;\chi_i), p\rangle}{\partial x}$ via autograd}
\STATE{$\widetilde{h}_i=\nabla_xF_i(x, y^N;\xi_{i})-\frac{\partial\langle\nabla_yG_i(x, y^N;\chi_i), p\rangle}{\partial x}$}
\ENDFOR
\STATE{Server aggregates $\widetilde{h}=\frac{1}{|S|}\sum_{i\in S}\widetilde{h}_i$}
% \STATE{Return $\widetilde{h},\;y^N$}
\end{algorithmic}
\end{algorithm}

\noindent {\bf Our idea.} Instead of constructing the federated hypergradient after obtaining the inner output $y^N$, our idea is to utilize the intermediate iterates $y^1,...,y^N$ and communication rounds of the inner $y$ loop also for the federated hypergradient approximation, and hence remove the expensive $T^\prime$ communication rounds. To do this, one possible solution is to use the idea of an ITD-based method from the non-federated bilevel optimization~\citep{ji2021bilevel,grazzi2020iteration}, which approximates the hypergradient $\frac{\partial f(x,y^*(x))}{\partial x}$ by computing $\frac{\partial f(x,y^N)}{\partial x}$ via the automatic differentiation, where $y^N$ is the $N$-step output of gradient descent\footnote{We take GD as an illustration example, and other solvers can also be used}, i.e., $y_{t+1}=y_t-\alpha\nabla_y g(x,y_t)$ for $t=0,...,N-1$. The explicit form of the indirect part of $\frac{\partial f(x,y^N)}{\partial x}$ is then taken as 
\begin{align}\label{aid_fedhyper}
    -\alpha&\sum_{t=0}^{N-1}\nabla_x\nabla_y  g(x,y^t)\nonumber
    \\&\times\prod_{j=t+1}^{N-1}(I-\alpha\nabla_y^2g(x,y^j))\nabla_y f(x,y^N),
\end{align}
which, however, still needs an extra communication loop for the construction because its matrix-vector computations require the information of $\nabla_y f(x,y^N)$ at the output $y^N$, and in addition, 
the $N$ summations complicate the federated implementation. To this end, we next provide a novel aggregated ITD-based estimator for FHC, which uses the same communication loop for both the $y$ updates and the federated hypergradient construction. 

\begin{algorithm}[t]
\caption{$y_+=\textbf{One-Round-Lower}(x,y,q,\beta)$}   
\begin{algorithmic}[1]\label{algorithm3}
\FOR{$i \in S$ in parallel}
\STATE{$y_{0}^{i}=y$ and choose $\beta^i\in(0,\beta]$}
% \IF{$r\neq$ null}
% \STATE{Compute $p^i=\frac{\partial \langle y^t-\alpha\nabla_yG_i(x,y^t;\zeta_{i,t}),r\rangle}{\partial y^t}$}
% \ENDIF
\FOR{$\upsilon=0,1,2,...,\tau_i-1$}
\STATE{$q_{\upsilon}^i=\nabla_yG_i(x,y_{\upsilon}^{i};\zeta^i_\upsilon)-\nabla_yG_i(x,y;\zeta^i_\upsilon)+q$}
\STATE{$y_{\upsilon+1}^{i}=y_{\upsilon}^{i}-\beta^i q_{\upsilon}^i$}
\ENDFOR
\ENDFOR
\STATE{$y_+=\frac{1}{|S|}\sum_{i\in S}y_{\tau_i}^{i}$}
% \STATE{Return $y_+$}
\end{algorithmic}
\end{algorithm}
% In addition, there has been no theoretical guarantee for this ITD-based estimator in the stochastic setting, which naturally exists in the federated learning due to  the client selection and the in-client data sampling. To this end,   
% where $y_N$ is the $N$-step output of some iterative method $y_{t+1}=\Psi(g,y_t)$. For the most widely-used gradient descent, i.e., $\Psi(y_t)=y_t-\alpha\nabla g(x,y_t)$, 

% \vspace{0.1cm}
% In this paper, we propose a novel communication-efficient approaches for federated hypergradient computation (FHC) via aggregated iterative differentiation (AggITD), which we refer as to FHC-AggITD. 
\noindent {\bf Proposed AggITD.} As shown in \cref{algorithm2} and the illustration in \Cref{Fig.process},  AggITD 
first samples an index $Q$ from the set $\{0,...,N\}$ uniformly at random, and then at each inner iteration $t$, each client $i$ computes the local gradient $\nabla_yG_i(x,y^t;\zeta_{i,t})$, which are aggregated for optimizing the lower-level objective via the federated SVRG-type One-Round-Lower sub-procedure in \Cref{algorithm3}. The steps in lines 5-6 and 12-13 provide an efficient iterative way to construct a novel estimate of federated Hessian-inverse-vector product {\small$(\nabla_y^2g(x,y^*_{(x)}))^{-1}\nabla_yf(x,y^*_{(x)})$}, which is given by 
\begin{align*}
  \widehat{\text{HessIV}}=&  \lambda(N+1)\prod_{t=N}^{Q+1}
\big(I-\frac{\lambda}{|S|}\sum_{i\in S}\nabla_y^2G_i(x,y^t;u_{i,t})\big)
\\&\times\bigg [\frac{1}{|S|}\sum_{i\in S}\nabla_yF_i(x,y^Q;\xi_{i,Q})\bigg],
\end{align*}
where we use $\prod_{j=N}^{N+1}(\cdot)=I$ for simplicity.  
% \Cref{Fig.process} provides an overview of how to calculate hypergradient estimate for FedNest and our method.
Note that these steps for the FHC process compute and communicate only efficient Hessian-vector products {\small $\frac{\partial \langle\nabla_yG_i(x,y^t;u_{i,t}),z^{t-1}\rangle}{\partial y^t}=\nabla_y^2G_i(x,y^t;u_{i,t})z^{t-1}$} using automatic differentiation (e.g., {\bf torch.autograd}), rather than Hessian or Hessian inverse matrices. After broadcasting the global $\widehat{\text{HessIV}}$, each client $i$ builds a local FHE $\widetilde h_i(x)=\widetilde h_i^D(x) - \widetilde h_i^I(x)$, where the direct and indirect parts are given by 
\begin{align*}
    \widetilde h_i^D(x)=&\nabla_xF_i(x, y^N;\xi_{i})
   \\ \widetilde h_i^I(x) = &\nabla_x\nabla_yG_i(x, y^N;\chi_i)\widehat{\text{HessIV}}.
\end{align*}
Then, the aggregated hypergradient estimate is given by $\widetilde h(x)=\widetilde h^D(x)-\widetilde h^I(x)=\frac{1}{|S|}\sum_{i\in S}\widetilde h_i(x)$. Meanwhile, we would like to point out the differences between our method and distributed bilevel problems, such as \cite{yang2022decentralized}. First, in our algorithm, the server is to aggregate the local weights from clients and broadcast the aggregated weights back to the clients. In contrast, for such decentralized methods, the server needs to compute the gradients or hypergradients. Then, our method runs multiple local updates to improve communication efficiency, whereas the decentralized methods do not have such operations. Third, all such decentralized methods use the AID-based hypergradient estimator, whereas our method uses the ITD-based scheme.
% Empirically, our comparison shows that the gossip method is much slower than our method because it involves the expensive computations of the Jacobian matrices and the Hessian-inverse matrices.
However, to analyze this AggITD-based estimator, several technical challenges arise as below.

% \vspace{0.2cm}
\noindent {\bf Technical challenges.} First, differently from the AID-based FHE that is evaluated at the last iterate $y^N$, our proposed estimator depends on less accurate intermediate iterates $y^{t},t=Q+1,...,N$, which may introduce larger or even uncontrollable estimation errors given the client drift effect. Thus, a more careful and tighter analysis is required. Second, the  randomness from stochastic Hessian matrices and gradients further complicates the analysis. In fact, there has been no analysis for even non-federated (i.e., $|S|=1$) {\bf stochastic} ITD-based estimators. 
Third, the aggregation $\frac{1}{|S|}\sum_{i\in S}$ complicates the bias and variance analysis.

\vspace{-0.2cm}
\section{Proposed Algorithm}
We now develop a new federated bilevel optimizer named FBO-AggITD based on the proposed AggITD estimator. As shown in \Cref{algorithm1}, FBO-AggITD first obtains the federated hypergradient estimate $\widetilde h$ and the approximate $y_{k+1}=y_k^N$ of the lower-level solution $y_k^*$ via the {\bf AggITD} sub-procedure in \Cref{algorithm2}. Then, building on $\widetilde h$ and $y_{k+1}$, similarly to \cite{tarzanagh2022fednest}, we use a local SVRG-type {\bf One-Round-Upper} sub-procedure for solving the upper-level problem w.r.t.~$x$, where 
each client $i$ runs $\tau_i$ steps based on the radient $h_{i,\upsilon}$ given by 
\begin{align*}
h_{i,\upsilon}=&h-\nabla_xF_i(x,y_+;\xi^i_{\upsilon})+\nabla_xF_i(x_\upsilon^i,y_+;\xi^i_{\upsilon}) \nonumber
\\ =& \widetilde h^D(x) - \widetilde h^I(x)- \nabla_xF_i(x,y_+;\xi^i_{\upsilon})
\\&+\nabla_xF_i(x_\upsilon^i,y_+;\xi^i_{\upsilon}), 
\end{align*}
where the direct part $\widetilde h^D(x)=\frac{1}{|S|}\sum_{i\in S}\nabla_xF_i(x, y_+;\xi_{i})$ of the global hypergradient estimate $\widetilde h$ uses different samples $\xi_i$ from $\xi_\upsilon^i$ of the local gradient $\nabla_x F_i(x,y_+;\xi_\upsilon^i),i\in S$ to provide an SVRG-type variance reduction effect on the direct part of the hypergradient. This is in contrast to  the upper update in FedNest~\citep{tarzanagh2022fednest} where 
the data samples $\xi_i$ and $\xi_\upsilon^i$ are chosen to be the same. 
Note that 
we do not apply the SVRG-type updates to the entire hypergradient but only the direct part because the indirect part requires the global Hessian information at iterates $x_\upsilon^i$, which is infeasible at each client $i$. 

\begin{algorithm}[H]
\caption{FBO-AggITD}   
\begin{algorithmic}[1]\label{algorithm1}
\STATE {\bfseries Input:}  $K,N\in\mathbb{N}$, $\alpha_k, \beta_k >0$, initializations $x_0, y_0$.
\FOR{$k=0,1,2,...,K$}
% \STATE{$y_k^0=y_k$}
\STATE{$\widetilde{h},\;y_{k+1}=\textbf{AggITD}(x_k,y_k,\beta_k)$}
% \STATE{Set $y_{k+1}=y_k^{N}$}
\STATE{$x_{k+1}=\textbf{One-Round-Upper}(x_k,y_{k+1}, \widetilde{h}, 
\alpha_k)$}
\ENDFOR		
\end{algorithmic}
\end{algorithm}

\vspace{-0.5cm}
\begin{algorithm}[ht]
\caption{$x_{+}=\textbf{One-Round-Upper}(x,y_+,h,\alpha)$}   
\begin{algorithmic}[1]\label{algorithm4}
\FOR{$i \in S$ in parallel}
\STATE{$x_0^i=x$ and choose $\alpha^i\in(0,\alpha]$}
\FOR{$\upsilon=0,1,2,...,\tau_i-1$}
% \STATE{$h_{i,\upsilon}=\nabla_xF_i(x_\upsilon^i,y_+;\xi^i_{\upsilon})-h^I$}
\STATE{$h_{i,\upsilon}=h-\nabla_xF_i(x,y_+;\xi^i_{\upsilon})+\nabla_xF_i(x_\upsilon^i,y_+;\xi^i_{\upsilon})$}
\STATE{$x_{\upsilon+1}^{i}=x_{\upsilon}^{i}-\alpha^i h_{i,\upsilon}$}
\ENDFOR
\ENDFOR
\STATE{$x_+=\frac{1}{|S|}\sum_{i\in S}x_{\tau_i}^{i}$}
\end{algorithmic}
\end{algorithm}
% \section{Definitions and Assumptions}\label{sec:def}

\vspace{-0.3cm}
\section{Main Results}
\vspace{-0.1cm}
%In this section, we first provide some necessary definitions and assumptions, then characterize the important estimation properties of the proposed AggITD estimator, and finally develop the convergence and complexity analysis for the proposed FBO-AggITD algorithm.  
\subsection{Definitions and Assumptions}\label{sec:def}
Let $z=(x,y) \in \mathbb{R}^{d_1+d_2}$.  Throughout this paper, we make the following definitions and standard assumptions on the lower- and upper-level objectives, as also adopted in stochastic bilevel optimization~\citep{ji2021bilevel,hong2020two,khanduri2021near,chen2021closing} as well as in the federated bilevel optimization~\citep{tarzanagh2022fednest}.
%Assumption 1
\begin{definition}
A mapping $f$ is $L$-Lipschitz continuous if for  $\forall\,z,z^\prime$,  
% \begin{align*}
$\|f(z)-f(z^\prime)\|\leq L\|z-z^\prime\|.$
% \end{align*}
\end{definition} 
Since the objective $f(x)$ is nonconvex, algorithms are expected to find an $\epsilon$-accurate stationary point defined below.
\begin{definition}
We say $\bar x$ is an $\epsilon$-accurate stationary point of the objective function $f(x)$ if $\mathbb{E}\|\nabla f(\bar x)\|^2\leq \epsilon$, where $\bar x$ is the output of an algorithm. 
\end{definition}
\begin{assum}\label{assum:muconvex}
The lower-level function $G_i(x, y;\zeta_i)$ is $\mu$-strongly-convex w.r.t. $y$ for any $\zeta_i$. \end{assum}
%Assumption 2
%Assumption 3
The following assumption imposes the Lipschitz conditions on the lower- and upper-level functions for each client $i$. 
\begin{assum}\label{assum:lip}
The objective functions satisfy
\begin{list}{$\bullet$}{\topsep=0.2ex \leftmargin=0.2in \rightmargin=0.in \itemsep =0.06in}
% \item \textbf{(2.1)} $f_i(z),\ \nabla f_i(z),\ \nabla g_i(z),\ \nabla^2g_i(z)\ $are $M,\ L_{f},\ L_{g},\ {\rho}-Lipschitz$ continuous.
% \item \textbf{(2.2)} The function $ f(z)$ is $M$-Lipschitz i.e., for any $z,z^\prime$, $$\| f(z)- f(z^\prime)\|\leq M\|z-z^\prime\|.$$
% \item \textbf{(2.3)} Gradients $\nabla g(z)$ is $L_g$-Lipschitz, i.e., for any $z,z^\prime$, $$\|\nabla g(z)-\nabla g(z^\prime)\|\leq L_g\|z-z^\prime\|.$$
\item The function $F_i(z;\xi_i)$ is $M$-Lipschitz continuous.
% , i.e., for any $z$ and $z^\prime$
% \begin{align*}
%     |F_i(z;\xi_i)-F_i(z^\prime;\xi_i)|\leq M\|z-z^\prime\|.
% \end{align*}
\item The gradients $\nabla F_i(z;\xi_i)$ and $\nabla G_i(z;\zeta_i)$ are unbiased estimators of $\nabla f_i(z)$ and $\nabla g_i(z)$.
\item The gradients $\nabla F_i(z;\xi_i)$ and $\nabla G_i(z;\zeta_i)$ are 
$L_f$- and $L_g$-Lipschitz continuous, respectively.
% , i.e., for any $z$ and $z^\prime$,
% \begin{align*}
%     \|\nabla F_i(z;\xi_i)-\nabla F_i(z^\prime;\xi_i)\|\leq L_f\|z-z^\prime\|, \quad  \|\nabla G_i(z;\zeta_i)-\nabla G_i(z^\prime;\zeta_i)\|\leq L_g\|z-z^\prime\|.
% \end{align*}
\end{list}
% For the stochastic case, the assumptions holds for $F_i(z;\xi_i)$ and $G_i(z;\zeta_i)$ for any given $\xi$ and $\zeta$.
\end{assum}
\begin{assum}\label{high_lip}
The second-order derivatives  satisfy 
\begin{list}{$\bullet$}{\topsep=0.2ex \leftmargin=0.2in \rightmargin=0.in \itemsep =0.06in}
\item The derivatives $\nabla_x\nabla_y  G_i(z;\zeta_i)$ and $\nabla_y^2 G_i(z;\zeta_i)$ are unbiased estimators of $\nabla_x\nabla_y  g_i(z)$ and $\nabla_y^2 g_i(z)$.
\item The derivatives  $\nabla_x\nabla_y  G_i(z;\zeta_i)$ and $\nabla_y^2 G_i(z;\zeta_i)$ are $\rho$-Lipschitz continuous.  
\end{list}
% , i.e., for $\forall\,z,z^\prime$
% \begin{align*}
% \|\nabla_x\nabla_y  G_i(z;\zeta_i)-\nabla_x\nabla_y  G_i(z^\prime;\zeta_i)\|\leq \rho \|z-z^\prime\|, \,
% \|\nabla_y^2 G_i(z;\zeta_i)-\nabla_y^2 G_i(z^\prime;\zeta_i)\|\leq \rho \|z-z^\prime\|.
% \end{align*}
\end{assum}
% For stochastic case, this assumption also holds for $\nabla_x\nabla_yG(z;\zeta)$ and $\nabla_y^2G(z;\zeta)$ for any given $\zeta$
\begin{assum}\label{bounded_variance}
% We assume the following variance bounds. 
The variances of gradients  $\nabla F_i(z;\xi_i)$ and $\nabla G_i(z;\zeta_i)$
% of gradients  $\nabla_xF_i(z;\xi_i)$, $\nabla G_i(z;\zeta_i)$ and second-order derivatives $\nabla_x\nabla_y  G_i(z^\prime;\zeta_i)$ and $\nabla_y^2 G_i(z;\zeta_i)$  
are bounded by $\sigma_f^2$ and $\sigma_{1}^2$. Moreover, the lower-level client  dissimilarity $\mathbb{E}\|\nabla g_i(z)-\nabla g(z)\|^2\leq \sigma_{2}^2$.
% in the sense that 
% % in the sense that 
% % \begin{align*}
% ${\E}_{\xi_i}\|\nabla F_i(z;\xi_i)-\nabla f_i(z)\|^2\leq\sigma_f^2$ and $
% {\E}_{\zeta_i}\|\nabla G_i(z;\zeta_i)-\nabla g_i(z)\|^2\leq\sigma_g^2$
% % \\{\E}_{\zeta_i}\|\nabla_x\nabla_y  G_i(z;\zeta_i)-\nabla_x\nabla_y  g_i(z)\|^2\leq& \sigma_2^2,\,{\E}_{\zeta_i}\|\nabla_y^2 G_i(z;\zeta_i)-\nabla_y^2 g_i(z)\|^2\leq  \sigma_2^2.\nonumber
% % \end{align*}
% for some positive constants $\sigma_f$ and $\sigma_g$. 
\end{assum}
In this paper, let $\sigma_g^2=\max\{\sigma_{1}^2,\sigma_{2}^2\}$ for notational simplicity. \Cref{bounded_variance} is commonly adopted in the heterogeneous FL, and it is reduced to the homogeneous setting when $\sigma_2=0$. It is worth noting that our assumptions are exactly the same as existing AID-based federated/distributed bilevel studies such as \citep{tarzanagh2022fednest}.

\subsection{Estimation Properties for AggITD}
We analyze the estimation properties of AggITD. Let 
\begin{align}
    B^I=&\E\big[\|{\E}[\widetilde{h}^I(x)]-\nabla_x\nabla_yg(x,y^N) \nonumber
\\&\times(\nabla_y^2g(x,y^N)^{-1})\nabla_yf(x, y^N)\|^2\,|\,x,y^N\big]\nonumber
\end{align}
denote the estimation error of the indirect part of $\widetilde{h}^I(x)$. 

\begin{proposition}\label{propsfrvcc}
Suppose Assumptions~\ref{assum:muconvex}-\ref{bounded_variance} are satisfied  and let $y_{(x)}^{\ast}=\argmin_{y}g(x,y)$. Further, set $\lambda\leq\min\{10, \frac{1}{L_g}\}$ and  $\beta^i=\frac{\beta}{\tau_i}$ and any stepsize $\alpha>0$, where $\beta\leq\min\{1,\lambda, \frac{1}{6L_{g}}\}$. Then, we have 
% Then, the implicit part $\widetilde{h}^I(x)$ of the proposed AggITD hypergradient estimate satisfies
\begin{align}\label{prop:auzilix}
&B^I\leq[4\lambda^2L_g^2M^2\alpha_1(N)+4\lambda^2L_f^2L_g^2\alpha_3(N)]\E\|y-y_{(x)}^\ast\|^2\nonumber
\\&\;\ +\frac{4L_g^2M^2(1-\lambda\mu)^{2N+2}}{\mu^2}+400\lambda^2\beta^2L_g^2M^2\sigma_g^2\rho^2\alpha_2(N)\nonumber
\\&\;\ +\frac{200\lambda\beta^2\sigma_g^2L_f^2L_g^2N(N+1)}{\mu},
\end{align}
where 
% \begin{align}
% \begin{split}
$\alpha_1(N)=4(N+1)(1-\frac{\beta\mu}{2})^N[\frac{\rho^2}{\lambda\mu^3}+\frac{4\rho^2}{\beta\mu^3}],\alpha_2(N)=\frac{N(N+1)(1+(1-\lambda\mu)^2)}{\lambda\mu^3}$ and $\alpha_3(N)=\frac{3(N+1)(1-\frac{\beta\mu}{2})^N}{\lambda\mu}$.
% are positive constants.
% \end{split}
% \end{align}
\end{proposition}
\Cref{propsfrvcc} provides an upper bound on the second moment of the estimation bias of the AggITD estimator. As shown in \cref{prop:auzilix}, the first two terms $\mathcal{O}((1-\frac{\beta\mu}{2})^N\E[\|y-y_{(x)}^\ast\|^2])$ and $\mathcal{O}((1-\lambda\mu)^{2N+2})$ correspond to the estimation errors without the client drift, which can be made small by choosing $N$ properly. In addition, the initialization gap $\E[\|y-y_{(x)}^\ast\|^2]$ further relaxes the requirement of $N$ due to the warm start   $y_{k}=y_{k-1}^N$ (see \Cref{algorithm1}), as shown in the final convergence analysis.
It is worth mentioning that these two terms match the error bound of the stochastic AID-based hypergradient estimator in non-federated setting~\citep{ji2021bilevel,ghadimi2018approximation,chen2021closing}, and hence our analysis can be of independent interest to non-federated bilevel optimization. 
Also note that the last two error terms $\mathcal{O}(\lambda^2\beta^2)$ and $\mathcal{O}(\lambda\beta^2)$
are induced by the client drift in the $y$ updates, which exists especially in the FL, can be addressed by choosing a sufficiently small stepsize $\beta$. 
Technically, we first show via a recursive analysis that the key approximation error between the expected indirect part of the AggITD estimator 
\begin{align}\label{eq:intermidB}
\mathbb{E}[\widetilde  h^I(x)|x,y^N]=&\lambda\nabla_x\nabla_yg(x,y^N) \nonumber
\\\times\sum_{Q=0}^{N}\prod_{t=N}^{Q+1}&(I-\lambda\nabla_y^2g(x,y^t))\nabla_yf(x,y^Q)
\end{align}
and the underlying truth is bounded by {\small$$\mathcal{O}\Big(\sum_{Q=0}^{N}(1-\lambda\mu)^{2N-2Q}\|y^Q-y_{(x)}^\ast\|^2+\|y^N-y_{(x)}^\ast\|^2\Big).$$} 
\hspace{-0.12cm}Note from \cref{eq:intermidB} that although the optimality gap $\|y^Q-y_{(x)}^\ast\|$ can be large for small $Q$ (which is induced by our ITD-based construction), the coupling factor $(1-\lambda\mu)^{2N-2Q}$ still makes the overall bound to be small, and this validates the design principle of our AggITD estimator.      
Then, unconditioning on $x, y^N$, incorporating the convergence bounds on the iterates $y^Q$ with intrinsic client drift, we derive the final estimation bounds on AggITD. 
The following proposition  characterizes the estimation variance of the global indirect hypergradient estimate $\widetilde h_i^I(x)$ and the local hypergradient estimate at iteration $\upsilon$ of client $i$. 

\begin{proposition}\label{propecevd}
Suppose Assumptions~\ref{assum:muconvex}-\ref{high_lip} are satisfied. Set $\lambda\leq\min\{10, \frac{1}{L_g}\}$. Then, conditioning on $x,y_+$, we have
\begin{align*}
&\E\|\widetilde{h}_i^I(x)-\Bar{h}^I_i(x)\|^2\leq\sigma_h^2,
\\&\E\|\widetilde{h}_i^D(x_{\upsilon}^i,y_+)-\widetilde{h}_i^D(x_0^i,y_+)+\widetilde{h}^D(x)-\widetilde{h}^I(x)\|^2\leq D_h^2
\end{align*}
where the constants are given by $\sigma_h^2=\frac{\lambda(N+1)L_g^2M^2}{\mu}$ and $D_h^2=12M^2 + \frac{4\lambda(N+1)L_g^2M^2}{\mu}$.
\end{proposition}
\Cref{propecevd} demonstrates that the varaince of our AggITD estimation is bounded. Based on the important bias and variance characterizations in Propositions~\ref{propsfrvcc}  and~\ref{propecevd}, we next provide the total convergence and complexity analysis for the proposed FBO-AggITD algorithm. 
\subsection{Convergence and Complexity  Analysis}  
We first provide a descent lemma on the total objective  $f(x)$. 
\begin{lemma}[Objective descent]\label{lemma:descentl}
Suppose Assumptions~\ref{assum:muconvex}-\ref{bounded_variance} hold. Let $y^{\ast}=\argmin_{y}g(x,y)$. Further, we set $\lambda\leq\min\{10, \frac{1}{L_g}\}$, $\alpha^i=\frac{\alpha}{\tau_i}$ with $\tau_i\geq1$ for some positive $\alpha$ and $\beta^i=\frac{\beta}{\tau_i}$, where $\beta\leq\min\{1,\lambda,\frac{1}{6L_{g}}\}$  $\forall i \in S$. We have
\begin{align}
&\E[f(x_+)]-\E[f(x)]\nonumber
\\&\leq-\frac{\alpha}{2}\E[\|\nabla f(x)\|^2]+4\alpha^2(\sigma_h^2+\sigma_f^2)L_f^\prime+2\alpha^2M^2L_f^\prime\nonumber
\\&-\frac{\alpha}{2}(1-4\alpha L_f^\prime)\E{\footnotesize\Big\|\frac{1}{m}\sum_{i=1}^{m}\frac{1}{\tau_i}\sum_{\upsilon=0}^{\tau_i-1}\bar{h}_i^D(x_{\upsilon}^i,y_+)-\bar{h}^I(x)\Big\|^2}\nonumber
\\&+\frac{3\alpha}{2}\Big[B^I(x,y)
+\frac{M_f^2}{m}\sum_{i=1}^{m}\frac{1}{\tau_i}\sum_{\upsilon=0}^{\tau_i-1}\E[\|x_{\upsilon}^i-x\|^2]\nonumber
\\&\qquad\quad+M_f^2\E[\|y_+-y^\ast\|^2]\Big]
\end{align}
where 
the estimation bias $B^I(x,y)$ is  defined in \Cref{propsfrvcc}, and the expected quantities $\Bar{h}^I(x)={\E}[\widetilde{h}^I(x)|x,y_+]$, $\Bar{h}_i^D(x_{\upsilon}^i,y_+)={\E}[\widetilde{h}_i^D(x_{\upsilon}^i,y_+)|x_\upsilon^i]$
\end{lemma}
Note from \Cref{lemma:descentl} that the bound on the total objective descent contains three error terms including the FHC bias $B^I(x,y)$, which is handled by \Cref{propsfrvcc}, the lower-level estimation error $\mathbb{E}\|y_+-y^*\|^2$, which is handled by the descent lemma on the lower-level objective function $g(x,\cdot)$, and 
the upper-level client drift $\sum_{i=1}^{m}\frac{1}{\tau_i}\sum_{\upsilon=0}^{\tau_i-1}\E[\|x_{\upsilon}^i-x\|^2]$. Also note that the bias error $B^I(x,y)$ contains the lower-level initialization gap $\mathbb{E}\|y-y^*\|^2$, which is characterized by the following lemma.
% Thus, the following lemma is to characterize this lower-level initialization gap.
\begin{lemma}[Lower-level initialization gap under warm start]\label{le:warm_start}
Suppose Assumptions~\ref{assum:muconvex}-\ref{bounded_variance} hold. Let $y^{\ast}=\argmin_{y}g(x,y)$ and  $y_{(x_+)}^{\ast}=\argmin_{y}g(x_+,y)$. Further, set $\alpha^i=\frac{\alpha}{\tau_i}$ with $\tau_i\geq1$ with some  $\alpha>0$, $\forall i \in S$. Then, we have 
% \\
\begin{align*}
\begin{split}
\E[\|&y_+-y_{(x_+)}^\ast\|^2]
\\\leq& b_1(\alpha)\E\Big[\Big\|\frac{1}{m}\sum_{i=1}^{m}\frac{1}{\tau_i}\sum_{\upsilon=0}^{\tau_i-1}\big(\Bar{h}_i^D(x_{v}^i,y_+)-\bar{h}^I(x)\big)\Big\|^2\Big]
\\&+b_2(\alpha)\E[\|y_+-y^\ast\|^2]+b_3(\alpha)(2\sigma_h^2+2\sigma_f^2+M^2)
\end{split}
\end{align*}
where the constants are given by 
% \begin{align*}
% \begin{split}
$b_1(\alpha)=4L_y^2\alpha^2+\frac{L_y^2\alpha^2}{4\gamma}+\frac{2L_{yx}\alpha^2}{\eta},
b_2(\alpha)=1+4\gamma+\frac{\eta L_{yx}D_h^2\alpha^2}{2},
b_3(\alpha)=4\alpha^2L_y^2+\frac{2L_{yx}\alpha^2}{\eta}$ 
% \end{split}
% \end{align*}
with a flexible parameter $\gamma>0$.
\end{lemma}
As shown in the above \Cref{le:warm_start}, the lower-level initialization gap contains a hypergradient estimate norm $\mathcal{O}(\alpha^2)\E\big\|\frac{1}{m}\sum_{i=1}^{m}\frac{1}{\tau_i}\sum_{\upsilon=0}^{\tau_i-1}\big(\Bar{h}_i^D(x_{v}^i,y_+)-\bar{h}^I(x)\big)\big\|^2$, which is dominated by the same hypergradient norm with the factor  $\Theta(-\alpha)$ in \Cref{lemma:descentl} for the stepsize $\alpha$ small enough. Then, the remaining step is to upper bound the upper-level client drift $\mathbb{E}\|x_\upsilon^i-x\|^2$.
\begin{lemma}[Upper client drift]\label{client_dripfscs}
Suppose Assumptions~ \ref{assum:muconvex}-\ref{bounded_variance} are satisfied. Set $\lambda\leq\min\{10, \frac{1}{L_g}\}$, $\alpha^i = \frac{\alpha}{\tau_i}\; and\; \beta^i = \frac{\beta}{\tau_i},\tau_i\geq1\ where\ \alpha\leq\frac{1}{324M_f^2+6M_f}\leq\frac{1}{6M_f},\; \beta\leq\min\{1,\lambda, \frac{1}{6L_{g}}\}$  $\forall i \in S$. Recall the definitions of $y^{\ast}=\argmin_{y}g(x,y)$, $\Bar{h}(x)={\E}[\widetilde{h}(x)|x,y_+]$. Then, we have
\begin{align*}
\E[\|x_{\upsilon}^i-x\|^2]\leq18&\tau_i^2(\alpha^i)^2\Big[3M_f^2\E[\|y_+-y^\ast\|^2]
\\+3\E[\|&\nabla f(x)\|^2]+B^I(x,y)+3\sigma_h^2+6\sigma_f^2\Big]
\end{align*}
where the bias $B^I(x,y)$ is  defined in \Cref{propsfrvcc}.
% $\alpha_1(N), \alpha_2(N), \alpha_3(N)$ are defined in \Cref{propsfrvcc}.
\end{lemma}
It can be seen from \Cref{client_dripfscs} that the upper-level client drift is bounded by the lower-level estimation error $\mathbb{E}\|y_+-y^*\|^2$,  the total gradient norm $\mathbb{E}\|\nabla f(x)\|^2$  and the hypergradient estimation bias $B^I(x,y)$, which can be addressed by the descent lemmas on $y$ and $x$ (i.e., \Cref{lemma:descentl}) and \Cref{propsfrvcc} for the stepsize $\alpha^i$ small enough. By combining the above lemmas, we next provide the general convergence analysis.

\begin{figure*}[ht] 
\centering 
\includegraphics[width=4.5cm]{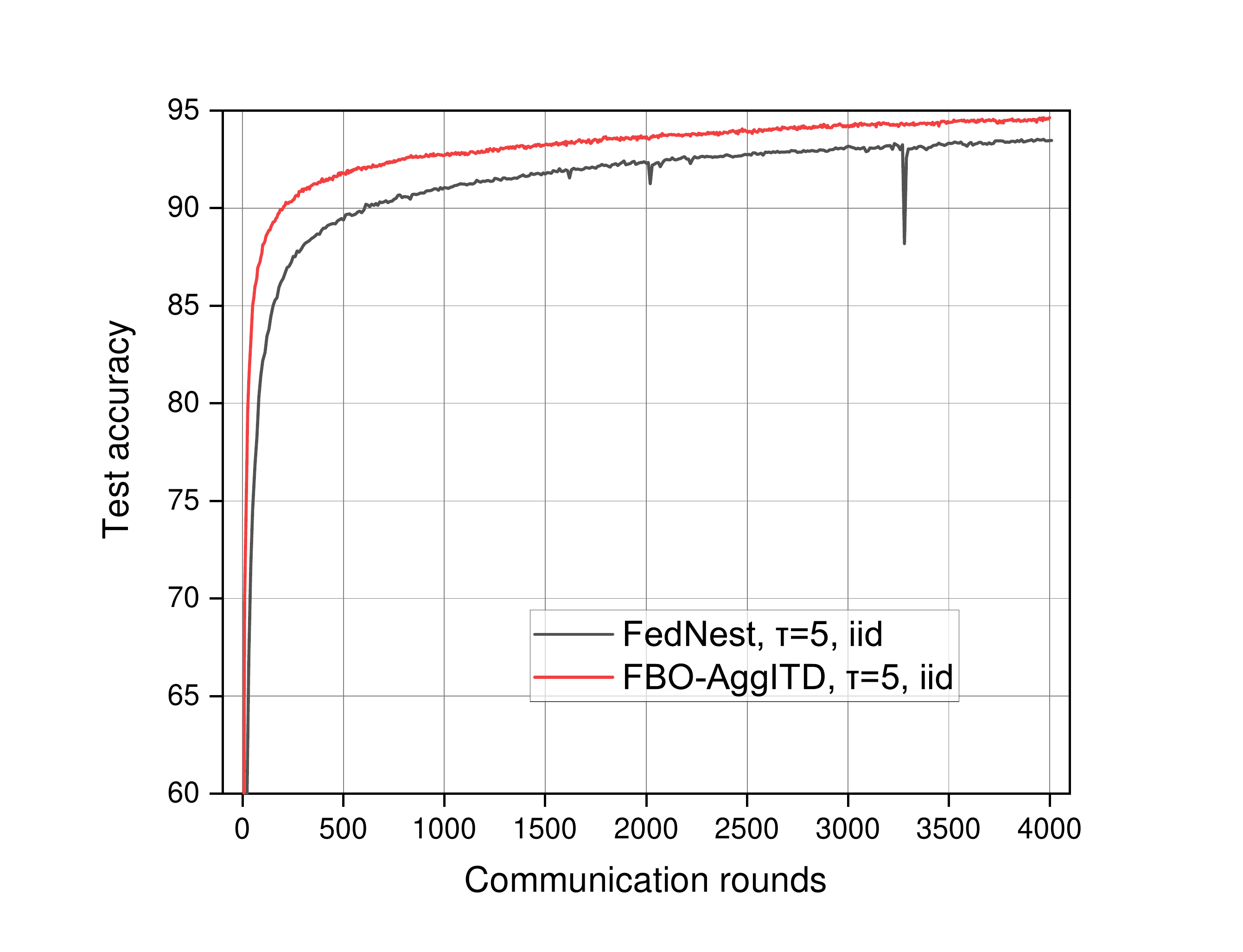}
\hspace{-0.8cm}
\includegraphics[width=4.5cm]{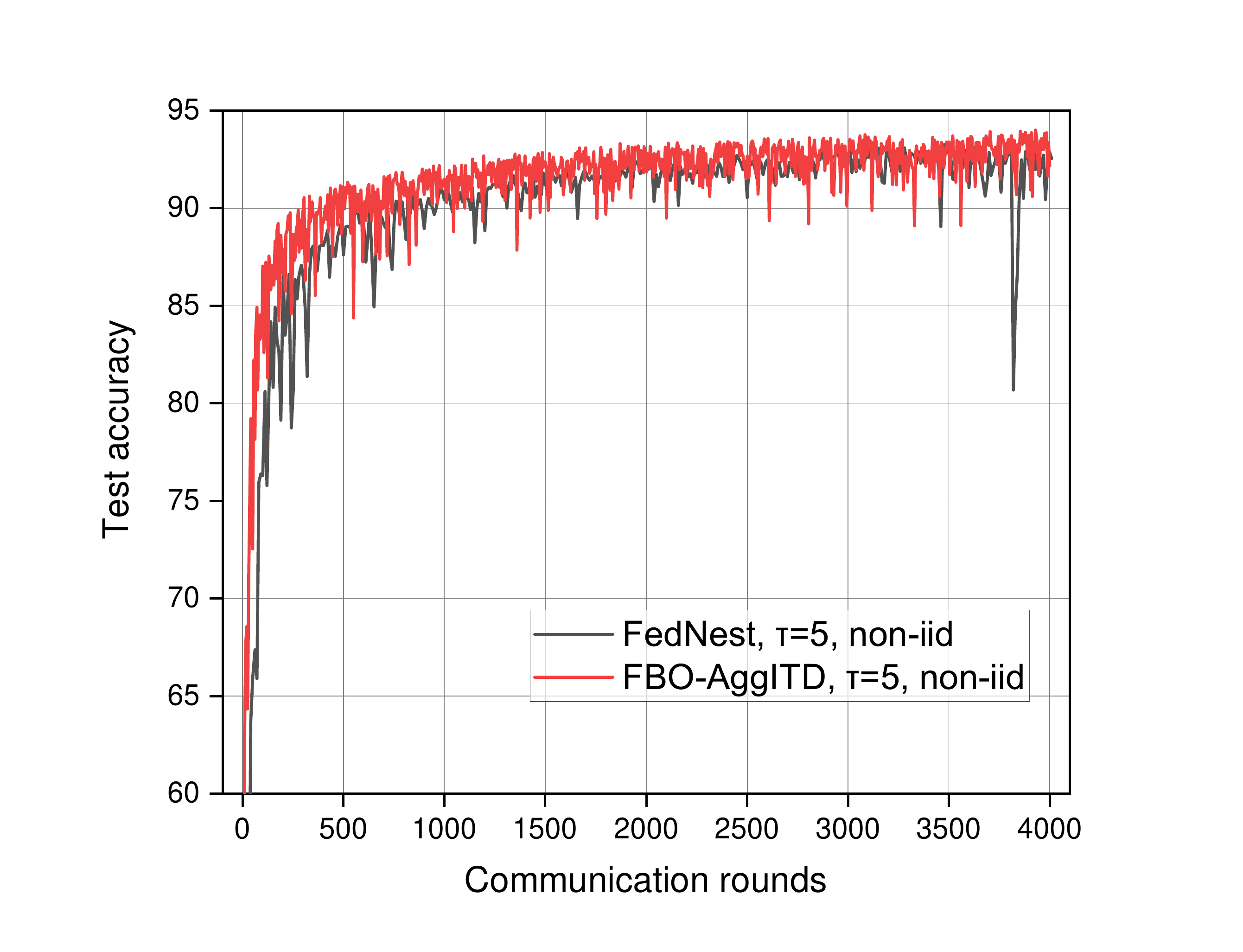}
\hspace{-0.8cm}
\includegraphics[width=4.5cm]{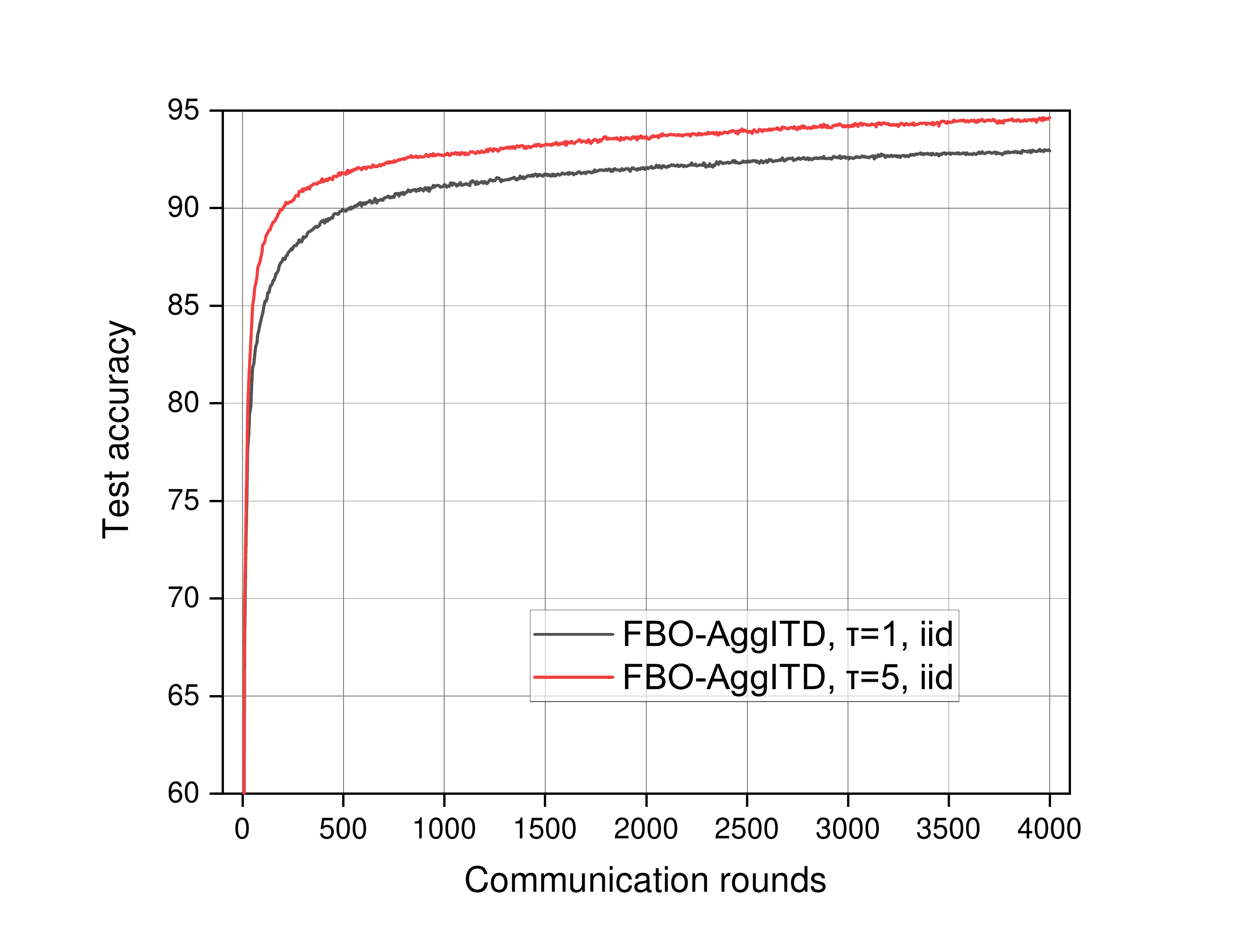}
\hspace{-0.8cm}
\includegraphics[width=4.5cm]{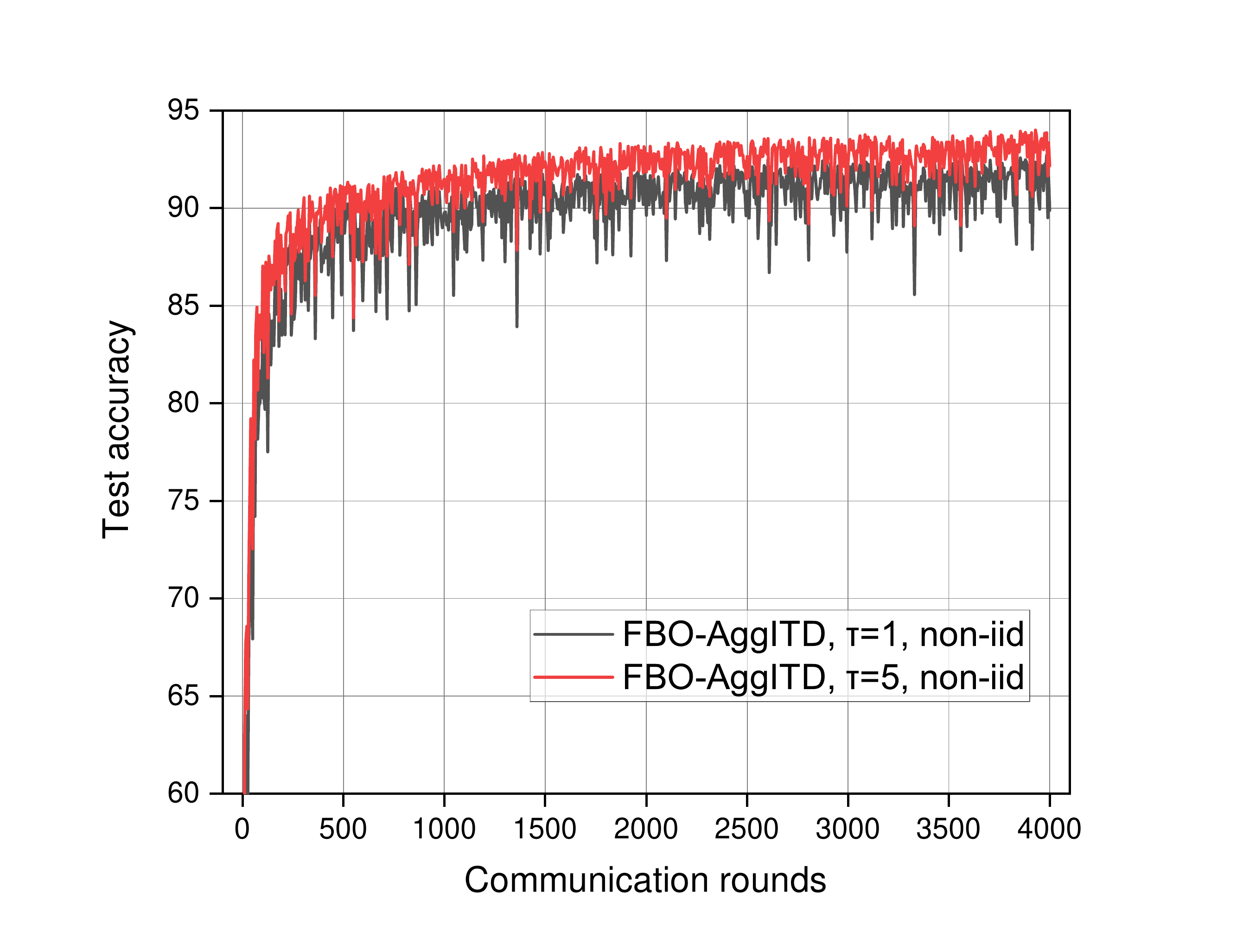}
\hspace{-0.8cm}
\vspace{-0.2cm}
\caption{Hyper-representation on MNIST dataset with a 2-layer MLP with SVRG-type optimizer. Left two plots: comparison of FBO-AggITD and FedNest~\citep{tarzanagh2022fednest} in the i.i.d.~and non-i.i.d.~cases. Right two plots: the impact of the number $\tau$ of local update steps on FBO-AggITD.} 
\label{Fig.main1} 
% \vspace{-0.2cm}
\end{figure*}

\begin{table*}[ht]
\centering
\small
\begin{tabular}{cclccclcl}
\hline
Algorithm                & \multicolumn{2}{c}{Comm\_rounds/Outer\_itr}            & Data                     & Outer\_ep & \multicolumn{2}{c}{Comm\_rounds (90\%)} & \multicolumn{2}{c}{Final Accuracy} \\ \hline
\multirow{4}{*}{FedNest} & \multicolumn{2}{c}{\multirow{4}{*}{2N+T+3}} & \multirow{2}{*}{IID}& $\tau$=1& \multicolumn{2}{c}{1630}     & \multicolumn{2}{c}{91.68\%}  \\
 & \multicolumn{2}{c}{}                        &                          & $\tau$=5       & \multicolumn{2}{c}{610}                 & \multicolumn{2}{c}{93.48\%}        \\ \cline{4-9}  & \multicolumn{2}{c}{}                        & \multirow{2}{*}{NON-IID} & $\tau$=1       & \multicolumn{2}{c}{1380}                & \multicolumn{2}{c}{91.46\%}        \\
   & \multicolumn{2}{c}{} &  & $\tau$=5       & \multicolumn{2}{c}{760}                 & \multicolumn{2}{c}{92.87\%}        \\ \hline
\multirow{4}{*}{FBO-AggITD} & \multicolumn{2}{c}{\multirow{4}{*}{2N+3}}   & \multirow{2}{*}{IID}     & $\tau$=1       & \multicolumn{2}{c}{530}                 & \multicolumn{2}{c}{92.94\%}        \\
  & \multicolumn{2}{c}{}                        &                          & $\tau$=5       & \multicolumn{2}{c}{195}                 & \multicolumn{2}{c}{94.61\%}        \\ \cline{4-9} 
                         & \multicolumn{2}{c}{}                        & \multirow{2}{*}{NON-IID} & $\tau$=1       & \multicolumn{2}{c}{520}                 & \multicolumn{2}{c}{92.67\%}        \\
                         & \multicolumn{2}{c}{}                        &                          & $\tau$=5       & \multicolumn{2}{c}{305}                 & \multicolumn{2}{c}{93.88\%}        \\ \hline
\end{tabular}
\caption{Quantitative comparison between FBO-AggITD and FedNest.} 
\label{Table1}
\vspace{-0.25cm}
\end{table*}

\begin{theorem}\label{Theorem2-mainB}
Suppose Assumptions~\ref{assum:muconvex}-\ref{bounded_variance} are satisfied.   
Set $\lambda\leq\min\{10, \frac{1}{L_g}\}$, $\alpha_k^i=\frac{\alpha_k}{\tau_i}$ and $\beta_k^i=\frac{\beta_k}{\tau_i}$ for  $i \in S$. Choose parameters such that 
% \begin{align*}
$\alpha_k=\min\{\bar{\alpha}_1,\bar{\alpha}_2,\bar{\alpha}_3, \frac{\bar{\alpha}}{\sqrt{K}}\},
\beta_k\in\big[\max\big\{\frac{\bar{\beta}\alpha_k}{N},\frac{\lambda}{10}\big\},\min\big\{1,\lambda,\frac{1}{6L_g}\big\}\big]$, 
% \end{align*}
where $\bar{\alpha}_1, \bar{\alpha}_2, \bar{\alpha}_3, \bar{\alpha}$ and $ \bar{\beta}$ are constants independent of $K$, whose specific forms are given in \Cref{proofoftheorem2}. Then, the outputs of the proposed FBO-AggITD algorithms satisfy 
\begin{align*}
\frac{1}{K}\sum_{k=0}^{K-1}&\E[\|\nabla f(x_k)\|^2]=\mathcal{O}\bigg(\frac{1}{\min\{\bar{\alpha}_1,\bar{\alpha}_2,\bar{\alpha}_3\}{K}}+\frac{1}{\bar\alpha\sqrt{K}}
\\&+\frac{\bar{\alpha}\max\{c_0,c_1\sigma_h^2,c_2,c_3\}}{\sqrt{K}}+(1-\lambda\mu)^{2N}\bigg),
\end{align*}
where $c_0, c_1, c_2$, and $c_3$ are positive constants independent of K, whose complete forms are given in~\Cref{proofoftheorem2}.
\end{theorem}
% \Cref{Theorem2-mainB} provides a general convergence analysis for the proposed FBO-AggITD method.
By specifying the parameters $N$ and $\bar\alpha$ properly, we obtain the following complexity results. 

\begin{corollary}\label{corollary2}
Under the same setting as in \Cref{Theorem2-mainB}, if we choose $N=\mathcal{O}(\kappa_g)$, $\bar{\alpha}=\mathcal{O}(\kappa_g^{-4})$, then we have 
\begin{align*}
\frac{1}{K}\sum_{k=0}^{K-1}\E[\|\nabla f(x_k)\|^2]=\mathcal{O}(\frac{\kappa_g^4}{K}+\frac{\kappa_g^{4}}{\sqrt{K}})
\end{align*}
To achieve an $\epsilon$-accurate stationary point, 
% it requires at most 
% $K=\mathcal{O}(\kappa_g^8\epsilon^{-2})$
% and  
the total number of samples required by FBO-AggITD is $\mathcal{O}(\kappa_g^9\epsilon^{-2})$.
\end{corollary}

As shown in \Cref{corollary2}, the overall sample complexity (i.e., the total number of data samples required to achieve an $\epsilon$-accurate stationary point) of our 
 FBO-AggITD  is $\mathcal{O}(\kappa^9\epsilon^{-2})$, which matches the sample complexities of stocBiO~\citep{ji2021bilevel}, BSA~\citep{ghadimi2018approximation} and ALSET~\citep{chen2021closing} in the non-federated bilevel optimization and FedNest~\citep{tarzanagh2022fednest} in the federated setting despite the data heterogeneity. Note that our method uses only $(2N+3)/(2N+T+3)$ communication rounds of FedNest (shown in \Cref{tab:compare}) at each outer iteration. As a result, in theory, our method achieves a constant-level improvement over FedNest. To improve the dependence on $\epsilon$, we suspect that the server-level variance reduction or periodic averaging can help, but this goes beyond the focus of this paper. We are happy to leave it for future study.
 
 % our method is much simpler with lower communication costs. 

% \vspace{0.1cm}
% \noindent{\bf Comparison of different FHC estimators.} Different from the comparison between ITD and AID in the non-federated setting that AID has  a lower computational and memory cost than ITD, our analysis here demonstrates the great advantage of ITD-based method over the AID-based approach to achieve a lower communication cost as well as a much simpler implementation.
% We suspect that our proposed ITD-based construction can also provide a simple and communication-efficient implementation in other distributed scenarios such as decentralized bilevel optimization~\citep{chen2022decentralized,yang2022decentralized}, asynchronous bilevel optimization over directed network~\citep{yousefian2021bilevel}, federated learning with optimized weighted nodes~\citep{huang2022federated}          

\vspace{-0.2cm}
\section{Experiments}\label{sec:exp}
In this section, we compare the performance of the proposed FBO-AggITD method with FedNest and LFedNest in \citealt{tarzanagh2022fednest} on a hyper-representation problem. 
% similar to \citep{franceschi2018bilevel} and investigate the influence of different attributes.
% As one of the nested problems with a bilevel structure, the inner problem attributes to optimize the header on training data while the outer problem optimizes the backbone on the validation data for feature representation.
Following the problem setup in \citealt{franceschi2018bilevel}, we use a 2-layer
multilayer perceptron (MLP) as the backbone, where the hidden layer is optimized at the upper-level problem and the head is optimized at the lower-level problem. We study the impact of data heterogeneity on the comparison algorithms by considering both the i.i.d.~and non-i.i.d.~ways of data partitioning of MNIST, following the setup in~\citealt{mcmahan2017communication}.  

% In this experiment, we use a 2-layer
% multilayer perceptron (MLP) following FedNest \citep{tarzanagh2022fednest} with 157,000 parameters optimizing the hidden layer as the outer problem and 2010 parameters optimizing the outer layer as the inner problem. 

% Moreover, we also study the impact of both i.i.d and non-i.i.d ways of the partitioning of MNIST data, which follows from the FedAvg \citep{mcmahan2017communication}, on the performance of the algorithms.

The first two plots in \Cref{Fig.main1} compare our FBO-AggITD method with FedNest in both i.i.d.~and non-i.i.d.~setups with $\tau=5$, respectively. It can be seen  that FBO-AggITD 
converges much faster than FedNest, and achieves a higher test accuracy with much fewer communication rounds. In the non-i.i.d.~case also shows that FBO-AggITD is more stable with lower variance than FedNest. The last two plots in \Cref{Fig.main1} show that local updates are useful to improve communication efficiency and stabilize the training. In \Cref{Table1}, it can be seen that to achieve an accuracy of $90\%$, our FBO-AggITD uses more than $2$-$3$ times fewer communication rounds than FedNest, in both the i.i.d.~and non-i.i.d.~cases and in addition, for all four setups, FBO-AggITD achieves a higher final test accuracy than FedNest. 

\vspace{-0.3cm}
\begin{figure}[H]  
\centering
\includegraphics[width=7cm]{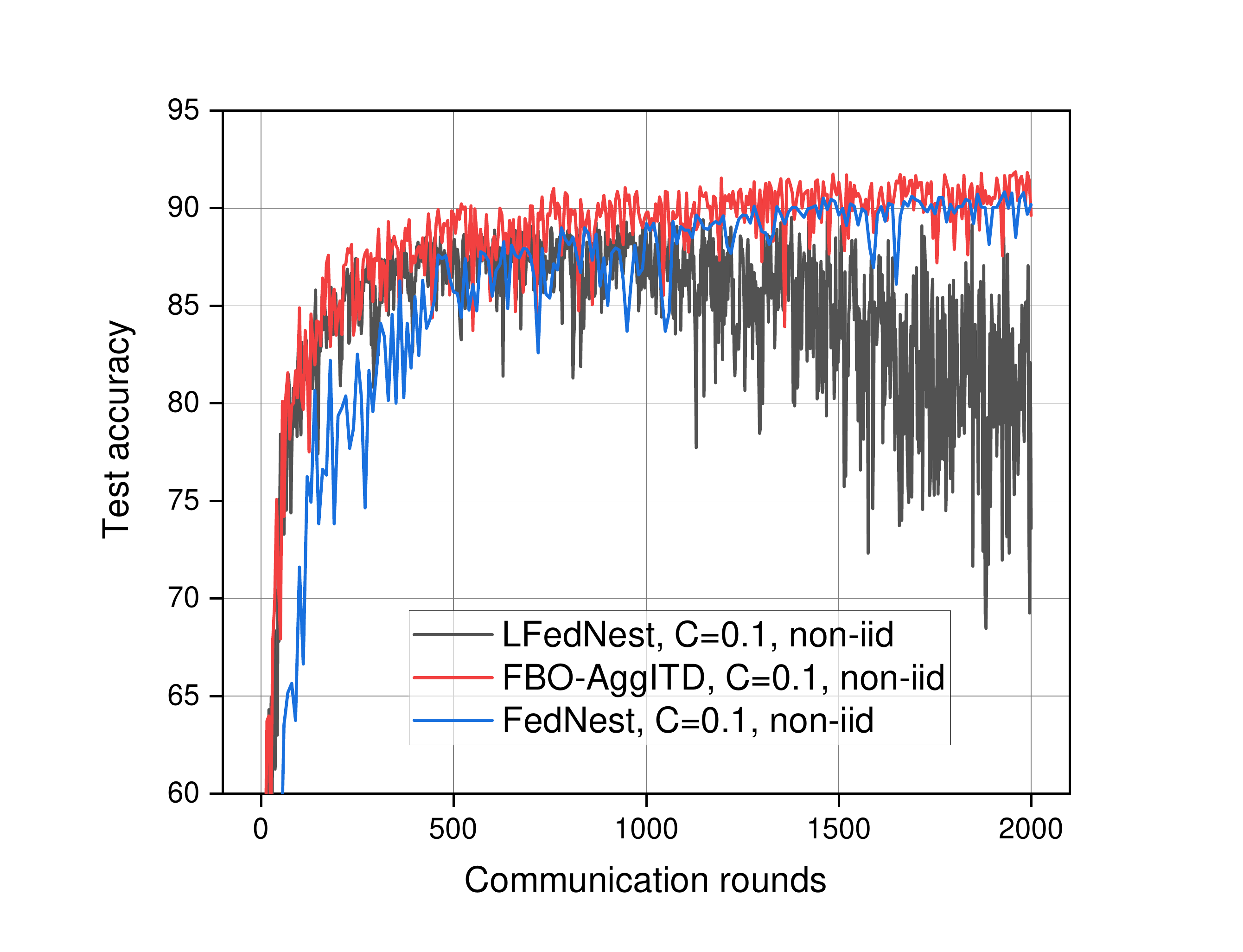}
\hspace{-0.7cm}
% \includegraphics[width=6cm]{icml2023/ratios5.png}
% \hspace{-1cm}
\includegraphics[width=7cm]{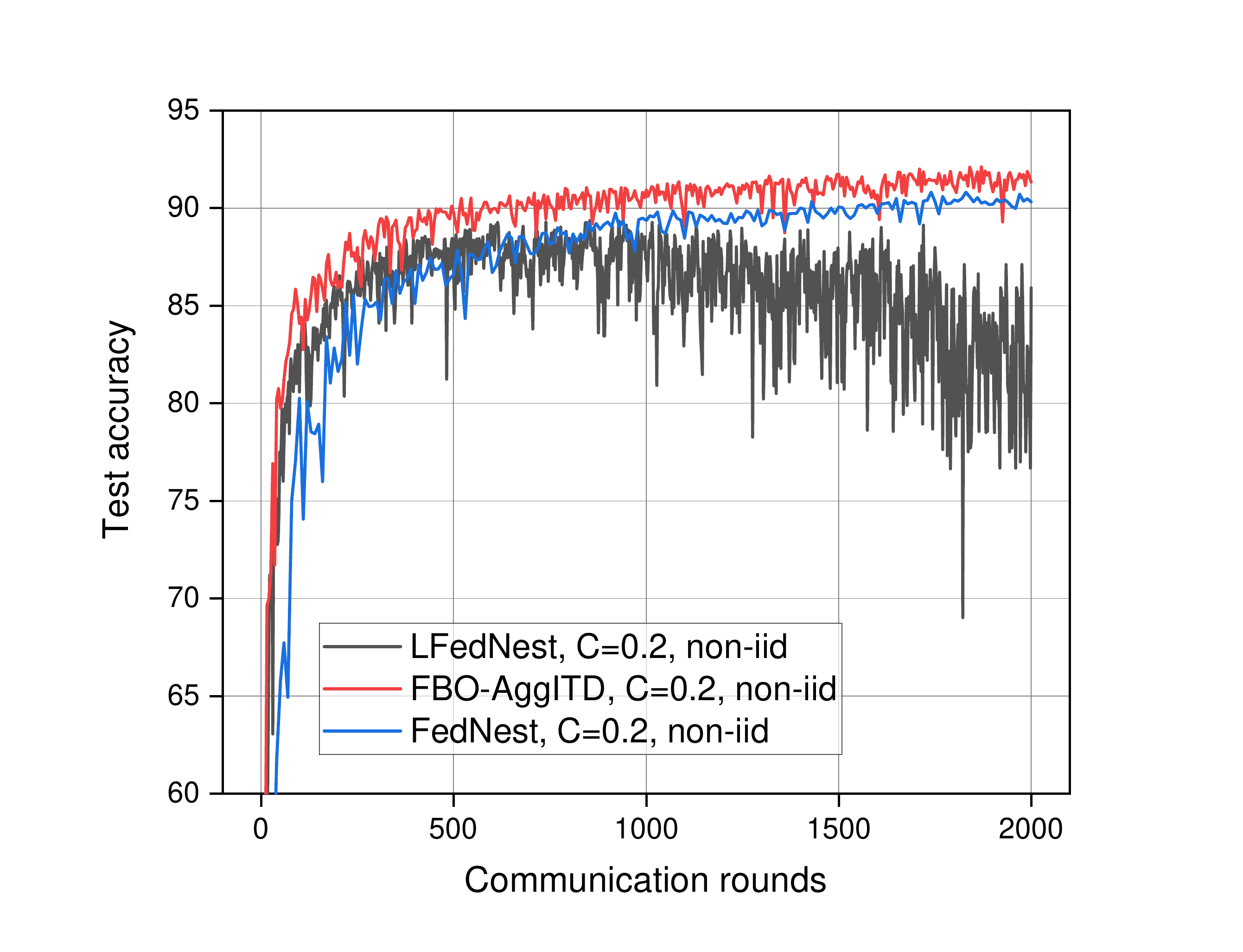}
\vspace{-0.3cm}
\caption{Comparison under different client participation ratios.} 
\label{Fig.ratios} 
\vspace{-0.2cm}
\end{figure}

In \Cref{Fig.process} and \Cref{Fig.ratios}, we compare the performance of our FBO-AggITD, FedNest, and LFedNest (which uses a fully local AID-based hypergradient estimator) given different client participation ratios (denoted as $C$) in the non-i.i.d.~setting. It can be seen that FBO-AggITD outperforms the other two algorithms with higher communication efficiency and higher accuracy. Note that LFedNest has the largest variance and the lowest accuracy, and this validates the importance of federated hypergradient computation. All above experiments use SVRG-type optimizer which outperforms the SGD-type optimizer, shown in \cref{Fig.sgd}. 
\vspace{-0.3cm}
\begin{figure}[h]  
\centering
\includegraphics[width=7cm]{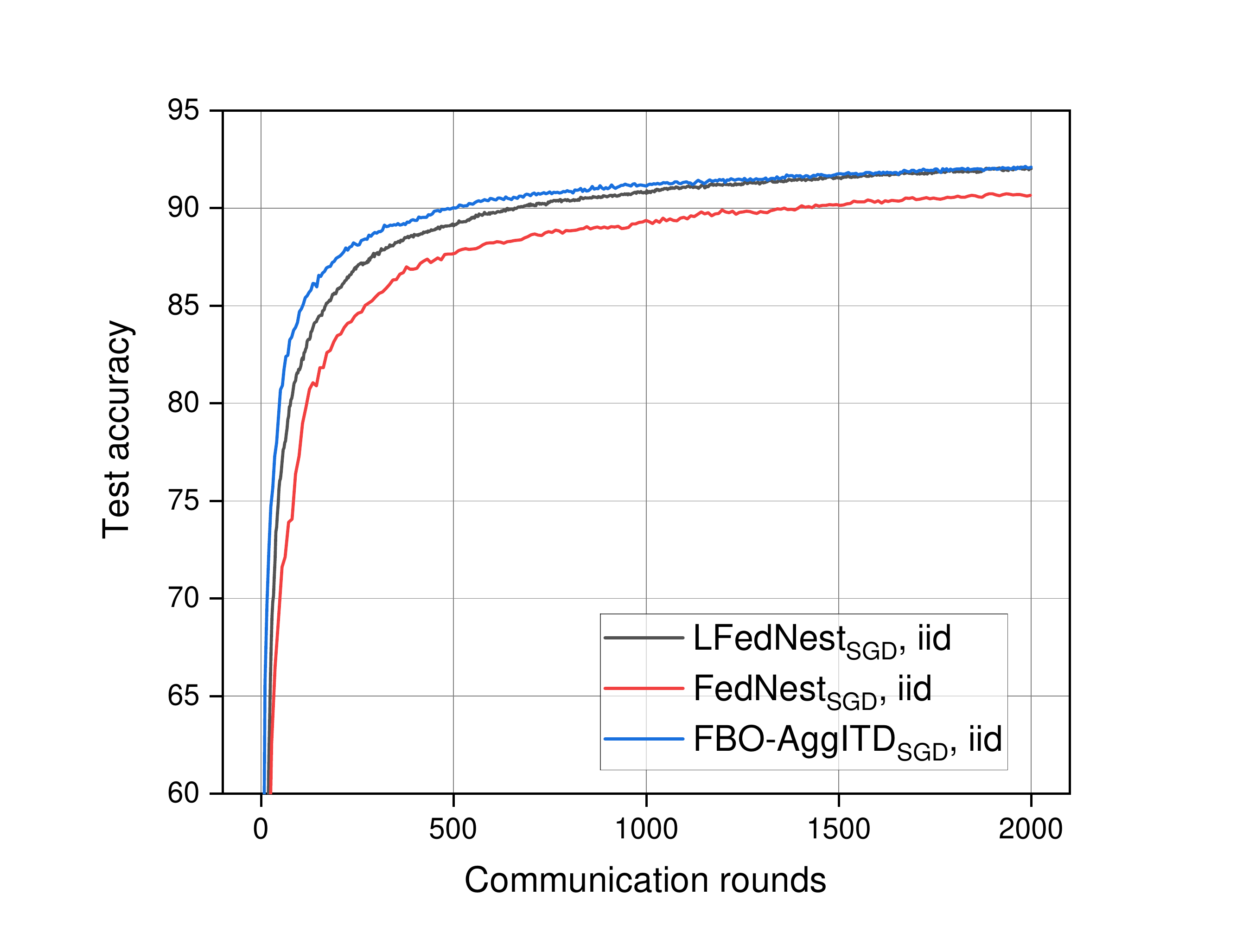}
\hspace{-0.7cm}
\includegraphics[width=7cm]{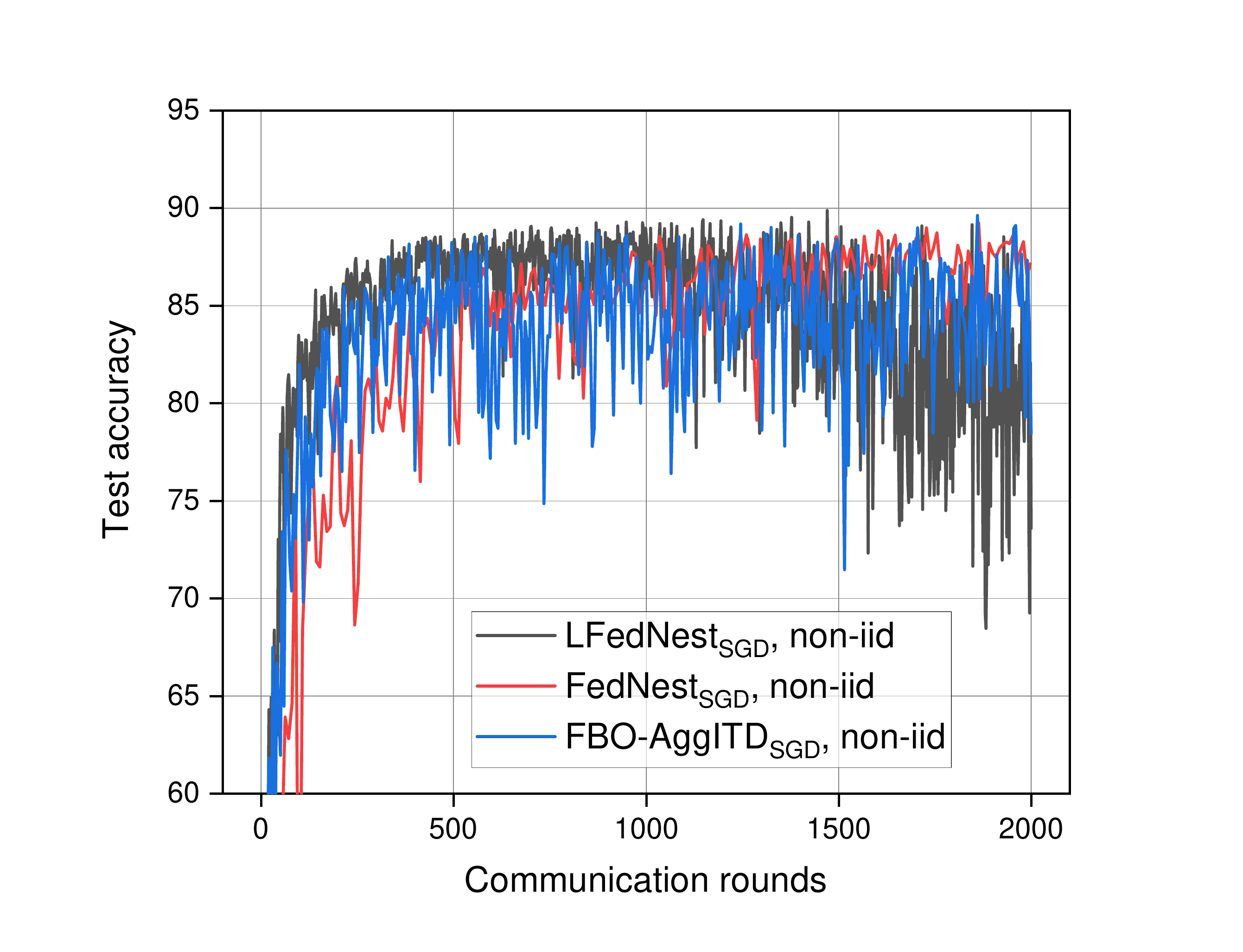}
\vspace{-0.3cm}
% \hspace{-5cm}
\caption{Performance using SGD-type optimizer.}\label{Fig.sgd} 
\vspace{-0.2cm}
\end{figure}

\Cref{Fig.sgd} compares the performance of FBO-AggITD, FedNest and LFedNest when the {\bf One-Round-Lower} uses the SGD-type FedAvg methed. In the both i.i.d. and non-i.i.d.~settings, our method (which is defined as FBO-AggITD$_{SGD}$) still performs the best with the fastest convergence rate w.r.t.~the number of communication rounds. Another observation is that using the the SGD-type lower-level solver introduces a larger variance and fluctuation than the SVRG-type optimizer, by comparing \Cref{Fig.main1} and \Cref{Fig.sgd}. This validates the importance of variance reduction in mitigating the impact of the client drift on the convergence performance. 

Finally, \Cref{Fig.cifarcompare} shows the performance of FBO-AggITD on CIFAR-10 with MLP/CNN network in the i.i.d. setting. We found that FedNest could not converge in this task after an extensive grid search on hyperparameters. However, our method can converge with both MLP and CNN backbones. However, the test accuracy is not satisfactory here. We suspect that it is because the objective function in hyper-representation is not good for federated setting, and a more careful network architecture should be designed for more challenging datasets. We would like to leave this for the future work.

% For the future work, we plan to further improve the test performance of our methods 

\begin{figure}[ht]  
\centering
\includegraphics[width=9cm]{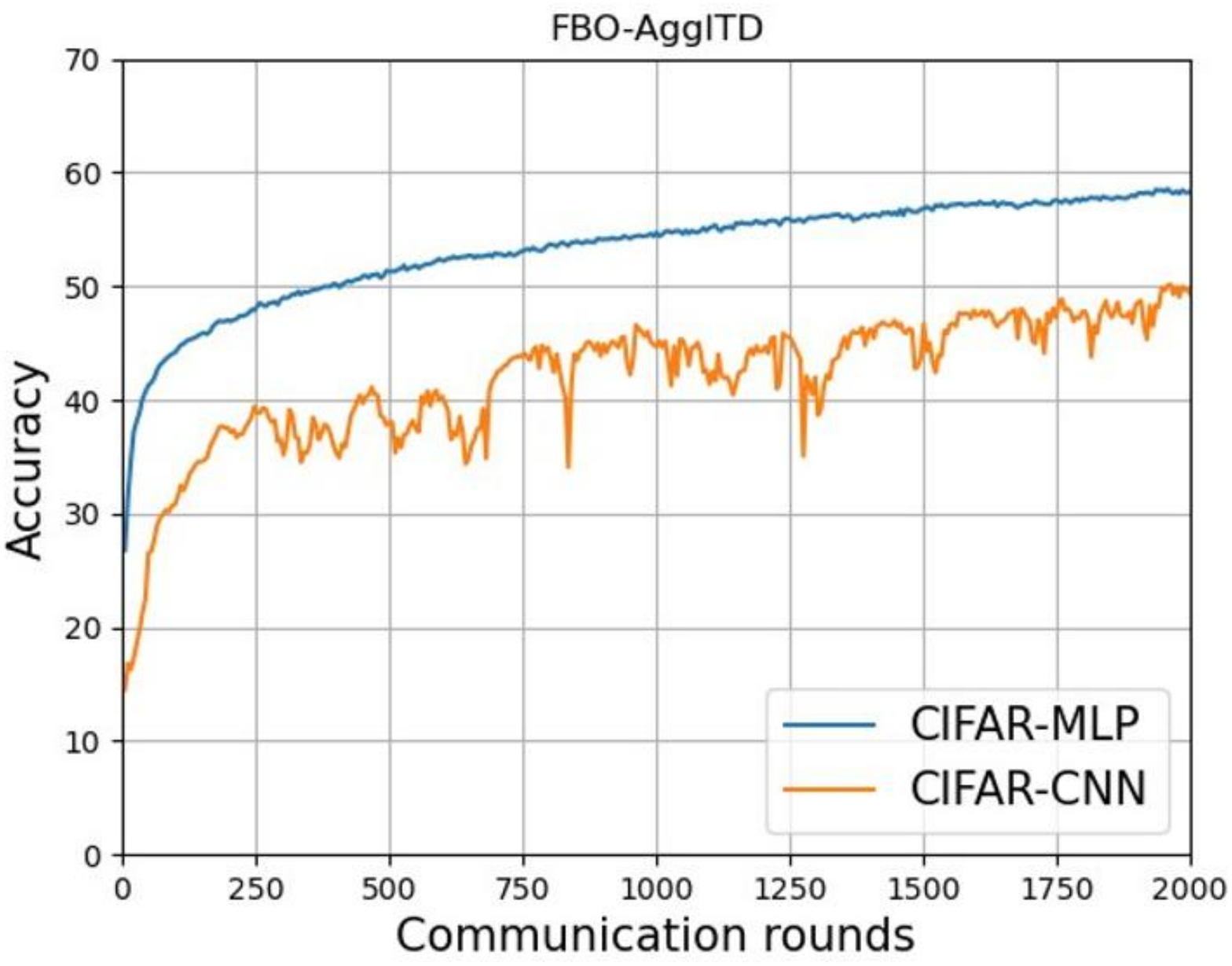}
\caption{Performance on CIFAR-10 with MLP and CNN.} 
\label{Fig.cifarcompare} 
\end{figure}

% \begin{figure*}[!ht] 
% \centering 
% \includegraphics[width=4.75cm]{icml2023/expfigure1.png}
% \hspace{-0.8cm}
% \includegraphics[width=4.75cm]{icml2023/expfigure2.png}
% \hspace{-0.8cm}
% \includegraphics[wid,th=4.75cm]{icml2023/expfigure3.png}
% \hspace{-0.8cm}
% \includegraphics[width=4.75cm]{icml2023/expfigure4.png}
% \hspace{-0.8cm}
% \caption{Hyper-representation experiments of FBO-AggITD and Feoutperforms 2-layer MLP and MNIST dataset.} 
% \label{Fig.main1} 
% \end{figure*}

% \begin{figure*}[!ht]  
% \centering
% \includegraphics[width=6cm]{icml2023/ratios4.png}
% \hspace{-1cm}
% \includegraphics[width=6cm]{icml2023/ratios5.png}
% \hspace{-1cm}
% \includegraphics[width=6cm]{icml2023/ratios6.png}
% \caption{Different algorithms performance with different fractions of heterogeneous clients.} 
% \label{Fig.ratios} 
% \end{figure*}
% \usepackage{multirow}
\vspace{-0.2cm}
\section{Conclusions}
In this paper, we propose a simple and communication-efficient federated hypergradient estimator based on a novel aggregated iterative differentiation (AggITD). 
We show that the proposed AggITD-based algorithm achieves the same sample complexity as existing approaches with much fewer communication rounds on non-i.i.d.~datasets. We anticipate our new estimator can be further applied to other distributed scenarios such as decentralized bilevel optimization.
\bibliography{reference}
\bibliographystyle{ref_style}

%%%%%%%%%%%%%%%%%%%%%%%%%%%%%%%%%%%%%%%%%%%%%%%%%%%%%%%%%%%%%%%%%%%%%%%%%%%%%%%
%%%%%%%%%%%%%%%%%%%%%%%%%%%%%%%%%%%%%%%%%%%%%%%%%%%%%%%%%%%%%%%%%%%%%%%%%%%%%%%
% APPENDIX
%%%%%%%%%%%%%%%%%%%%%%%%%%%%%%%%%%%%%%%%%%%%%%%%%%%%%%%%%%%%%%%%%%%%%%%%%%%%%%%
%%%%%%%%%%%%%%%%%%%%%%%%%%%%%%%%%%%%%%%%%%%%%%%%%%%%%%%%%%%%%%%%%%%%%%%%%%%%%%%
\newpage
\onecolumn
\allowdisplaybreaks
\appendix
\vspace{0.2cm}

\noindent{\Large\bf Supplementary Materials} 
\vspace{0.2cm}

\section{Further Specifications on Experiments}
\subsection{Additional experiments}\label{additionalexp}

\textbf{Experiments on MNIST with CNN networks. } \Cref{Fig.mnistcnncompare} compares the performance of FBO-AggITD, FedNest and LFedNest on MNIST when the backbone is chosen as CNN and One-Round-Lower uses the SVRG-type method. In the non-i.i.d. setting, it turns out that both FedNest and LFedNest failed to converge depsite of a grid search for stepsizes. The grid search on inner step sizes and outer step sizes of 4 settings are [(0.003, 0.01), (0.001, 0.005), (0.0005, 0.003), (0.0003, 0.001)]. However, our method (which is defined as FBO-AggITD) can have the ability to converge in both non-i.i.d. and i.i.d. cases with high training accuracies. The inner step szie and outer step size are chosen as [0.003, 0.01] after grid search. The training accuracies after 2000 communication rounds in i.i.d. and non-i.i.d. cases are 97.6\% and 96.7\%, respectively. 
\begin{figure*}[ht]  
\centering
\includegraphics[width=5cm]{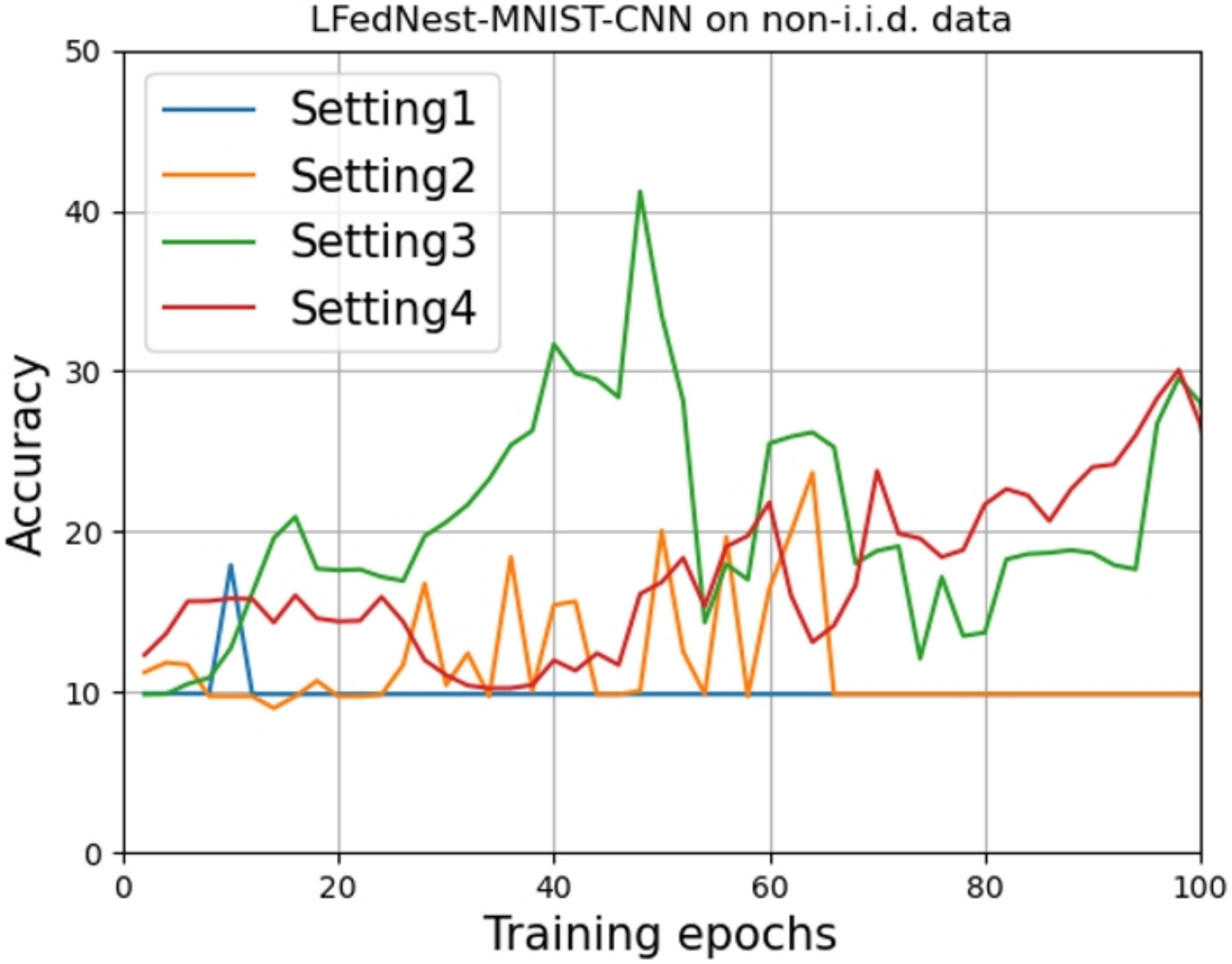}
\hspace{-0.3cm}
\includegraphics[width=5cm]{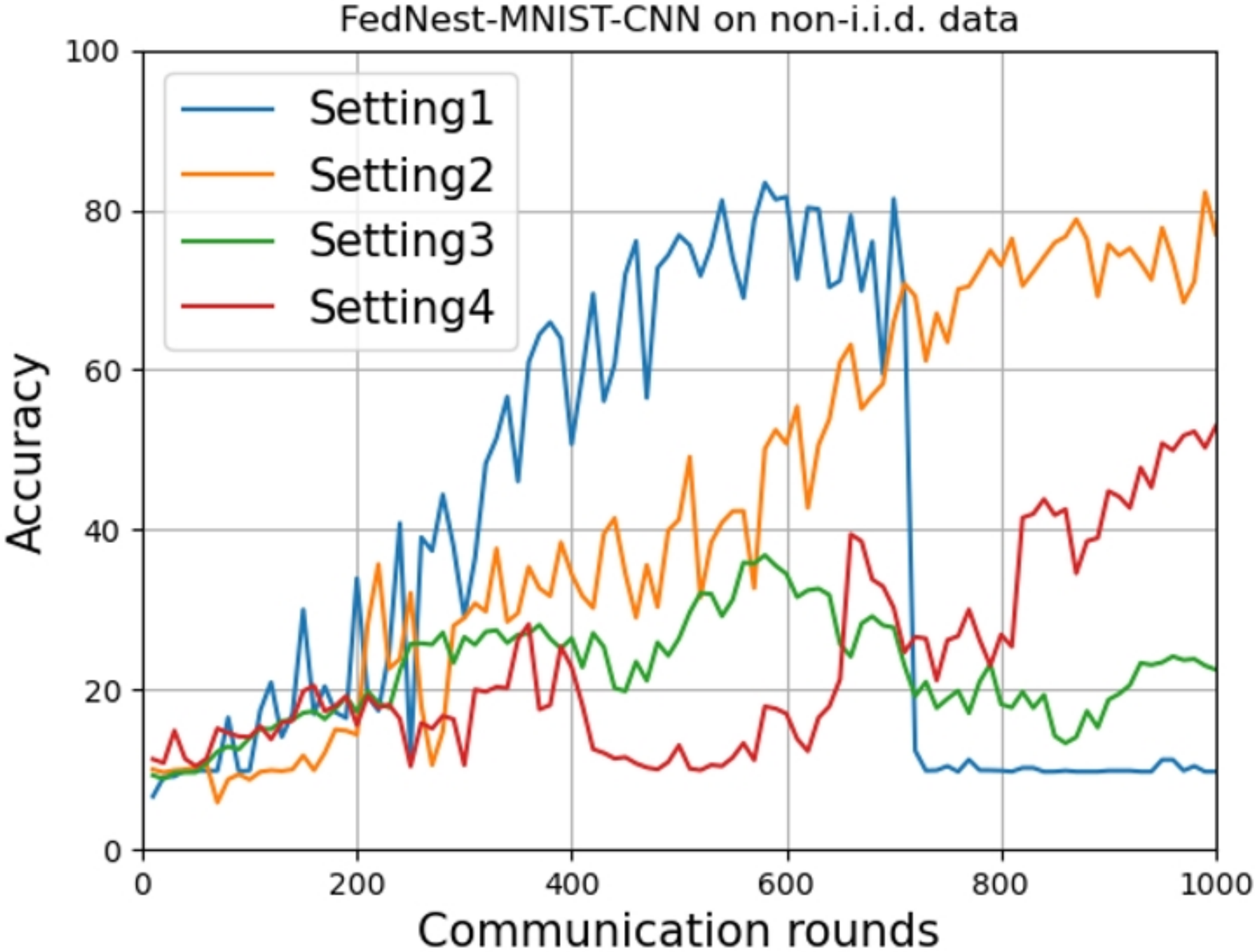}
\hspace{-0.3cm}
\includegraphics[width=5cm]{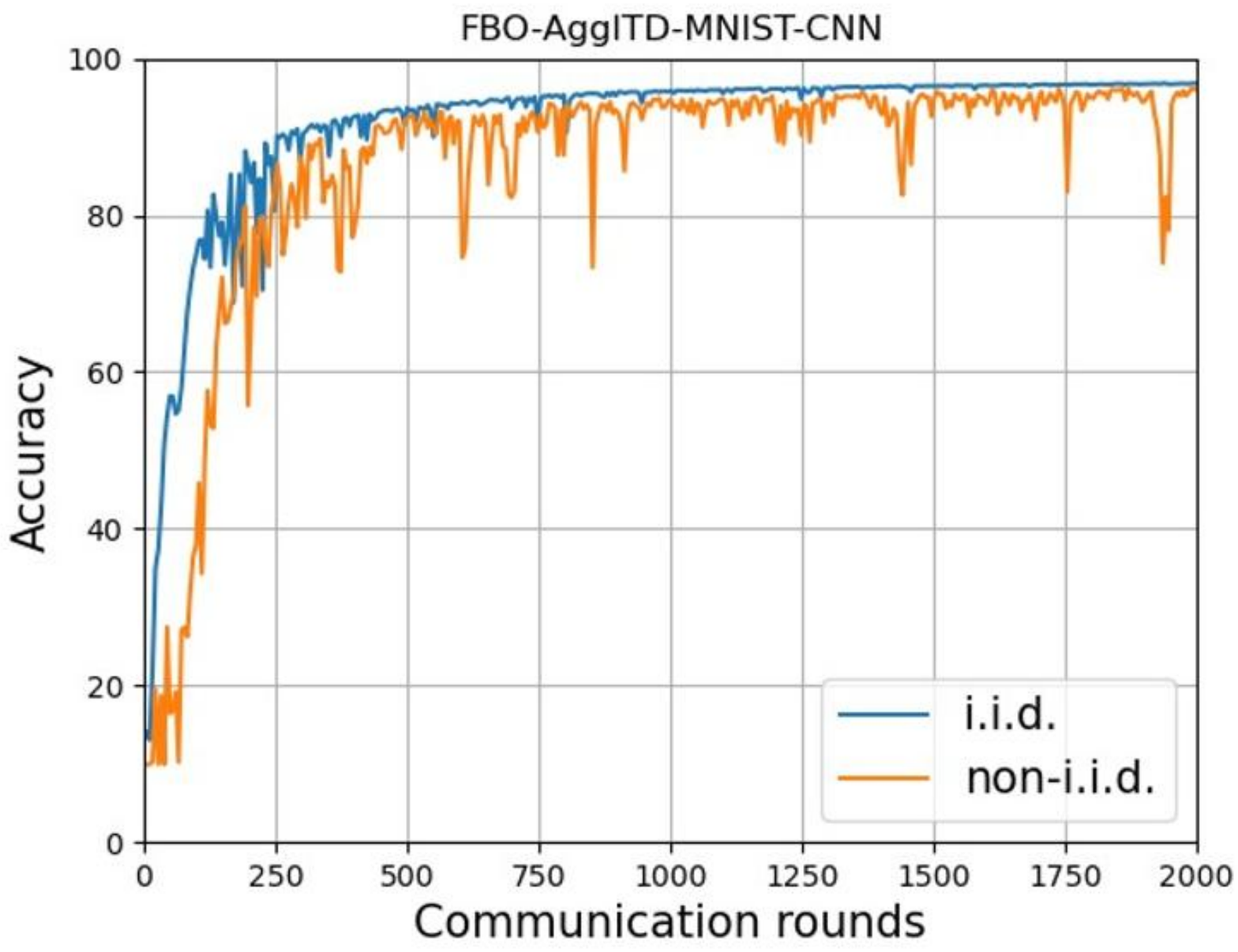}
\caption{Performance of LFedNest, FedNest and FBO-AggITD on MNIST when the backbone is chosen as CNN and One-Round-Lower uses the SVRG-type method.} 
\label{Fig.mnistcnncompare} 
\end{figure*}

\textbf{Running time comparison. } The following \Cref{Fig.runtimecompare} shows the running time comparison between FedNest, FBO-AggITD and gossip-based method, Algorithm 2 in \cite{yang2022decentralized}. Our FBO-AggITD archives a running time comparable to that of FedNest because both methods consume a similar number of gradient and Hessian-vector computations.However, our FBO-AggITD converges much faster than this gossip-based method, which is slower due to the computation of the Hessian and Jacobian matrices. Since no codes are provided in \cite{yang2022decentralized}, we wrote a code for comparison.

\begin{figure*}[ht]  
\centering
\includegraphics[width=9cm]{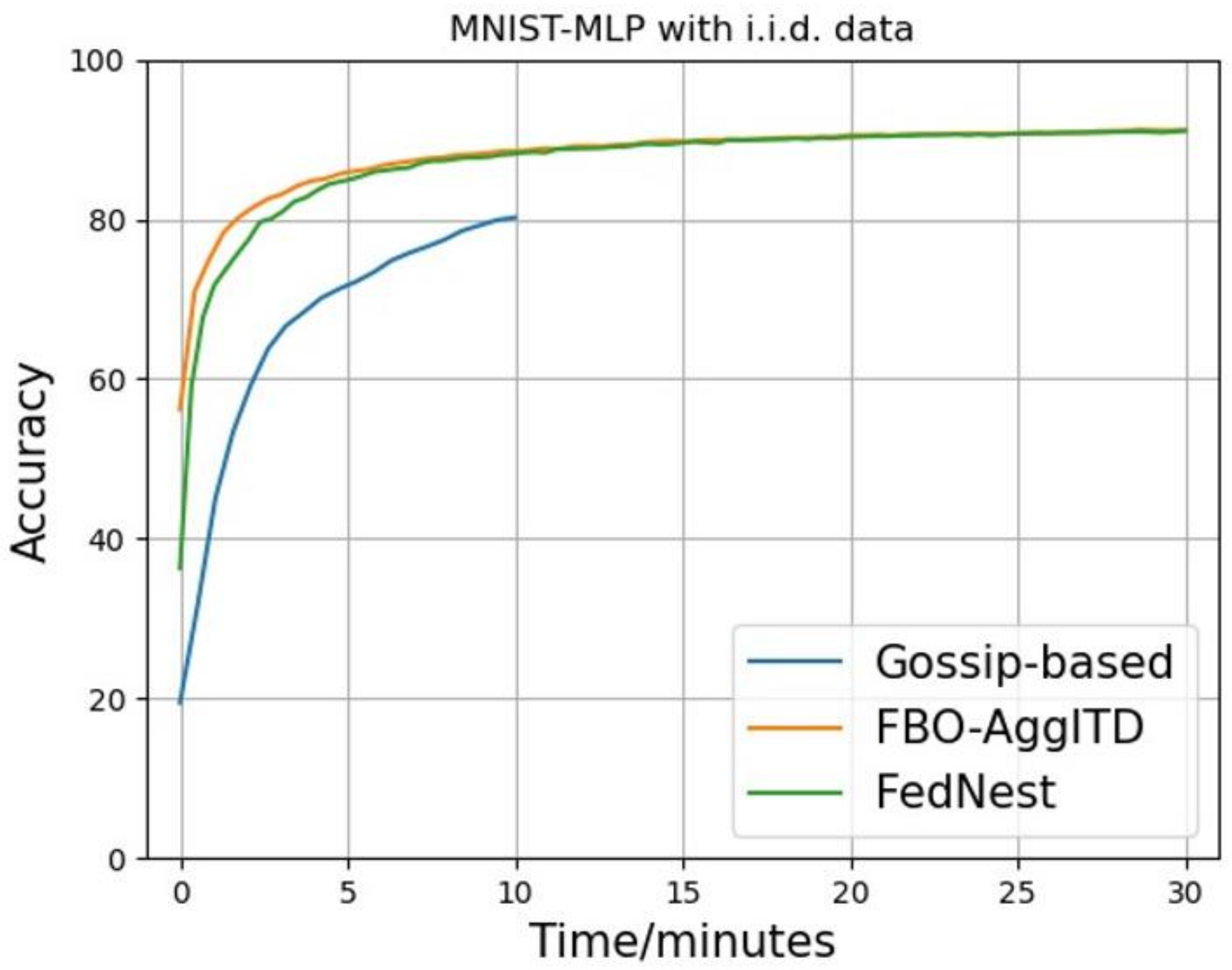}
\caption{Performance of FedNest,  FBO-AggITD and gossip-based method on MNIST when the backbone is chosen as MLP with i.i.d. data.} 
\label{Fig.runtimecompare} 
\end{figure*}

% \textbf{Experiments on CIFAR-10. } \textcolor{red}{xxx}The following \Cref{Fig.cifarcompare} shows the performance of our method training CIFAR-10 with MLP and CNN models. Our FBO-AggITD archives a running time comparable to that of FedNest because both methods consume a similar number of gradient and Hessian-vector computations.However, our FBO-AggITD converges much faster than this gossip-based method, which is slower due to the computation of the Hessian and Jacobian matrices. Since no codes are provided in \cite{yang2022decentralized}, we wrote a code for comparison.
% \begin{figure*}[ht]  
% \centering
% \includegraphics[width=7cm]{CIFARcompare.pdf}
% \caption{Performance of FedNest,  FBO-AggITD and gossip-based method on MNIST when the backbone is chosen as MLP with i.i.d. data.} 
% \label{Fig.cifarcompare} 
% \end{figure*}

\subsection{Model Architectures}
% The MNIST dataset collects handwritten digits containing 60000 examples in training data and 10000 examples in the test set. Each picture size is 28$\times$28 pixels.
We first follow the same experiment in \cite{tarzanagh2022fednest}, thus the model is a 2-layer multilayer perceptron (MLP) with 200 hidden units. The
outer problem optimizes the hidden layer with 157,000 parameters, and the inner problem optimizes the output layer with 2,010 parameters. Additionally, for the CIFAR-10-CNN experiment, we use the 7-layer CNN~\citep{lecun1998gradient} model to train CIFAR-10. We optimize the last two fully connected layers' parameters for solving  the lower-level problem and optimize the rest layers' parameters for solving the upper-level problem.
\subsection{Hyperparameter settings}
For all comparison methods, we optimize their hyperparameters via grid search guided by the default values in their source codes, to ensure the best performance given the algorithms are convergent.

\textbf{Parameter selection for the experiments in \Cref{Fig.main1} and \Cref{Fig.runtimecompare}.} For FedNest and FBO-AggITD, we used the same hyperparameter configuration for both the i.i.d.~and non-i.i.d.~settings. In particular, the inner-stepsize is 0.003, the outer-loop stepsize is 0.01, the constant $\lambda=0.01$ and the number of inner-loop steps is $5$. The choice of the number $\tau$ of outer local epochs  and the data setup are indicated in the figures. 
Then the default value for the client participation ratio is $C=0.1$. Here, it is worth mentioning that for all comparison methods, we optimize their hyperparameters via grid search guided by the default values in their source codes, to ensure the best performance given the algorithms are convergent.

% \textbf{Parameter selection for the experiments in \Cref{Fig.main1}.} For FedNest and FBO-AggITD, we used the same hyperparameter settings despite the heterogeneity. The inner-stepsize is 0.003, the outer-loop stepsize is 0.01, constant $\lambda=0.01$ and the number of inner-loop steps is 5. The choice of the number of outer local epochs $\tau$ and the data setup are indicated in the figures. For LFedNest under i.i.d.~case, the choice of hyperparameters is the same as the~settings mentioned above but under non-i.i.d. case, the inner-stepsize is 0.001 and the outer-loop stepsize is 0.005 or overfitting would occur. Then the default value for the fraction of available clients is C=0.1.

\textbf{Parameters selection for the experiments in \Cref{Fig.ratios} and \Cref{Fig.sgd}.} In \Cref{Fig.ratios} and \Cref{Fig.sgd}, the choice of stepsizes and constant $\lambda$ of FedNest and FBO-AggITD is the same as in \Cref{Fig.main1}. For LFedNest, we choose the same hyperparameters  as FedNest and FBO-AggITD, except that in the non-i.i.d.~case, the inner- and outer-stepsizes are set smaller to be $0.001$ and $0.005$ to avoid the overfitting. 
 The number $\tau$ of outer local epochs is set to be $1$ for all cases. In \Cref{Fig.sgd}, the client participation ratio is $C=0.1$, and the update optimizer in the inner loop is the SGD-type FedAvg method rather than FedSVRG. The choice of hyperparameters for \Cref{Fig.mnistcnncompare} is indicated above and  for \Cref{Fig.cifarcompare} the choice of inner step size and the outer step size are 0.002 and 0.01, respectively while the other options keep the same.
\section{Notations}
\noindent 
For simplicity, we remove subscript $k$ as long as  the involved definitions are clear  in the context. In some proof steps, we will use $x$ and $x_+$ (similarly for $y$ and $y_+$) to denote $x_k$ and $x_{k+1}$ (similarly $y_k$ and $y_{k+1}$), where the definitions of $x_k$ and $y_k$ are given in \Cref{algorithm1}. 
% For example, we let $\mathcal{F}_{k,\upsilon}^{i}$ or $\mathcal{F}_{\upsilon}^{i}$ denote all randomness up to the model parameters $x_{k,\upsilon}^i$ or $x_{\upsilon}^i$.
% Noteworthy, the indirect part of the first estimator is 
% \begin{align}
% \hat{h}^I(x_k)=\alpha\nabla_x\nabla_yG(x_k,y_k^N;\chi)\sum_{t=0}^N\prod_{j=t+1}^{N}(I-\alpha\nabla_y^2G(x_k,y_k^j;\zeta_j))\nabla_yF(x_k,y_k^N;\xi)\nonumber
% \end{align}
% and the indirect part of the second part is
% \begin{align}
% \widetilde{h}^I(x_k)=\alpha\nabla_x\nabla_yG(x_k,y_k^N;\chi)\sum_{t=0}^N\prod_{j=t+1}^{N}(I-\alpha\nabla_y^2G(x_k,y_k^j;\zeta_j))\nabla_yF(x_k,y_k^t;\xi_t)\nonumber
% \end{align}
% Before the proof, we firstly declair all the notations we used.
Based on \Cref{algorithm1}, we also have the definition of $y_{+}=y^N$. 
We recall and define  useful notations for the ease of presentation. 
\begin{align}\label{auxilary_notations}
&\text{Direct parts.}\quad \widetilde{h}_i^D(x_\upsilon^i,y_+)=\nabla_xF_i(x_\upsilon^i,y_+;\xi_{\upsilon}^i),\;\widetilde{h}_i^D(x)=\nabla_xF_i(x,y_+;\xi_i), \;\bar{\nabla}f_i^D(x,y)=\nabla_xf_i(x,y)   \nonumber
\\&\text{Indirect parts.}\quad  \widetilde{h}_i^I(x) =\lambda(N+1)\nabla_x\nabla_yG_i(x,y^N;\chi_i)\prod_{t=N}^{Q+1}(I-\lambda\nabla_y^2G(x,y^t;u_{t}))\nabla_yF(x,y^Q;\xi_{Q}) \nonumber
\\&\qquad\qquad\qquad\;\;\bar{\nabla}f_i^I(x,y) =\nabla_x\nabla_yg_i(x,y)(\nabla_y^2g(x,y))^{-1}\nabla_yf(x,y),
\end{align}
where $\xi_\upsilon^i$ and $\xi_i$ are different data samples and two crucial components are defined by 
\begin{align*}
\nabla_yF(x,y^Q;\xi_Q)=&\frac{1}{|S|}\sum_{i\in S}\nabla_yF_i(x,y^Q;\xi_{i,Q}),\;\;\nabla_y^2G(x,y^t;u_t)=\frac{1}{|S|}\sum_{i\in S}\nabla_x\nabla_yG_i(x,y^t;u_{i,t}).
\end{align*}
Based on the notations in \cref{auxilary_notations},  
we also recall the important forms of our stochastic hypergradient estimate  $\widetilde{h}(x)$ constructed by the proposed AggITD method as well as its expectation form $\Bar{h}(x)={\E}[\widetilde{h}(x)|x,y_+]$, and an auxiliary hypergradient notation $\bar\nabla f(x,y_+)$, respectively. 
\begin{align}\label{notations}
% h(x)=&\frac{1}{|S|}\sum_{i\in S}h_i(x)=h^D(x,y_+)-h^I(x)=\frac{1}{|S|}\sum_{i\in S}[h_i^D(x,y_+)-h_i^I(x)]\nonumber
% \\=&\nabla_xf(x,y_+)-\lambda(N+1)\nabla_x\nabla_yg(x,y^N)\prod_{t=N}^{Q+1}(I-\lambda\nabla_y^2g(x,y^t))\nabla_yf(x,y^Q)
% \nonumber
\widetilde{h}(x)=&\frac{1}{|S|}\sum_{i\in S}\widetilde{h}_i(x)=\frac{1}{|S|}\sum_{i\in S}[\widetilde{h}_i^D(x)-\widetilde{h}_i^I(x)]=\widetilde{h}^D(x)-\widetilde{h}^I(x)\nonumber
\\=&\nabla_xF(x,y_+;\xi)
-\lambda(N+1)\nabla_x\nabla_yG(x,y^N;\chi)\prod_{t=N}^{Q+1}(I-\lambda\nabla_y^2G(x,y^t;u_t))\nabla_yF(x,y^Q;\xi_Q)\nonumber
\\\bar{h}(x)=&\frac{1}{|S|}\sum_{i\in S}\bar{h}_i(x)=\frac{1}{|S|}\sum_{i\in S}[\bar{h}_i^D(x)-\bar{h}_i^I(x)]=\bar{h}^D(x)-\bar{h}^I(x)\nonumber
\\=&\nabla_xf(x,y_+)-\lambda\nabla_x\nabla_yg(x,y^N)\sum_{Q=0}^{N}\prod_{t=N}^{Q+1}(I-\lambda\nabla_y^2g(x,y^t))\nabla_yf(x,y^Q)
\nonumber
\\\bar{\nabla}f(x,y)=&\frac{1}{|S|}\sum_{i\in S}\bar{\nabla}f_i(x,y)=\frac{1}{|S|}\sum_{i\in S}[\bar{\nabla}f_i^D(x,y)-
\bar{\nabla}f_i^I(x,y)]=\bar{\nabla}f^D(x,y)-
\bar{\nabla}f^I(x,y)\nonumber
\\=&\nabla_xf(x,y)-\nabla_x\nabla_yg(x,y)(\nabla_y^2g(x,y))^{-1}\nabla_yf(x,y),
\end{align}
Based on \cref{notations}, it is noted that the hypergradient $\nabla f(x)=\bar{\nabla} f(x,y_{(x)}^\ast)$.   
% We also remove subscript $k$ in the proofs of Lemmas \ref{lemma4}-\ref{lemma6} and \ref{lemma9}-\ref{lemma11}, where we use $x$ and $x_+$ (similarly for $y$ and $y_+$) to denote $x_k$ and $x_{k+1}$.
By the analysis in \citealt{ghadimi2018approximation} and \citealt{chen2021closing}, the following lemma  characterizes the continuity and smoothness properties of  the inner and outer functions $(f_i,g_i) $ for all $ i\in S$. 
\begin{lemma}\label{lemma1}
Suppose \Cref{assum:muconvex}-\Cref{high_lip} hold, for all $x_1$ and $x_2$:
\begin{align}\label{L_fL_yL_yx}
\begin{split}
\|\nabla f(x_1)-\nabla f(x_2)\|\leq&L_f^\prime\|x_1-x_2\|,
\\\|y_{(x_1)}^\ast-y_{(x_2)}^\ast\|\leq& L_y\|x_1-x_2\|,\, 
\\\|\nabla y _{(x_1)}^\ast-\nabla y_{(x_2)}^\ast\|\leq& L_{yx}\|x_1-x_2\|.
\end{split}
\end{align}
% \\
% Besides, for all $i\in S, \upsilon\in{0,...,\tau_i-1},x_1,x_2 and\ y$, we have:
Besides, for all $i\in S,x_1,x_2$ and $y$, we have
\begin{align*}
\|\bar{\nabla}f_i(x_1,y)-\bar{\nabla}f_i(x_1,y_{(x_1)}^\ast)\|\leq&M_f\|y_{(x_1)}^\ast-y\|
\\\|\bar{\nabla}f_i(x_2,y)-\bar{\nabla}f_i(x_1,y)\|\leq& M_f\|x_2-x_1\|,
\end{align*}
where all constants are given by 
\begin{align}
\begin{split}
L_y:=&\frac{L_{g}}{\mu}=\mathcal{O}(\kappa_g)
\\L_{yx}:=&\frac{{\rho}+{\rho}L_y}{\mu}+\frac{L_{g}({\rho}+{\rho}L_y)}{\mu^2}=\mathcal{O}(\kappa_g^3)
\\M_f:=&L_{f}+\frac{L_{g}L_{f}}{\mu}+\frac{M}{\mu}({\rho}+\frac{L_{g}{\rho}}{\mu})=\mathcal{O}(\kappa_g^2)
\\L_f^\prime:=&L_{f}+\frac{L_{g}(L_{f}+M_f)}{\mu}+\frac{M}{\mu}({\rho}+\frac{L_{g}{\rho}}{\mu})=\mathcal{O}(\kappa_g^3)
\end{split}
\end{align}
where all other Lipschitzness constants are provided in Assumptions~\ref{assum:muconvex}-\ref{bounded_variance}. 
\end{lemma}

\section{Proof of \Cref{propsfrvcc} and \Cref{propecevd}}
For the estimator, recall from \cref{notations} that the indirect part is given by   $${\widetilde{h}}^I(x)=\lambda(N+1)\nabla_x\nabla_yG(x,y^N;\chi)\prod_{t=N}^{Q+1}(I-\lambda\nabla_y^2G(x,y^t;u_t))\nabla_yF(x,y^Q;\xi_Q),$$
where $Q$ is drawn form $\{0,...,N\}$ uniformly at random.
\subsection{Proof of \Cref{propsfrvcc}}
\begin{proof}\label{eq:breakhaars}
First, based on the definition of $\widetilde{h}^I(x)$ in \cref{notations} and conditioning on $x,y^N$, we have
\begin{align}
{\E}[\widetilde{h}^I(x)]=&{\E}\Big[\lambda(N+1)\nabla_x\nabla_yG(x,y^N;\chi)\prod_{t=N}^{Q+1}(I-\lambda\nabla_y^2G(x,y^t;u_t))\nabla_yF(x,y^Q;\xi_Q)\Big]\nonumber
\\\overset{(i)}=&\lambda\nabla_x\nabla_yg(x,y^N)\sum_{Q=0}^{N}\prod_{t=N}^{Q+1}(I-\lambda\nabla_y^2g(x,y^t))\nabla_yf(x,y^Q),
\end{align}
where $(i)$ follows from the fact that $Q$ is drawn from $\{0,...,N\}$ uniformly at random and from the independence among $\chi,u_t, \xi_Q $ for $t=1,...,N$. Then the estimation bias of $\widetilde{h}^I(x)$ is bounded by
\begin{align}\label{estimator2bias}
\|{\E}&[\widetilde{h}^I(x)]-\nabla_x\nabla_yg(x,y^N)(\nabla_y^2g(x,y^N))^{-1}\nabla_yf(x, y^N)\|^2\nonumber
\\{\leq}&\Big[\|\nabla_x\nabla_yg(x,y^N)\|^2\Big\|\lambda\sum_{Q=0}^N\prod_{t=N}^{Q+1}(I-\lambda\nabla_y^2g(x,y^t))\nabla_yf(x,y^Q)-(\nabla_y^2g(x,y^N))^{-1}\nabla_yf(x, y^N)\Big\|^2\Big]\nonumber
\\\overset{(i)}{\leq}&L_g^2\Big[\Big\|\lambda\sum_{Q=0}^N\prod_{t=N}^{Q+1}(I-\lambda\nabla_y^2g(x,y^t))\nabla_y f(x,y^Q)-(\nabla_y^2g(x,y^N))^{-1}\nabla_yf(x,y^N)\Big\|^2\Big]\nonumber
\\=&L_g^2\Big[\Big\|\lambda\sum_{Q=0}^N\prod_{t=N}^{Q+1}(I-\lambda\nabla_y^2g(x,y^t))\nabla_yf(x,y^Q)-\lambda\sum_{Q=0}^N\prod_{t=N}^{Q+1}(I-\lambda\nabla_y^2g(x,y^t))\nabla_yf(x,y^N)\nonumber
\\&+\lambda\sum_{Q=0}^N\prod_{t=N}^{Q+1}(I-\lambda\nabla_y^2g(x,y^t))\nabla_yf(x,y^N)-(\nabla_y^2g(x,y^N))^{-1}\nabla_yf(x,y^N)\Big\|^2\Big]\nonumber
\\\overset{(ii)}{\leq}&2\lambda^2 L_g^2\Big[\Big\|\sum_{Q=0}^N\prod_{t=N}^{Q+1}(I-\lambda\nabla_y^2g(x,y^t))[\nabla_yf(x,y^Q)-\nabla_yf(x,y^N)]\Big\|^2\Big]\nonumber
\\&+2L_g^2\Big[\Big\|\lambda\sum_{Q=0}^N\prod_{t=N}^{Q+1}(I-\lambda\nabla_y^2g(x,y^t))-(\nabla_yg(x,y^N))^{-1}\Big\|^2\|\nabla_yf(x,y^N)\|^2\Big]\nonumber
\\\overset{(iii)}{\leq}&2\lambda^2 L_f^2L_g^2\underbrace{(N+1)\sum_{Q=0}^{N}(1-\lambda\mu)^{2N-2Q}[\|y^Q-y^N\|^2]}_\text{\textcircled{1}}\nonumber
\\&+2L_g^2M^2\Big\|\lambda\sum_{Q=0}^N\prod_{t=N}^{Q+1}(I-\lambda\nabla_y^2g(x,y^t))-(\nabla_yg(x,y^N))^{-1}\Big\|^2,
\end{align}
where
$(i)$ 
uses \Cref{assum:lip}, $(ii)$ follows from Young's inequality, and $(iii)$ follows from \Cref{lemma1} and \Cref{assum:lip}. Then, unconditioning on $x_k,y_k^N$ yields
\begin{align}
    \E\Big[&\|{\E}[\widetilde{h}^I(x)]-\nabla_x\nabla_yg(x,y^N)(\nabla_y^2g(x,y^N))^{-1}\nabla_yf(x, y^N)\|^2\,|\,x,y^N\Big] \nonumber
   \\ \leq&2\lambda^2 L_f^2L_g^2\underbrace{(N+1)\sum_{Q=0}^{N}(1-\lambda\mu)^{2N-2Q}\E[\|y^Q-y^N\|^2]}_\text{\textcircled{1}}\nonumber
\\&+2L_g^2M^2\E\Big[\Big\|\lambda\sum_{Q=0}^N\prod_{t=N}^{Q+1}(I-\lambda\nabla_y^2g(x,y^t))-(\nabla_yg(x,y^N))^{-1}\Big\|^2\Big].
\end{align}
Based on Theorem 4 in \citealt{mitra2021linear}, 
for all $t \in [0,...,N-1]$, we obtain
\begin{align}\label{eq:qmitras}
\E[\|y^{t+1}-y_{(x)}^\ast\|^2]\leq(1-\frac{\beta\mu}{2})\E[\|y^t-y_{(x)}^\ast\|^2]+25\beta^2\sigma_g^2
\end{align}
which, by telescoping over $t$ from $0$ to $Q-1$ for any $Q\in\{0,...,N\}$, yields 
% in conjunction with our algorithmic setting that $y_{+} = y^N$, yields 
\begin{align}\label{mitra}
\E[\|y^Q-y_{(x)}^\ast\|^2]\leq(1-\frac{\beta\mu}{2})^Q\E[\|y-y_{(x)}^\ast\|^2]+25N\beta^2\sigma_g^2.
\end{align}
Now we provide  the upper bound of the first term on the RHS
of \cref{estimator2bias} as 
\begin{align}\label{textcircled2}
\text{\textcircled{1}}\leq&(N+1)\sum_{Q=0}^{N}(1-\lambda\mu)^{2N-2Q}\Big[2\E[\|y^Q-y_{(x)}^\ast\|^2]+2\E[\|y^N-y_{(x)}^\ast\|^2]\Big]\nonumber
\\\overset{(i)}{\leq}&2(N+1)\sum_{Q=0}^{N}(1-\lambda\mu)^{2N-2Q}\Big[(1-\frac{\beta\mu}{2})^N\E[\|y-y_{(x)}^\ast\|^2] \nonumber
\\&+(1-\frac{\beta\mu}{2})^Q\E[\|y-y_{(x)}^\ast\|^2]+50N\beta^2\sigma_g^2\Big]\nonumber
\\\overset{(ii)}{\leq}&2(N+1)\Big(\frac{(1-\frac{\beta\mu}{2})^N}{\lambda\mu}+\frac{(1-\frac{\beta\mu}{2})^N}{1-\frac{(1-\lambda\mu)^2}{1-\frac{\beta\mu}{2}}}\Big)\E[\|y-y_{(x)}^\ast\|^2]+\frac{100N(N+1)\beta^2\sigma_g^2}{\lambda\mu}\nonumber
\\\leq&2\underbrace{(N+1)\frac{3(1-\frac{\beta\mu}{2})^N}{\lambda\mu}}_\text{$\alpha_3(N)$}\E[\|y-y_{(x)}^\ast\|^2]+\frac{100N(N+1)\beta^2\sigma_g^2}{\lambda\mu},
\end{align}
where $(i)$ follows from \cref{mitra}, $(ii)$ follows because $\frac{(1-\lambda\mu)^2}{1-\frac{\beta\mu}{2}}\leq\frac{1-\lambda\mu}{1-\frac{\beta\mu}{2}}\leq\frac{1-\lambda\mu}{1-\frac{\lambda\mu}{2}}\leq1$ as the selection that $\beta<\lambda\leq\frac{1}{L_g}$. 
Then we provide the upper bound of the second term in \cref{estimator2bias} as 
\begin{align}\label{MN-t2start}
\E\Big[\Big\|\lambda&\sum_{Q=0}^{N}\prod_{t=N}^{Q+1}(I-\lambda\nabla_y^2g(x,y^t))-(\nabla_y^2g(x,y^N))^{-1}\Big\|^2\Big]\nonumber
\\=&\lambda^2\E\Big[\Big\|\sum_{Q=0}^{N}\prod_{t=N}^{Q+1}(I-\lambda\nabla_y^2g(x,y^t))-\sum_{Q=0}^{N}(I-\lambda\nabla_y^2g(x,y^N))^{N-Q} \nonumber
\\&\qquad-\sum_{Q=N+1}^{\infty}(I-\lambda\nabla_y^2g(x,y^N))^Q\Big\|^2\Big]\nonumber
\\\overset{(i)}\leq&2\lambda^2(N+1)\sum_{Q=0}^{N}\E\Big[\Big\|\underbrace{\prod_{t=N}^{Q+1}(I-\lambda\nabla_y^2g(x,y^t))-(I-\lambda\nabla_y^2g(x, y^N))^{N-Q}\Big\|^2}_\text{$M_{N-Q}^2$}\Big]+\frac{2(1-\lambda\mu)^{2N+2}}{\mu^2}	
\end{align}
where $(i)$ 
follows from the Young's inequality. Now we 
provide the upper bound of the term $M_{N-Q}$ as 
\begin{align}\label{eq:mqt_eqpp}
M_{N-Q} =&\Big\|\prod_{t=N}^{Q+1}(I-\lambda\nabla_y^2g(x, y^t))-(I-\lambda\nabla_y^2g(x, y^N))^{N-Q}\Big\|\nonumber
\\=&\Big\|(I-\lambda\nabla_y^2g(x,y^N))\Big[\prod_{t=N}^{Q+2}(I-\lambda\nabla_y^2g(x,y^t))-(I-\lambda\nabla_y^2g(x,y^N))^{N-Q-1}\Big]\nonumber
\\&+(\lambda\nabla_y^2g(x,y^N)-\lambda\nabla_y^2g(x,y^{Q+1}))\prod_{t=N}^{Q+2}(I-\lambda\nabla_y^2g(x,y^t))\Big\|\nonumber
\\\overset{(i)}{\leq}&(1-\lambda\mu)\underbrace{\Big\|\prod_{t=N}^{Q+2}(I-\lambda\nabla_y^2g(x,y^t))-(I-\lambda\nabla_y^2g(x,y^N))^{N-Q-1}\Big\|}_\text{$M_{N-Q-1}$} \nonumber
\\&+\lambda\rho(1-\lambda\mu)^{N-Q-1}\|y^N-y^{Q+1}\|\nonumber
\\\overset{(ii)}{\leq}&(1-\lambda\mu)^{N-Q}M_0+\lambda\rho(1-\lambda\mu)^{N-Q-1}\sum_{\tau=Q+1}^{N}\|y^\tau-y^N\|\nonumber
\\\overset{(iii)}{\leq}&\lambda\rho(1-\lambda\mu)^{N-Q-1}\sum_{\tau=Q+1}^{N}\|y^\tau-y^N\|,
\end{align}	
where $(i)$ follows from the \Cref{assum:muconvex} and \Cref{high_lip}, $(ii)$ can be obtained after telescoping over $t$ from 0 to $N-1$ and $(iii)$ follows from that $M_0 = 0$. Then substitute \cref{eq:mqt_eqpp} into \cref{MN-t2start}, we obtain,
\begin{align}\label{MN-t2bound}
(N+1)\sum_{Q=0}^N&\E[M_{N-Q}^2] \nonumber
\\\leq&\lambda^2\rho^2(N+1)\sum_{Q=0}^N\Big[(1-\lambda\mu)^{2N-2Q-2}\Big](N-Q)\sum_{\tau=Q+1}^{N}\Big[2\Big(1-\frac{\beta\mu}{2}\Big)^\tau\E[\|y-y_{(x)}^\ast\|^2]\nonumber
\\&+2\Big(1-\frac{\beta\mu}{2}\Big)^N\E[\|y-y_{(x)}^\ast\|^2]+50N\beta^2\sigma_g^2+50\tau\beta^2\sigma_g^2\Big]\nonumber
\\\leq&\lambda^2\rho^2(N+1)\sum_{Q=0}^N(1-\lambda\mu)^{2N-2Q-2}(N-Q)\Big[\frac{4(1-\frac{\beta\mu}{2})^{Q+1}}{\beta\mu}\E[\|y-y_{(x)}^\ast\|^2]\nonumber
\\&+2(N-Q)\Big(1-\frac{\beta\mu}{2}\Big)^N\E[\|y-y_{(x)}^\ast\|^2+100N(N-Q)\beta^2\sigma_g^2\Big]\nonumber
\\=&2(N+1)\lambda^2\rho^2\sum_{Q=0}^{N}(1-\lambda\mu)^{2N-2Q-2}(N-Q)^2(1-\frac{\beta\mu}{2})^N\E[\|y-y_{(x)}^\ast\|^2]\nonumber
\\&+4(N+1)\lambda^2\rho^2\sum_{Q=0}^{N}(1-\lambda\mu)^{2N-2Q-2}(N-Q)\frac{(1-\frac{\beta\mu}{2})^{Q+1}}{\beta\mu}\E[\|y-y_{(x)}^\ast\|^2]\nonumber
\\&+100\beta^2\sigma_g^2\lambda^2\rho^2N(N+1)\sum_{Q=0}^N(1-\lambda\mu)^{2N-2Q-2}(N-Q)^2\nonumber
% \\\overset{(i)}{<}&2(N+1)(1-\frac{\beta\mu}{2})^N\Big[\frac{1+(1-\lambda\mu)^2}{\lambda\mu^3}\rho^2+\frac{2\lambda^2\rho^2}{\beta\mu(1-\frac{(1-\lambda\mu)^2}{1-\frac{\beta\mu}{2}})^2}\Big]\E[\|y-y_{(x)}^\ast\|^2]\nonumber
% \\&+100\beta^2\rho^2\sigma_g^2\Big[\frac{N(N+1)(1+(1-\lambda\mu)^2)}{\lambda\mu^3}\Big]\nonumber
% \\\overset{(ii)}{\leq}&2(N+1)(1-\frac{\beta\mu}{2})^N\Big[\frac{1+(1-\lambda\mu)^2}{\lambda\mu^3}\rho^2+\frac{2\lambda\rho^2}{\mu(1-\frac{(1-\lambda\mu)^2}{1-\frac{\beta\mu}{2}})^2}\Big]\E[\|y-y_{(x)}^\ast\|^2]\nonumber
% \\&+100\beta^2\rho^2\sigma_g^2\Big[\frac{N(N+1)(1+(1-\lambda\mu)^2)}{\lambda\mu^3}\Big]\nonumber
\\<&\underbrace{4(N+1)(1-\frac{\beta\mu}{2})^N\Big(\frac{\rho^2}{\lambda\mu^3}+\frac{4\rho^2}{\beta\mu^3}\Big)}_\text{$\alpha_1(N)$}\E[\|y-y_{(x)}^\ast\|^2]\nonumber
\\&+100\beta^2\rho^2\sigma_g^2\underbrace{\Big[\frac{N(N+1)(1+(1-\lambda\mu)^2)}{\lambda\mu^3}\Big]}_\text{$\alpha_2(N)$},
\end{align}
where the last inequality follows because 
$\sum_{t=0}^{N}(1-\lambda\mu)^{2N-2t-2}(N-t)^2<\frac{1+(1-\lambda\mu)^2}{\lambda^3\mu^3}$ and $\sum_{t=0}^{N}(1-\lambda\mu)^{2N-2t-2}(N-t)(1-\frac{\beta\mu}{2})^{t+1}<\frac{1}{\big(1-\frac{(1-\lambda\mu)^2}{1-\frac{\beta\mu}{2}}\big)^2}\leq\frac{(2-\lambda\mu)^2}{\lambda^2\mu^2}$. 
% Applying 
% Incorporating 
Substituting~\cref{MN-t2bound} into \cref{MN-t2start}, and applying \cref{MN-t2bound} and \cref{textcircled2} to \cref{estimator2bias}, 
% we can obtain
we have 
\begin{align*}
\E\big[\|{\E}&[\widetilde{h}^I(x)]-\nabla_x\nabla_yg(x,y^N)(\nabla_y^2g(x,y^N)^{-1})\nabla_yf(x, y^N)\|^2\,|\,x,y^N\big]\nonumber
\\\leq&4\lambda^2L_f^2L_g^2\alpha_3(N)\E[\|y-y_{(x)}^\ast\|^2]+\frac{200\lambda\beta^2\sigma_g^2L_f^2L_g^2N(N+1)}{\mu}\nonumber\\
&+\frac{4L_g^2M^2(1-\lambda\mu)^{2N+2}}{\mu^2}+4\lambda^2L_g^2M^2\alpha_1(N)\E[\|y-y_{(x)}^\ast\|^2]+400\lambda^2\beta^2L_g^2M^2\sigma_g^2\rho^2\alpha_2(N)\nonumber
\\=&[4\lambda^2L_g^2M^2\alpha_1(N)+4\lambda^2L_f^2L_g^2\alpha_3(N)]\E[\|y-y_{(x)}^\ast\|^2]+\frac{4L_g^2M^2(1-\lambda\mu)^{2N+2}}{\mu^2}\nonumber\\
&+400\lambda^2\beta^2L_g^2M^2\sigma_g^2\rho^2\alpha_2(N)+\frac{200\lambda\beta^2\sigma_g^2L_f^2L_g^2N(N+1)}{\mu},
\end{align*}
which completes the proof. 
\end{proof}
\subsection{Proof of \Cref{propecevd}}
% \begin{proof}
Based on the definition of $\widetilde{h}^I(x)$ and $\Bar{h}^I_i(x)$, using the fact that Var$(X)\leq \mathbb{E}[X^2]$, and conditioning on $x,y^N$, we have 
\begin{align}\label{eq:varincceBBs}
\E\|\widetilde{h}_i^I&(x)-\Bar{h}^I_i(x)\|^2\leq \E\|\widetilde{h}_i^I(x)\|^2\nonumber
\\\leq&\E\Big\|\lambda(N+1)\nabla_x\nabla_yG_i(x,y^N;\chi)\prod_{t=N}^{Q+1}(I-\lambda\nabla_y^2G(x,y^t;u_t))\nabla_yF(x,y^Q;\xi_Q)\Big\|^2\nonumber
\\\overset{(i)}\leq&\lambda^2(N+1)^2L_g^2M^2{\E}\Big\|\prod_{t=N}^{Q+1}(I-\lambda\nabla_y^2G(x,y^t;u_t))\Big\|^2 \nonumber
\\\overset{(ii)}\leq&\lambda^2(N+1)^2L_g^2M^2{\E}_Q(1-\lambda\mu)^{2(N-Q)} \nonumber
\\=&\lambda^2(N+1)L_g^2M^2\sum_{Q=0}^N(1-\lambda\mu)^{2Q} = \lambda^2(N+1)L_g^2M^2\frac{1-(1-\lambda\mu)^{2N}}{1-(1-\lambda\mu)^2} \nonumber
\\\overset{(iii)}\leq & \frac{\lambda(N+1)L_g^2M^2}{\mu},
\end{align}
where $(i)$ follows from \Cref{assum:lip}, $(ii)$ follows from \Cref{assum:muconvex} and $(iii)$ follows from $\lambda\leq\frac{1}{L_g}$. Then, the first part is proved. 
For the second part, conditioning on $x,y_+$, we have 
\begin{align*}
\E\|&\widetilde{h}_i^D(x_{\upsilon}^i,y_+)-\widetilde{h}_i^D(x_0^i,y_+)+\widetilde{h}^D(x)-\widetilde{h}^I(x)\|^2 \nonumber
\\&\leq 4\E\|\widetilde{h}_i^D(x_{\upsilon}^i,y_+)\|^2 + 4\E\|\widetilde{h}_i^D(x_0^i,y_+)\|^2+4\E\Big\|\frac{1}{|S|}\sum_{i\in S}\widetilde{h}^D_i(x)\Big\|^2 +4\E\Big\|\frac{1}{|S|}\sum_{i\in S}\widetilde{h}_i^I(x)\Big\|^2 \nonumber
\\&\overset{(i)}\leq 8M^2+ 4\E\|\widetilde{h}^D_i(x)\|^2 + 4\E\|\widetilde{h}^I_i(x)\|^2 \nonumber
\\&\overset{(ii)}\leq 12M^2 + \frac{4\lambda(N+1)L_g^2M^2}{\mu},
\end{align*}
    where $(i)$ follows from \Cref{assum:lip} and $(ii)$ follows from \cref{eq:varincceBBs}. Then, the proof is complete. 
% \end{proof}
\section{Proof of \Cref{Theorem2-mainB} and \Cref{corollary2}}
We now provide some auxiliary lemmas to characterize the \Cref{Theorem2-mainB} and \Cref{corollary2}.
\begin{lemma}[Restatement of \Cref{lemma:descentl}]\label{lemma9}
% According to \cref{lemma4}, the following result can be obtained by the same steps and settings:
Suppose Assumptions~\ref{assum:muconvex}-\ref{bounded_variance} are satisfied. Let $y^{\ast}=\argmin_{y}g(x,y)$. Further, we set $\lambda\leq\min\{10, \frac{1}{L_g}\}$, $\alpha^i=\frac{\alpha}{\tau_i}$ with $\tau_i\geq1$ for some positive $\alpha$ and $\beta^i=\frac{\beta}{\tau_i}$, where $\beta\leq\min\{1,\lambda,\frac{1}{6L_{g}}\}$ $\forall i \in S$. Then, we have the following inequality 
\begin{align}
\E[f(x_+)]-\E[f(&x)]\leq-\frac{\alpha}{2}\E[\|\nabla f(x)\|^2]+4\alpha^2\sigma_h^2L_f^\prime+4\alpha^2\sigma_f^2L_f^\prime+2\alpha^2M^2L_f^\prime\nonumber
\\&-\frac{\alpha}{2}(1-4\alpha L_f^\prime)\E\Big[\Big\|\frac{1}{m}\sum_{i=1}^{m}\frac{1}{\tau_i}\sum_{\upsilon=0}^{\tau_i-1}\big(\bar{h}_i^D(x_{\upsilon}^i,y_+)-\bar{h}^I(x)\big)\Big\|^2\Big]\nonumber
\\&+\frac{3\alpha}{2}\Big[\big(4\lambda^2L_g^2M^2\alpha_1(N)+4\lambda^2L_f^2L_g^2\alpha_3(N)\big)\E[\|y-y^\ast\|^2]+\frac{4L_g^2M^2(1-\lambda\mu)^{2N+2}}{\mu^2}\nonumber\\
&+400\lambda^2\beta^2L_g^2M^2\sigma_g^2\rho^2\alpha_2(N)+\frac{200\lambda\beta^2\sigma_g^2L_f^2L_g^2N(N+1)}{\mu}
\nonumber
\\&+\frac{M_f^2}{m}\sum_{i=1}^{m}\frac{1}{\tau_i}\sum_{\upsilon=0}^{\tau_i-1}\E[\|x_{\upsilon}^i-x\|^2]+M_f^2\E[\|y_+-y^\ast\|^2]\Big]
\end{align}
where $\Bar{h}^I(x)={\E}[\widetilde{h}^I(x)|x,y_+]$, $\Bar{h}_i^D(x_{\upsilon}^i,y_+)={\E}[\widetilde{h}_i^D(x_{\upsilon}^i,y_+)|x_\upsilon^i]$ and $\alpha_1(N), \alpha_2(N), \alpha_3(N)$ are defined in \Cref{propsfrvcc}.
\end{lemma}
\begin{proof}
% From Algorithm~\ref{algorithm5}, we know that, $\forall i \in S$
From \Cref{algorithm4}, we have, $\forall i \in S$
\begin{align*}
x_+=&x-\frac{1}{m}\sum_{i=1}^{m}\alpha^i\sum_{\upsilon=0}^{\tau_i-1}\big(\widetilde{h}_i^D(x_{\upsilon}^i,y_+)-\widetilde h_i^D(x_0^i,y_+)+\widetilde h^D(x) -\widetilde{h}^I(x)\big),
\end{align*}
where $x_0^i=x$ and the data samples for $\widetilde h_i^D(x_0^i,y_+)$ and $\widetilde h^D(x)$ are different. 
Using the descent lemma yields 
\begin{align}\label{lemma4start}
\E[f(x_+)]-\E[f(x)]\leq&\E[\langle x_+-x,\nabla f(x)\rangle]+\frac{L_f^\prime}{2}\E[\|x_+-x\|^2]\nonumber
\\=&-\E\Big[\Big\langle\frac{1}{m}\sum_{i=1}^m\alpha^i\sum_{\upsilon=0}^{\tau_i-1}\big(\widetilde{h}_i^D(x_{v}^i,y_+)-\widetilde{h}_i^D(x_0^i,y_+)+\widetilde{h}^D(x)-\widetilde{h}^I(x)\big),\nabla f(x)\Big\rangle\Big] \nonumber
\\&+\frac{L_f^\prime}{2}\E\Big[\Big\|\frac{1}{m}\sum_{i=1}^m\alpha^i\sum_{\upsilon=0}^{\tau_i-1}\big(\widetilde{h}_i^D(x_{v}^i,y_+)-\widetilde{h}_i^D(x_0^i,y_+)+\widetilde{h}^D(x)-\widetilde{h}^I(x)\big)\Big\|^2\Big].
\end{align}
% In the following, we bound each term on the right hand side(RHS), for the first term,
We next bound each term of the right hand side (RHS) of \cref{lemma4start}. In specific, for the first term, we have 
\begin{align}\label{eq:emaiscs}
-\E\Big[\Big\langle\frac{1}{m}\sum_{i=1}^m&\alpha^i\sum_{\upsilon=0}^{\tau_i-1}\big(\widetilde{h}_i^D(x_{v}^i,y_+)-\widetilde{h}_i^D(x_0^i,y_+)+\widetilde{h}^D(x)-\widetilde{h}^I(x)\big),\nabla f(x)\Big\rangle\Big]\nonumber
\\{=}&-\E\Big[\E\Big[\Big\langle\frac{1}{m}\sum_{i=1}^m\alpha^i\sum_{\upsilon=0}^{\tau_i-1}\big(\widetilde{h}_i^D(x_{v}^i,y_+)-\widetilde{h}_i^D(x_0^i,y_+)+\widetilde{h}^D(x)-\widetilde{h}^I(x)\big),\nabla f(x)\Big\rangle\Big|x,y_+\Big]\Big]\nonumber
\\\overset{(i)}{=}&-\E\Big[\E\Big[\Big\langle\frac{1}{m}\sum_{i=1}^m\alpha^i\sum_{\upsilon=0}^{\tau_i-1}\big(\widetilde{h}_i^D(x_{v}^i,y_+)-\bar{h}^I(x)\big),\nabla f(x)\Big\rangle\Big|x_\upsilon^i\Big]\Big]\nonumber
\\\overset{(ii)}{=}&-\alpha\E\Big[\Big\langle\frac{1}{m}\sum_{i=1}^m\frac{1}{\tau_i}\sum_{\upsilon=0}^{\tau_i-1}\big(\Bar{h}_i^D(x_{v}^i,y_+)-\bar{h}^I(x)\big),\nabla f(x)\Big\rangle\Big]\nonumber
\\=&-\frac{\alpha}{2}\E\Big[\Big\|\frac{1}{m}\sum_{i=1}^{m}\frac{1}{\tau_i}\sum_{\upsilon=0}^{\tau_i-1}\big(\Bar{h}_i^D(x_{v}^i,y_+)-\bar{h}^I(x)\big)\Big\|^2\Big]-\frac{\alpha}{2}\E[\|\nabla f(x)\|^2]\nonumber
\\&+\frac{\alpha}{2}\E\Big[\Big\|\frac{1}{m}\sum_{i=1}^{m}\frac{1}{\tau_i}\sum_{\upsilon=0}^{\tau_i-1}\big(\Bar{h}_i^D(x_{v}^i,y_+)-\bar{h}^I(x)\big)-\nabla f(x)\Big\|^2\Big],
\end{align}
where $(i)$
follows because $\E\Big[\frac{1}{m}\sum_{i=1}^m\alpha^i\sum_{\upsilon=0}^{\tau_i-1}\big(-\widetilde{h}_i^D(x_0^i,y_+)+\widetilde{h}^D(x)\big)\Big|x,y_+\Big]=0$ and $\Bar{h}^I(x)={\E}[\widetilde{h}^I(x)|x,y_+]$, $(ii)$ follows because
$\Bar{h}_i^D(x_{\upsilon}^i,y_+)={\E}[\widetilde{h}_i^D(x_{\upsilon}^i,y_+)|x_\upsilon^i] $. 
% Then the goal turns to give the upper bound of the first and the last term. Notice that the notations can be referred in \cref{notations}
The next step is to upper bound the last term of RHS of \cref{eq:emaiscs}. Based on the notations in \cref{auxilary_notations} and \cref{notations}, we have 
% \begin{align}
% -\E\Big[\Big\|\frac{1}{m}\sum_{i=1}^{m}&\frac{1}{\tau_i}\sum_{\upsilon=0}^{\tau_i-1}\big(\Bar{h}_i^D(x_{v}^i,y_+)-\bar{h}^I(x)\big)\Big\|^2\Big]\nonumber
% \\=&-\E\Big[\Big\|\frac{1}{m}\sum_{i=1}^{m}\frac{1}{\tau_i}\sum_{\upsilon=0}^{\tau_i-1}\big(\Bar{h}_i^D(x_{v}^i,y_+)-\bar{h}_i^D(x,y_+)\big)+\bar{h}^D(x,y_+)-\bar{h}^I(x)\Big\|^2\Big]\nonumber
% \end{align}
\begin{align}\label{eq:hiDminusHi}
\Big\|\frac{1}{m}\sum_{i=1}^{m}&\frac{1}{\tau_i}\sum_{\upsilon=0}^{\tau_i-1}\big(\Bar{h}_i^D(x_{v}^i,y_+)-\bar{h}^I(x)\big)-\nabla f(x)\Big\|^2\nonumber
\\=&\Big\|\frac{1}{m}\sum_{i=1}^{m}\frac{1}{\tau_i}\sum_{\upsilon=0}^{\tau_i-1}\big(\Bar{h}_i^D(x_{v}^i,y_+)-\bar{h}^I(x)\big)-\bar{\nabla}f(x,y_+)+\bar{\nabla}f(x,y_+)-\nabla f(x)\Big\|^2\nonumber
\\=&\Big\|\frac{1}{m}\sum_{i=1}^{m}\frac{1}{\tau_i}\sum_{\upsilon=0}^{\tau_i-1}\big(\Bar{h}_i^D(x_{v}^i,y_+)-\bar{\nabla}f_i^D(x,y_+)\big)-\bar{h}^I(x)+\bar{\nabla}f^I(x,y_+)+\bar{\nabla}f(x,y_+)-\nabla f(x)\Big\|^2\nonumber
\\\overset{(i)}\leq&3\Big\|\frac{1}{m}\sum_{i=1}^{m}\frac{1}{\tau_i}\sum_{\upsilon=0}^{\tau_i-1}\big(\Bar{h}_i^D(x_{v}^i,y_+)-\bar{\nabla}f_i^D(x,y_+)\big)\|^2 \nonumber
\\&+3\Big\|\bar{h}^I(x)-\bar{\nabla}f^I(x,y_+)\Big\|^2+3\|\bar{\nabla}f(x, y_+)-\nabla f(x)\|^2\nonumber
\\\overset{(ii)}\leq&\frac{3M_f^2}{m}\sum_{i=1}^m\frac{1}{\tau_i}\sum_{\upsilon=0}^{\tau_i-1}\|x_\upsilon^i-x\|^2+3M_f^2\|y_+-y^\ast\|^2+3\|\bar{h}^I(x)-\bar{\nabla}f^I(x,y_+)\|^2
\end{align}
where $(i)$ follows from the Young's inequality and $(ii)$ follows from  \Cref{lemma1} and Assumption \ref{assum:lip}.
% the first inequality uses Young's inequality and the last one use \cref{lemma1}, and Assumption \ref{assum:lip}. 
Then applying \Cref{propsfrvcc} to \cref{eq:hiDminusHi}, we can obtain
\begin{align}\label{lemma4firstterm}
\E\Big[\Big\|\frac{1}{m}\sum_{i=1}^{m}&\frac{1}{\tau_i}\sum_{\upsilon=0}^{\tau_i-1}\big(\Bar{h}_i^D(x_{v}^i,y_+)-\bar{h}^I(x)\big)-\nabla f(x)\Big\|^2\Big]\nonumber
\\\leq&\frac{12L_g^2M^2(1-\lambda\mu)^{2N+2}}{\mu^2}+[12\lambda^2L_g^2M^2\alpha_1(N)+12\lambda^2L_f^2L_g^3\alpha_3(N)]\E[\|y-y^\ast\|^2]\nonumber
\\&+1200\lambda^2\beta^2L_g^2M^2\sigma_g^2\rho^2\alpha_2(N)+\frac{600\lambda\beta^2\sigma_g^2L_f^2L_g^2N(N+1)}{\mu}\nonumber
\\&+\frac{3M_f^2}{m}\sum_{i=1}^{m}\frac{1}{\tau_i}\sum_{\upsilon=0}^{\tau_i-1}\E[\|x_{\upsilon}^i-x\|^2]+3M_f^2\E[\|y_+-y^\ast\|^2].
\end{align}
Then 
for the second term of \cref{lemma4start}, we have 
% \begin{align*}
% x_+=&x-\frac{1}{m}\sum_{i=1}^{m}\alpha^i\sum_{\upsilon=0}^{\tau_i-1}\big(\widetilde{h}_i^D(x_{\upsilon}^i,y_+)-\widetilde h_i^D(x_0^i,y_+)+\widetilde h^D(x) -\widetilde{h}^I(x)\big),
% \end{align*}
\begin{align}\label{lemma4secondterm}
\E\Big[\Big\|&\frac{1}{m}\sum_{i=1}^m \alpha^i\sum_{\upsilon=0}^{\tau_i-1}\big(\widetilde{h}_i^D(x_{v}^i,y_+)-\widetilde h_i^D(x_0^i,y_+)+\widetilde h^D(x) -\widetilde{h}^I(x)\big)\Big\|^2\Big]\nonumber
\\\overset{(i)}{\leq}&2\alpha^2\E\Big[\Big\|\frac{1}{m}\sum_{i=1}^m \frac{1}{\tau_i}\sum_{\upsilon=0}^{\tau_i-1}\big(\widetilde{h}_i^D(x_{v}^i,y_+)-\widetilde{h}^I(x)\big)\Big\|^2\Big]  \nonumber
\\&+2\alpha^2\E\Big[\Big\|\frac{1}{m}\sum_{i=1}^m\frac{1}{\tau_i}\sum_{\upsilon=0}^{\tau_i-1}\big(-\widetilde h_i^D(x_0^i,y_+)+\widetilde h^D(x)\big)\Big\|^2\Big]\nonumber
\\=&2\alpha^2\E\Big[\Big\|\frac{1}{m}\sum_{i=1}^{m}\frac{1}{\tau_i}\sum_{\upsilon=0}^{\tau_i-1}\big(\widetilde{h}_i^D(x_{\upsilon}^i,y_+)-\Bar{h}_i^D(x_{\upsilon}^i,y_+)+\Bar{h}^I(x)-\widetilde{h}^I(x)+\bar{h}_i^D(x_\upsilon^i,y_+)-\bar{h}^I(x)\big)\Big\|^2\Big]\nonumber
\\&+2\alpha^2\E\Big[\Big\|\frac{1}{m}\sum_{i=1}^m\frac{1}{\tau_i}\sum_{\upsilon=0}^{\tau_i-1}\big(\widetilde h_i^D(x_0^i,y_+)\big)\Big\|^2\Big]+2\alpha^2\E[\|\widetilde{h}^D(x)\|^2]\nonumber
\\\overset{(ii)}\leq&4\alpha^2\E\Big[\Big\|\frac{1}{m}\sum_{i=1}^{m}\frac{1}{\tau_i}\sum_{\upsilon=0}^{\tau_i-1}\Bar{h}_i^D(x_{\upsilon}^i,y_+)-\bar{h}^I(x)\Big\|^2\Big] \nonumber
\\&+4\alpha^2\E\Big[\Big\|\frac{1}{m}\sum_{i=1}^{m}\frac{1}{\tau_i}\sum_{\upsilon=0}^{\tau_i-1}\big(\widetilde{h}_i^D(x_{\upsilon}^i,y_+)-\Bar{h}_i^D(x_{\upsilon}^i,y_+)+\Bar{h}^I(x)-\widetilde{h}^I(x)\big)\Big\|^2\Big]+4\alpha^2M^2\nonumber
\\\overset{(iii)}\leq&4\alpha^2\E\Big[\Big\|\frac{1}{m}\sum_{i=1}^{m}\frac{1}{\tau_i}\sum_{\upsilon=0}^{\tau_i-1}\Bar{h}_i^D(x_{\upsilon}^i,y_+)-\bar{h}^I(x)\Big\|^2\Big] \nonumber
\\&+8\alpha^2\E\Big[\Big\|\frac{1}{m}\sum_{i=1}^{m}\frac{1}{\tau_i}\sum_{\upsilon=0}^{\tau_i-1}\big(\widetilde{h}_i^D(x_{\upsilon}^i,y_+)-\Bar{h}_i^D(x_{\upsilon}^i,y_+)\big)\Big\|^2\nonumber
\\&+8\alpha^2\E\Big[\Big\|\frac{1}{m}\sum_{i=1}^{m}\frac{1}{\tau_i}\sum_{\upsilon=0}^{\tau_i-1}\big(\Bar{h}^I(x)-\widetilde{h}^I(x)\big)\Big\|^2\Big]+4\alpha^2M^2\nonumber
\\\overset{(iv)}\leq&4\alpha^2\E\Big[\Big\|\frac{1}{m}\sum_{i=1}^{m}\frac{1}{\tau_i}\sum_{\upsilon=0}^{\tau_i-1}\Bar{h}_i^D(x_{\upsilon}^i,y_+)-\bar{h}^I(x)\Big\|^2\Big]+8\alpha^2\sigma_f^2+8\alpha^2\sigma_h^2+4\alpha^2M^2
\end{align}
where $(i)$ and $(iii)$ follow from the Young's inequality, $(ii)$ follows from Young's inequality and \Cref{assum:lip} and $(iv)$ follows from Assumption \ref{bounded_variance} and lemma \ref{lemma1}.
Plugging \cref{lemma4firstterm} and \cref{lemma4secondterm} into \cref{lemma4start} completes the proof. 
\end{proof}
\begin{lemma}[Restatement of \Cref{le:warm_start}]\label{lemma5}
Suppose Assumptions~\ref{assum:muconvex}-\ref{bounded_variance} are satisfied. Let $y^{\ast}=\argmin_{y}g(x,y)$ and  $y_{(x_+)}^{\ast}=\argmin_{y}g(x_+,y)$. Further, set $\alpha^i=\frac{\alpha}{\tau_i}$ with $\tau_i\geq1$ with some positive $\alpha$, $\forall i \in S$. Then, we have
% \\
\begin{align*}
\begin{split}
\E[\|y_+-y_{(x_+)}^\ast\|^2]\leq& b_1(\alpha)\E\Big[\Big\|\frac{1}{m}\sum_{i=1}^{m}\frac{1}{\tau_i}\sum_{\upsilon=0}^{\tau_i-1}\big(\Bar{h}_i^D(x_{v}^i,y_+)-\bar{h}^I(x)\big)\Big\|^2\Big]
+b_2(\alpha)\E[\|y_+-y^\ast\|^2]
\\&+b_3(\alpha)(2\sigma_h^2+2\sigma_f^2+M^2)
\end{split}
\end{align*}
where the constants are given by 
\begin{align*}
\begin{split}
b_1(\alpha):=&4L_y^2\alpha^2+\frac{L_y^2\alpha^2}{4\gamma}+\frac{2L_{yx}\alpha^2}{\eta},\; b_2(\alpha):=1+4\gamma+\frac{\eta L_{yx}D_h^2\alpha^2}{2},\;
b_3(\alpha):=4\alpha^2L_y^2+\frac{2L_{yx}\alpha^2}{\eta} 
\end{split}
\end{align*}
with a flexible parameter $\gamma>0$ decided later.
\end{lemma}
\begin{proof}
First note that
\begin{align}\label{eq:firsty+}
\begin{split}
\E[\|y_+-y_{(x_+)}^\ast\|^2]=&\E[\|y_+-y^\ast\|^2]+\E[\|y_{(x_+)}^\ast-y^\ast\|^2]+2\E[\langle y_+-y^\ast,y^\ast-y_{(x_+)}^\ast\rangle].
\end{split}
\end{align}
In \cref{eq:firsty+}, 
we bound the second term using \Cref{lemma1} and \cref{lemma4secondterm} as 
\begin{align*}
\begin{split}
\E[\|y_{(x_+)}^\ast-y^\ast\|^2]\leq&L_y^2\E\Big[\Big\|\frac{1}{m}\sum_{i=1}^{m}\alpha^i\sum_{\upsilon=0}^{\tau_i-1}\big(\widetilde{h}_i^D(x_{v}^i,y_+)-\widetilde{h}_i^D(x_0^i,y_+)+\widetilde{h}^D(x)-\widetilde{h}^I(x)\big)\Big\|^2\Big]
\\\leq&4\alpha^2L_y^2\E\Big[\Big\|\frac{1}{m}\sum_{i=1}^{m}\frac{1}{\tau_i}\sum_{\upsilon=0}^{\tau_i-1}\big(\Bar{h}_i^D(x_{v}^i,y_+)-\bar{h}^I(x)\big)\Big\|^2\Big] \nonumber
\\&+8\alpha^2L_y^2\sigma_h^2+8\alpha^2L_y^2\sigma_f^2+4\alpha^2L_y^2M^2,
\end{split}
\end{align*}
and for the third term, we have
\begin{align}\label{eq:y+yexp}
\E[\langle y_+-y^\ast,y^\ast-y_{(x_+)}^\ast\rangle]=&-\E[\langle y_+-y^\ast,\nabla y^\ast(x_+-x)\rangle]\nonumber
\\&-\E[\langle y_+-y^\ast,y_{(x_+)}^\ast-y^\ast-\nabla y^\ast(x_+-x)\rangle].
\end{align}
For the first term on the RHS of 
% the above equality, 
the above \cref{eq:y+yexp}, 
we have
\begin{align}\label{eq:ey+firstone}
-\E[&\langle y_+-y^\ast,\nabla y^\ast(x_+-x)\rangle]\nonumber
\\=&-\E\Big[\Big\langle y_+-y^\ast,\E\Big[\frac{1}{m}\nabla y^\ast\sum_{i=1}^{m}\alpha^i\sum_{\upsilon=0}^{\tau_i-1}\big(\widetilde{h}_i^D(x_{v}^i,y_+)-\widetilde{h}_i^D(x_0^i,y_+)+\widetilde{h}^D(x)-\widetilde{h}^I(x)\big)\Big|x,y_+\Big]\Big\rangle\Big]\nonumber
\\=&-\E\Big[\Big\langle y_+-y^\ast,\E\Big[\frac{1}{m}\nabla y^\ast\sum_{i=1}^{m}\alpha^i\sum_{\upsilon=0}^{\tau_i-1}\big(\widetilde{h}_i^D(x_{v}^i,y_+)-\bar{h}^I(x)\big)\Big|x_\upsilon^i\Big]\Big\rangle\Big]\nonumber
\\=&-\E\Big[\Big\langle y_+-y^\ast,\frac{1}{m}\nabla y^\ast\sum_{i=1}^{m}\alpha^i\sum_{\upsilon=0}^{\tau_i-1}\big(\Bar{h}_i^D(x_{v}^i,y_+)-\bar{h}^I(x)\big)\Big\rangle\Big]\nonumber
\\\overset{(i)}\leq&\E\Big[\|y_+-y^\ast\|\Big\|\frac{1}{m}\nabla y^\ast\sum_{i=1}^{m}\alpha^i\sum_{\upsilon=0}^{\tau_i-1}\big(\Bar{h}_i^D(x_{v}^i,y_+)-\bar{h}^I(x)\big)\Big\|\Big]\nonumber
\\\overset{(ii)}\leq& L_y\E\Big[\|y_+-y^\ast\|\Big\|\frac{1}{m}\sum_{i=1}^{m}\alpha^i\sum_{\upsilon=0}^{\tau_i-1}\big(\Bar{h}_i^D(x_{v}^i,y_+)-\bar{h}^I(x)\big)\Big\|\Big]\nonumber
\\\overset{(iii)}\leq&2\gamma\E[\|y_+-y^\ast\|^2]+\frac{L_y^2\alpha^2}{8\gamma}\E\Big[\Big\|\frac{1}{m}\sum_{i=1}^{m}\frac{1}{\tau_i}\sum_{\upsilon=0}^{\tau_i-1}\big(\Bar{h}_i^D(x_{v}^i,y_+)-\bar{h}^I(x)\big)\Big\|^2\Big]
\end{align}
where 
% the first inequality follows from the Cauchy–Schwarz inequality, the second inequality follows from \cref{lemma1}, and the last inequality follows from Young's inequality such that $ab\leq2\gamma a^2+\frac{b^2}{2\gamma}$. 
$(i)$ follows from the Cauchy–Schwarz inequality, $(ii)$ follows from \Cref{lemma1}, and $(iii)$ follows from Young's inequality that $ab\leq2\gamma a^2+\frac{b^2}{2\gamma}$. 
For the second term of RHS of \cref{eq:y+yexp}, we have
\begin{align}\label{eq:y+secondone}
-\E&[\langle y_+-y^\ast,y_{(x_+)}^\ast-y^\ast-\nabla y^\ast(x_+-x)\rangle]\nonumber
\\\leq&\E[\|y_+-y^\ast\|\|y_{(x_+)}^\ast-y^\ast-\nabla y^\ast(x_+-x)\|]\nonumber
\\\overset{(i)}{\leq}&\frac{L_{yx}}{2}\E[\|y_+-y^\ast\|\|x_+-x\|^2]\nonumber
\\\overset{(ii)}{\leq}&\frac{\eta L_{yx}}{4}\E[\|y_+-y^\ast\|^2\|x_+-x\|^2]+\frac{L_{yx}}{4\eta}\E[\|x_+-x\|^2]\nonumber
\\\leq&\frac{\eta L_{yx}}{4}\frac{1}{m}\sum_{i=1}^{m}\frac{\alpha^2}{\tau_i}\sum_{\upsilon=0}^{\tau_i-1}\E\big[\|y_+-y^\ast\|^2\E[\|\widetilde{h}_i^D(x_{\upsilon}^i,y_+)-\widetilde{h}_i^D(x_0^i,y_+)+\widetilde{h}^D(x)-\widetilde{h}^I(x)\|^2|x,y_+]\big] \nonumber
\\&+\frac{L_{yx}}{4\eta}\E[\|x_+-x\|^2]\nonumber
\\\overset{(iii)}{\leq}&\frac{\eta L_{yx}D_h^2\alpha^2}{4}\E[\|y_+-y^\ast\|^2]+\frac{L_{yx}\alpha^2}{\eta}\E\Big[\Big\|\frac{1}{m}\sum_{i=1}^{m}\frac{1}{\tau_i}\sum_{\upsilon=0}^{\tau_i-1}\big(\Bar{h}_i^D(x_{v}^i,y_+)-\bar{h}^I(x)\big)\Big\|^2\Big]\nonumber
\\&+\frac{L_{yx}\alpha^2}{\eta}(2\sigma_h^2+2\sigma_f^2+M^2)
\end{align}
where 
$(i)$ follows from the decent lemma by the smoothness of $y^*(\cdot)$, $(ii)$ follows from the Young's inequality, and $(iii)$ follows from \Cref{propsfrvcc} and \cref{lemma4secondterm}. 
Substituting \cref{eq:ey+firstone} and \cref{eq:y+secondone} into \cref{eq:y+yexp}, and using \cref{eq:firsty+}, we complete the proof.  
\end{proof}

\begin{lemma}[Restatement of \Cref{client_dripfscs}]\label{lemma6}
 Suppose Assumptions~ \ref{assum:muconvex}-\ref{bounded_variance} are satisfied. Set $\lambda\leq\min\{10, \frac{1}{L_g}\}$, $\alpha^i = \frac{\alpha}{\tau_i}\; and\; \beta^i = \frac{\beta}{\tau_i},\tau_i\geq1\ where\ \alpha\leq\frac{1}{324M_f^2+6M_f}\leq\frac{1}{6M_f},\; \beta\leq\min\{1, \lambda, \frac{1}{6L_{g}}\}$  $\forall i \in S$. Recall the definitions of $y^{\ast}=\argmin_{y}g(x,y)$, $\Bar{h}(x)={\E}[\widetilde{h}(x)|x,y_+]$. Then, we have
\begin{align}
\E[\|x_{\upsilon}^i-x\|^2]\leq&18\tau_i^2(\alpha^i)^2\Big[3M_f^2\E[\|y_+-y^\ast\|^2]+3\E[\|\nabla f(x)\|^2]+\frac{4L_g^2M^2(1-\lambda\mu)^{2N+2}}{\mu^2}\nonumber
\\&+[4\lambda^2L_g^2M^2\alpha_1(N)+4\lambda^2L_f^2L_g^2\alpha_3(N)]\E[\|y-y^\ast\|^2]\nonumber
\\&+400\lambda^2\beta^2L_g^2M^2\rho^2\sigma_g^2\alpha_2(N)+\frac{200\lambda\beta^2L_f^2L_g^2N(N+1)\sigma_g^2}{\mu}+3\sigma_h^2+6\sigma_f^2\Big]
\end{align}
where $\alpha_1(N), \alpha_2(N), \alpha_3(N)$ are defined in \Cref{propsfrvcc}.
\end{lemma}
\begin{proof}
% Notice that all the notations can be referred in \cref{notations}. The result holds for $\tau_i = 1$, when $\tau_i>1$, we firstly define
The result holds for  $\tau_i = 1$ according to line 2 in \Cref{algorithm4} where $x_0^i=x$, and hence we consider the case when $\tau_i > 1$. Based on the notations in~\cref{notations}, we define 
\begin{align}
\begin{split}
v_{\upsilon}^i:=&\bar{h}(x)-\bar{\nabla}f(x,y_+),
\\\omega_{\upsilon}^i:=&\nabla_xF_i(x_{\upsilon}^i,y_+;\xi_{i,\upsilon})-\nabla_xf_i(x_{\upsilon}^i,y_+)+\nabla_xf_i(x,y_+)
\\&-\nabla_xF_i(x,y_+;\xi_{i,\upsilon})+\widetilde{h}(x)-\bar{h}(x),
\\z_{\upsilon}^i:=&\nabla_xf_i(x_{\upsilon}^i,y_+)-\nabla_xf_i(x,y_+)+\bar{\nabla}f(x,y_+)-\nabla f(x)+\nabla f(x).
\end{split}
\end{align}
Based on 
\Cref{algorithm4},
for each $i \in S$, and $\forall \upsilon\in{0,...,\tau_i-1}$, we have,
\begin{align}\label{x updates}
x_{\upsilon+1}^i-x=x_{\upsilon}^i-x-\alpha^i(v_{\upsilon}^i+\omega_{\upsilon}^i+z_{\upsilon}^i).
\end{align}
Based on \Cref{lemma1} and \Cref{propsfrvcc}, we bound $v_{\upsilon}^i$, $\omega_{\upsilon}^i$, and $z_{\upsilon}^i$ as 
\begin{align}\label{vwz}
\begin{split}
\E[\|v_{\upsilon}^i\|^2]\leq&\frac{4L_g^2M^2(1-\lambda\mu)^{2N+2}}{\mu^2}+[4\lambda^2L_g^2M^2\alpha_1(N)+4\lambda^2L_g^2L_f^2\alpha_3(N)]\E[\|y-y^\ast\|^2]
\\&+400\lambda^2\beta^2L_g^2M^2\rho^2\sigma_g^2\alpha_2(N)+\frac{200\lambda\beta^2\sigma_g^2L_f^2L_g^2N(N+1)}{\mu},
\\\E[\|\omega_{\upsilon}^i\|^2]\leq&3\E[\|\nabla_xF_i(x_{\upsilon}^i,y_+;\xi_{i,\upsilon})-\nabla_xf_i(x_{\upsilon}^i,y_+)\|^2\\&+\|\nabla_xf_i(x,y_+)-\nabla_xF_i(x,y_+;\xi_{i,\upsilon})\|^2+\|\widetilde{h}(x)-\bar{h}(x)\|^2]\\
\leq&6\sigma_f^2+3\sigma_h^2,
\\\E[\|z_{\upsilon}^i\|^2]\leq&3\E[\|\nabla_xf_i(x_{\upsilon}^i,y_+)-\nabla_xf_i(x,y_+)\|^2
\\&+\|\bar{\nabla}f(x,y_+)-\nabla f(x)\|^2+\E\|\nabla f(x)\|^2]
\\\leq&3(M_f^2\E[\|x_{\upsilon}^i-x\|^2]+M_f^2\E[\|y_+-y^\ast\|^2]+\E[\|\nabla f(x)\|^2]).
\end{split}
\end{align}
Now, we bound RHS of \cref{x updates} as
\begin{align}\label{xupdatebound}
\E[\|x_{\upsilon}^i-&x-\alpha^i(v_{\upsilon}^i+\omega_{\upsilon}^i+z_{\upsilon}^i)\|^2]\nonumber
\\\overset{(i)}\leq&(1+\frac{1}{2\tau_i-1})\E[\|x_{\upsilon}^i-x\|^2]+2\tau_i\E[\|\alpha^i(v_{\upsilon}^i+\omega_{\upsilon}^i+z_{\upsilon}^i)\|^2]\nonumber
\\\overset{(ii)}\leq&(1+\frac{1}{2\tau_i-1})\E[\|x_{\upsilon}^i-x\|^2]+6\tau_i(\alpha^i)^2\E[\|v_{\upsilon}^i\|^2+\|\omega_{\upsilon}^i\|^2+\|z_{\upsilon}^i\|^2]\nonumber
\\\overset{(iii)}\leq&(1+\frac{1}{2\tau_i-1}+18\tau_i(\alpha^i)^2M_f^2)\E[\|x_{\upsilon}^i-x\|^2]\nonumber
\\&+6\tau_i(\alpha^i)^2\Big[3M_f^2\E[\|y_+-y^\ast\|^2]+3\E[\|\nabla f(x)\|^2]+\frac{4L_g^2M^2(1-\lambda\mu)^{2N+2}}{\mu^2}\nonumber
\\&+[4\lambda^2L_g^2M^2\alpha_1(N)+4\lambda^2L_g^2L_f^2\alpha_3(N)]\E[\|y-y^\ast\|^2]+400\lambda^2\beta^2L_g^2M^2\rho^2\sigma_g^2\alpha_2(N)\nonumber
\\&+\frac{200\lambda\beta^2\sigma_g^2L_f^2L_g^2N(N+1)}{\mu}+6\sigma_f^2+3\sigma_h^2\Big]
\end{align}
where 
% the first inequality
$(i)$ 
follows from $\|x+y\|^2\leq(1+c)\|x\|^2+(1+\frac{1}{c})\|y\|^2$,
% the second inequality 
$(ii)$ follows from the Young's inequality, 
% the last inequality
and $(iii)$ uses \cref{vwz}.
Substituting \cref{xupdatebound} into \cref{x updates} yields 
\begin{align}\label{newxupdatebound}
\E[\|x_{\upsilon+1}^i-x\|^2]\leq&(1+\frac{1}{2\tau_i-1}+18\tau_i(\alpha^i)^2M_f^2)\E[\|x_{\upsilon}^i-x\|^2]\nonumber
\\&+6\tau_i(\alpha^i)^2\Big[3M_f^2\E[\|y_+-y^\ast\|^2]+3\E[\|\nabla f(x)\|^2]+\frac{4L_g^2M^2(1-\lambda\mu)^{2N+2}}{\mu^2}\nonumber
\\&+[4\lambda^2L_g^2M^2\alpha_1(N)+4\lambda^2L_g^2L_f^2\alpha_3(N)]\E[\|y-y^\ast\|^2]+400\lambda^2\beta^2L_g^2M^2\rho^2\sigma_g^2\alpha_2(N)\nonumber
\\&+\frac{200\lambda\beta^2\sigma_g^2L_f^2L_g^2N(N+1)}{\mu}+3\sigma_h^2+6\sigma_f^2\Big]\nonumber
\\\leq&(1+\frac{1}{\tau_i-1})\E[\|x_{\upsilon}^i-x\|^2]\nonumber
\\&+6\tau_i(\alpha^i)^2\Big[3M_f^2\E[\|y_+-y^\ast\|^2]+3\E[\|\nabla f(x)\|^2]+\frac{4L_g^2M^2(1-\lambda\mu)^{2N+2}}{\mu^2}\nonumber
\\&+[4\lambda^2L_g^2M^2\alpha_1(N)+4\lambda^2L_g^2L_f^2\alpha_3(N)]\E[\|y-y^\ast\|^2]+400\lambda^2\beta^2L_g^2M^2\rho^2\sigma_g^2\alpha_2(N)\nonumber
\\&+\frac{200\lambda\beta^2\sigma_g^2L_f^2L_g^2N(N+1)}{\mu}+3\sigma_h^2+6\sigma_f^2\Big],
\end{align}
where the last inequality follows because 
$\alpha^i\leq1/(6M_f\tau_i)$.
For all $\tau_i>1$, we have
\begin{align}\label{eq:taurelationship}
\sum_{j=0}^{\upsilon-1}(1+\frac{1}{\tau_i-1})^j=\frac{(1+\frac{1}{\tau_i-1})^\upsilon-1}{(1+\frac{1}{\tau_i-1})-1}\leq\tau_i(1+\frac{1}{\tau_i})^\upsilon\leq\tau_i(1+\frac{1}{\tau_i})^{\tau_i}\leq\exp(1)\tau_i<3\tau_i.
\end{align}
Finally, telescoping \cref{newxupdatebound} and using \cref{eq:taurelationship}, we have 
\begin{align*}
\E[\|x_{\upsilon}^i-x\|^2]\leq&18\tau_i^2(\alpha^i)^2\Big[3M_f^2\E[\|y_+-y^\ast\|^2]+3\E[\|\nabla f(x)\|^2]+\frac{4L_g^2M^2(1-\lambda\mu)^{2N+2}}{\mu^2}\nonumber
\\&+[4\lambda^2L_g^2M^2\alpha_1(N)+4\lambda^2L_g^2L_f^2\alpha_3(N)]\E[\|y-y^\ast\|^2]+400\lambda^2\beta^2L_g^2M^2\rho^2\sigma_g^2\alpha_2(N)\nonumber
\\&+\frac{200\lambda\beta^2\sigma_g^2L_f^2L_g^2N(N+1)}{\mu}+3\sigma_h^2+6\sigma_f^2\Big].
\end{align*}
Then, the proof is complete. 
\end{proof}

\subsection{Proof of~\Cref{Theorem2-mainB}}\label{proofoftheorem2}
\begin{theorem}[Restatement of \Cref{Theorem2-mainB}]\label{Theorem2}
Suppose \Cref{assum:muconvex}-\ref{bounded_variance} hold. Further set $\lambda\leq\min\{10, \frac{1}{L_g}\}$, $\alpha_k^i=\frac{\alpha_k}{\tau_i}$ an $\beta_k^i=\frac{\beta_k}{\tau_i}$ for all $i\in S$. Define $\bar{\beta}=\Big(\frac{M_fL_y}{2}\bar{\alpha}_2+11M_fL_y+\eta L_{yx}D_h^2\bar{\alpha}_2+\frac{(6+\frac{\bar{\alpha}_2}{3})(N+1)\lambda L_yL_g^2}{M_f}\Big(\frac{328\rho^2M^2}{\mu^3}+\frac{6L_f^2}{\mu}\Big)\Big)\frac{1}{\mu}$, $\bar{\alpha}_1=\frac{1}{8L_f^\prime+16M_fL_y+\frac{8M_fL_{yx}}{\eta L_y}}$, $\bar{\alpha}_2=\frac{1}{324M_f^2+6M_f}$, $\bar{\alpha}_3=\frac{N\min\{1,\lambda,\frac{1}{6L_g}\}}{2\bar{\beta}}$, and $\sigma_h^2=\frac{\lambda(N+1)L_g^2M^2}{\mu}$ , where $L_f^\prime=L_{f}+\frac{L_{g}(L_{f}+M_f)}{\mu}+\frac{M}{\mu}({\rho}+\frac{L_{g}{\rho}}{\mu})$, $M_f=L_{f}+\frac{L_{g}L_{f}}{\mu}+\frac{M}{\mu}({\rho}+\frac{L_{g}{\rho}}{\mu})$, $L_y=\frac{L_{g}}{\mu}$, and $L_{yx}=\frac{{\rho}+{\rho}L_y}{\mu}+\frac{L_{g}({\rho}+{\rho}L_y)}{\mu^2}$. Besides, define 
\begin{align*}
c_0=&2L_f^\prime+4M_fL_y+\frac{2L_{yx}M_f}{\eta L_y},
\\c_1=&\frac{1}{4}+4L_f^\prime+8M_fL_y+\frac{4L_{yx}M_f}{\eta L_y},
\\c_2=&\frac{1}{2}+4L_f^\prime+8M_fL_y+\frac{4L_{yx}M_f}{\eta L_y},
\\ c_3=&\frac{25M_f}{L_y}\Big[1+(12+\frac{2\alpha_k}{3})(\frac{2\bar{\alpha}_2\lambda^2L_g^2M^2\rho^2L_y}{NM_f}\alpha_2(N)+\frac{\bar{\alpha}_2\lambda L_f^2L_g^2(N+1)L_y}{\mu M_f})+\frac{M_fL_y\bar{\alpha}_2^2}{4}
\\&\quad\qquad+\frac{11\bar{\alpha}_2M_fL_y}{2}+\frac{\eta L_{yx}D_h^2\bar{\alpha}_2^2}{2}\Big]\frac{\bar{\beta}^2}{N}, 
\end{align*}
where $\eta=\frac{M_f}{L_y}$ and $D_h^2=8M^2+\frac{4\lambda(N+1)L_g^2M^2}{\mu}$.
Choose parameters such that $\alpha_k=\min\{\bar{\alpha}_1,\bar{\alpha}_2,\bar{\alpha}_3,\frac{\bar{\alpha}}{\sqrt{K}}\}$, $\beta_k\in[\max\{\frac{\bar{\beta}\alpha_k}{N},\frac{\lambda}{10}\},\min\{1,\lambda,\frac{1}{6L_g}\}]$, where $\bar{\alpha}$ is a parameter that can be tuned. Then we have
\begin{align*}
\frac{1}{K}\sum_{k=0}^{K-1}\E[\|\nabla f(x_k)\|^2]=\mathcal{O}\Big(\frac{1}{\min\{\bar{\alpha}_1,\bar{\alpha}_2,\bar{\alpha}_3\}{K}}+\frac{1}{\bar\alpha\sqrt{K}}+\frac{\bar{\alpha}\max\{c_0,c_1\sigma_h^2,c_2,c_3\}}{\sqrt{K}}+(1-\lambda\mu)^{2N}).
\end{align*}
\end{theorem}
\begin{proof}
Now, we define a Lyapunov function
\begin{align*}
{\W}_k:=f(x_k,y_{(x_k)}^\ast)+\frac{M_f}{L_y}\|y_k-y_{(x_k)}^\ast\|^2.
\end{align*}
Motivated by 
% \citep{chen2021closing},
\cite{chen2021closing},
we bound the difference between two Lyapunov functions as 
\begin{align}\label{eq:wkplusw1sc}
{\W}_{k+1}-{\W}_k = f(x_{k+1},y_{(x_{k+1})}^\ast)-f(x_k,y_{(x_k)}^\ast)+\frac{M_f}{L_y}(\|y_{k+1}-y_{(x_{k+1})}^\ast\|^2-\|y_k-y_{(x_k)}^\ast\|^2).
\end{align}
% From our assumptions, 
Recall that   $\alpha_{k}^i = \frac{\alpha_k}{\tau_i}$, $\beta_k^i=\frac{\beta_k}{\tau_i}, \forall i\in S$. 
Using 
such stepsizes 
and substituting~\Cref{lemma9} into \cref{eq:wkplusw1sc}, we have 
\begin{align}\label{theorem2begins}
\E[{\W}_{k+1}]&-\E[{\W}_k]
\nonumber
\\\leq&-\frac{\alpha_k}{2}\E[\|\nabla f(x_k)]+4\alpha_k^2\sigma_h^2L_f^\prime+4\alpha_k^2\sigma_f^2L_f^\prime+2\alpha_k^2M^2L_f^\prime\nonumber
\\&-\frac{\alpha_k}{2}(1-4\alpha_k L_f^\prime)\E\Big[\Big\|\frac{1}{m}\sum_{i=1}^{m}\frac{1}{\tau_i}\sum_{\upsilon=0}^{\tau_i-1}\big(\bar{h}_i^D(x_{k,\upsilon
}^i,y_+)-\bar{h}^I(x)\big)\Big\|^2\Big]\nonumber
\\&+\frac{3\alpha_k}{2}\Big[\big(4\lambda^2L_g^2M^2\alpha_1(N)+4\lambda^2L_f^2L_g^2\alpha_3(N)\big)\E[\|y_k-y_{(x_k)}^\ast\|^2]+\frac{4L_g^2M^2(1-\lambda\mu)^{2N+2}}{\mu^2}\nonumber\\
&+400\lambda^2\beta_k^2L_g^2M^2\sigma_g^2\rho^2\alpha_2(N)+\frac{200\lambda\beta_k^2\sigma_g^2L_f^2L_g^2N(N+1)}{\mu}
+\frac{M_f^2}{m}\sum_{i=1}^{m}\frac{1}{\tau_i}\sum_{\upsilon=0}^{\tau_i-1}\E[\|x_{k,\upsilon}^i-x_k\|^2]\nonumber
\\&+M_f^2\E[\|y_{k+1}-y_{(x_k)}^\ast\|^2]\Big]+\frac{M_f}{L_y}\E[\|y_{k+1}-y_{(x_{k+1})}^\ast\|^2-\|y_k-y_{(x_k)}^\ast\|^2].
\end{align}
Then, following \Cref{lemma5}, \cref{theorem2begins} can be rewritten as
\begin{subequations}
\begin{align}
% \E[f(x_{k+1})]-\E[f(x_k)]
\E[{\W}_{k+1}]&-\E[{\W}_k] \nonumber
\\
\leq&4\alpha_k^2\sigma_h^2L_f^\prime+4\alpha_k^2\sigma_f^2L_f^\prime+2\alpha_k^2M^2L_f^\prime+\frac{M_f}{L_y}b_3(\alpha_k)(2\sigma_h^2+2\sigma_f^2+M^2)\nonumber
\\&+\frac{300\alpha_k\lambda\beta_k^2L_f^2L_g^2N(N+1)\sigma_g^2}{\mu}+\frac{6\alpha_kL_g^2M^2(1-\lambda\mu)^{2N+2}}{\mu^2}
+600\alpha_k\lambda^2\beta_k^2 L_g^2M^2\rho^2\sigma_g^2\alpha_2(N)\notag
\\&-\frac{\alpha_k}{2}\E[\|\nabla f(x_k)\|^2]+\frac{3\alpha_kM_f^2}{2m}\sum_{i=1}^{m}\frac{1}{\tau_i}\sum_{\upsilon=0}^{\tau_i-1}\E[\|x_{k,\upsilon}^i-x_k\|^2]\label{apart}
\\&-\big(\frac{\alpha_k}{2}-2\alpha_k^2L_f^\prime-\frac{M_f}{L_y}b_1(\alpha_k)\big)\E\Big[\Big\|\frac{1}{m}\sum_{i=1}^{m}\frac{1}{\tau_i}\sum_{\upsilon=0}^{\tau_i-1}\big(\bar{h}_i(x_{k,\upsilon}^i,y_+)-\bar{h}(x)\big)\Big\|^2\Big]\label{bpart}
\\&+\big(\frac{3\alpha_kM_f^2}{2}+\frac{M_f}{L_y}b_2(\alpha_k)\big)\E[\|y_{k+1}-y_{(x_k)}^\ast\|^2] \notag
\\&+\big(6\alpha_k\lambda^2L_g^2M^2\alpha_1(N)+6\alpha_k\lambda^2L_f^2L_g^2\alpha_3(N)-\frac{M_f}{L_y}\big)\E[\|y_k-y_{(x_k)}^\ast\|^2]\label{cpart}.
\end{align}
\end{subequations}
Set $\gamma=M_fL_y\alpha_k$. 
Then according to the selections in \Cref{Theorem2} that $\alpha_k\leq\frac{1}{324M_f^2+6M_f}$, $\alpha_k\leq\frac{1}{8L_f^\prime+16M_fL_y+\frac{8M_fL_{yx}}{\eta L_y}}$, and substituting \cref{mitra} in \cref{cpart}, the following results can be obtained.
\begin{align}
\eqref{apart}&\leq-\frac{\alpha_k}{4}\E[\|\nabla f(x_k)\|^2]+\frac{\alpha_k^2}{4}\sigma_h^2+\frac{\alpha_k^2}{2}\sigma_f^2+\frac{\alpha_k^2L_g^2M^2(1-\lambda\mu)^{2N+2}}{3\mu^2}\nonumber
\\+&\frac{100\alpha_k^2\lambda^2\beta_k^2L_g^2M^2\rho^2\sigma_g^2}{3}\alpha_2(N)+\frac{50\alpha_k^2\lambda\beta_k^2L_f^2L_g^2N(N+1)\sigma_g^2}{3\mu}+\frac{25M_f}{L_y}(\frac{M_fL_y}{4}\alpha_k^2)N\beta_k^2\sigma_g^2\label{aresultintheorem2}
\\+&\frac{M_f}{L_y}\Big[(\frac{M_fL_y}{4}\alpha_k^2)(1-\frac{\beta_k\mu}{2})^N+\frac{\alpha_k^2\lambda^2L_yL_g^2M^2}{3M_f}\alpha_1(N)+\frac{\alpha_k^2\lambda^2L_yL_f^2L_g^2}{3M_f}\alpha_3(N)\Big]\E\|y_k-y_{(x_k)}^\ast\|^2\nonumber,
\\\text{In }\eqref{bpart}&\text{, we have }\frac{\alpha_k}{2}-2\alpha_k^2L_f^\prime-\frac{M_f}{L_y}b_1(\alpha_k)\geq0\label{bresultintheorem2},
\\\eqref{cpart}\leq&\frac{25M_f}{L_y}\big(\frac{3\alpha_kM_fL_y}{2}+b_2(\alpha_k)\big)N\beta_k^2\sigma_g^2+\frac{M_f}{L_y}\Big[\big(\frac{3\alpha_kM_fL_y}{2}+b_2(\alpha_k)\big)(1-\frac{\beta_k\mu}{2})^N\nonumber
\\&+\frac{6\alpha_k\lambda^2L_g^2L_yM^2\alpha_1(N)}{M_f}+\frac{6\alpha_k\lambda^2L_f^2L_yL_g^2}{M_f}\alpha_3(N)-1\Big]\E[\|y_k-y_{(x_k)}^\ast\|^2]\label{cresultintheorem2}.
\end{align}
Then, adding \cref{aresultintheorem2}, \cref{bresultintheorem2} and \cref{cresultintheorem2} together, we have
% {\small
\begin{align}\label{addingabc}
\E[{\W}&_{k+1}]-\E[{\W}_k] \nonumber
\\\leq&-\frac{\alpha_k}{4}\E[\|\nabla f(x_k^\ast)\|^2]+\frac{\alpha_k^2\sigma_f^2}{2}+\frac{\alpha_k^2\sigma_h^2}{4}+\frac{50\alpha_k\lambda\beta_k^2L_f^2L_g^2N(N+1)\sigma_g^2}{\mu}(6+\frac{\alpha_k}{3})\nonumber
\\&+2\alpha_k^2L_f^\prime(2\sigma_f^2+2\sigma_h^2+M^2)+100\alpha_k\lambda^2\beta_k^2L_g^2M^2\rho^2\sigma_g^2(6+\frac{\alpha_k}{3})\alpha_2(N)\nonumber
\\&+\frac{M_f(2\sigma_f^2+2\sigma_h^2+M^2)}{L_y}b_3(\alpha_k)\nonumber
\\&+\frac{2\alpha_kL_g^2M^2(1-\lambda\mu)^{2N+2}}{\mu^2}(3+\frac{\alpha_k}{6})+\frac{25M_f}{L_y}(\frac{M_fL_y}{4}\alpha_k^2+\frac{3M_fL_y\alpha_k}{2}+b_2(\alpha_k))N\beta_k^2\sigma_g^2\nonumber
\\&+\frac{M_f}{L_y}\Big(\big(\frac{M_fL_y}{4}\alpha_k^2+\frac{3M_fL_y\alpha_k}{2}+b_2(\alpha_k)\big)(1-\frac{\beta_k\mu}{2})^N-1+\frac{2\alpha_k\lambda^2L_yL_g^2M^2}{M_f}\alpha_1(N)(3+\frac{\alpha_k}{6})\nonumber
\\&\hspace{1cm}+\frac{2\alpha_k\lambda^2L_yL_f^2L_g^2}{M_f}\alpha_3(N)(3+\frac{\alpha_k}{6})\Big)\E[\|y_k-y_{(x_k)}^\ast\|^2]\nonumber
\\\leq&-\frac{\alpha_k}{4}\E[\|\nabla f(x_k^\ast)\|^2]+\frac{\alpha_k^2\sigma_f^2}{2}+\frac{\alpha_k^2\sigma_h^2}{4}+\frac{50\alpha_k\lambda\beta_k^2L_f^2L_g^2N(N+1)\sigma_g^2}{\mu}(6+\frac{\alpha_k}{3})\nonumber
\\&+2\alpha_k^2L_f^\prime(2\sigma_f^2+2\sigma_h^2+M^2)+100\alpha_k\lambda^2\beta_k^2L_g^2M^2\rho^2\sigma_g^2(6+\frac{\alpha_k}{3})\alpha_2(N) \nonumber
\\&+\frac{M_f(2\sigma_f^2+2\sigma_h^2+M^2)}{L_y}b_3(\alpha_k)\nonumber
\\&+\frac{2\alpha_kL_g^2M^2(1-\lambda\mu)^{2N+2}}{\mu^2}(3+\frac{\alpha_k}{6})+\frac{25M_f}{L_y}(\frac{M_fL_y}{4}\alpha_k^2+\frac{3M_fL_y\alpha_k}{2}+b_2(\alpha_k))N\beta_k^2\sigma_g^2\nonumber
\\&+\frac{M_f}{L_y}\Big[\big(\frac{M_fL_y}{4}\alpha_k^2+\frac{3M_fL_y\alpha_k}{2}+b_2(\alpha_k)\big)(1-\frac{\beta_k\mu}{2})^N-1\nonumber
\\&\hspace{1cm}+\frac{2\alpha_k\lambda^2L_yL_g^2M^2}{M_f}(6+\frac{\alpha_k}{3})(N+1)(1-\frac{\beta_k\mu}{2})^N\Big(\frac{2\rho^2}{\lambda\mu^3}+\frac{80\rho^2}{\lambda\mu^3}\Big)\nonumber
\\&\hspace{1cm}+\frac{3\alpha_k\lambda L_yL_f^2L_g^2}{\mu M_f}(6+\frac{\alpha_k}{3})(N+1)(1-\frac{\beta_k\mu}{2})^N\Big]\E[\|y_k-y_{(x_k)}^\ast\|^2],
\end{align}
where in the last inequality, recalling from \Cref{lemma5} and \Cref{propsfrvcc} that $b_2(\alpha):=1+4\gamma+\frac{\eta L_{yx}D_h^2\alpha_k^2}{2}$, $\alpha_1(N)=4(N+1)(1-\frac{\beta_k\mu}{2})^N\Big(\frac{\rho^2}{\lambda\mu^3}+\frac{4\rho^2}{\beta_k\mu^3}\Big),$ $\alpha_3(N)=3(N+1)\frac{(1-\beta_k\mu)^N}{\lambda\mu}$, we choose $\beta_k\geq\frac{\lambda}{10}$. Based on the parameters selections in
\Cref{Theorem2} that $\beta_k\geq\Big(\frac{M_fL_y}{2}\alpha_k+11M_fL_y+\eta L_{yx}D_h^2\alpha_k+\frac{(6+\frac{\alpha_k}{3})(N+1)\lambda L_yL_g^2}{M_f}\Big(\frac{328\rho^2M^2}{\mu^3}+\frac{6L_f^2}{\mu}\Big)\Big)\frac{\alpha_k}{\mu N}$ and $\gamma=M_fL_y\alpha_k$, we have
\begin{align}\label{firstinequalityintheorem2}
&\Rightarrow\exp\Big(\frac{M_fL_y}{4}\alpha_k^2+\frac{11M_fL_y\alpha_k}{2}+\frac{\eta L_{yx}D_h^2\alpha_k^2}{2}+\frac{\alpha_k(6+\frac{\alpha_k}{3})(N+1)\lambda^2L_yL_g^2}{M_f}\times\nonumber
\\&\hspace{1.5cm}\Big(\frac{164\rho^2M^2}{\lambda\mu^3}+\frac{6L_f^2}{\lambda\mu}\Big)\Big)\exp(-\frac{N\beta_k\mu}{2})\leq1\nonumber
\\&\Rightarrow\big(\frac{M_fL_y}{4}\alpha_k^2+\frac{3M_fL_y\alpha_k}{2}+b_2(\alpha_k)\big)(1-\frac{\beta_k\mu}{2})^N+\frac{2\alpha_k\lambda^2L_yL_g^2M^2}{M_f}\alpha_1(N)(3+\frac{\alpha_k}{6})\nonumber
\\&\hspace{1cm}+\frac{2\alpha_k\lambda^2L_yL_f^2L_g^2}{M_f}\alpha_3(N)(3+\frac{\alpha_k}{6})-1\leq0.
\end{align}
Then plugging \cref{firstinequalityintheorem2} into \cref{addingabc}, we can obtain that
\begin{align}\label{eq:ggminmis}
\E[{\W}&_{k+1}]-\E[{\W}_k]\nonumber
\\\leq&-\frac{\alpha_k}{4}\E[\|\nabla f(x_k)\|^2]+\frac{2\alpha_kL_g^2M^2(1-\lambda\mu)^{2N+2}}{\mu^2}(3+\frac{\alpha_k}{6})+\big(2\alpha_k^2L_f^\prime+\frac{M_f}{L_y}b_3(\alpha_k)\big)M^2\nonumber
\\&+\big(4\alpha_k^2L_f^\prime+\frac{\alpha_k^2}{4}+\frac{2M_f}{L_y}b_3(\alpha_k)\big)\sigma_h^2+\big(4\alpha_k^2L_f^\prime+\frac{\alpha_k^2}{2}+\frac{2M_f}{L_y}b_3(\alpha_k)\big)\sigma_f^2\nonumber
\\&+\frac{25M_f}{L_y}\big(\frac{\alpha_k\lambda^2L_g^2M^2\rho^2L_y}{NM_f}(24+\frac{4\alpha_k}{3})\alpha_2(N)+\frac{\alpha_k\lambda L_f^2L_g^2(N+1)L_y}{\mu M_f}(12+\frac{2\alpha_k}{3})\nonumber
\\&+\frac{M_fL_y\alpha_k^2}{4}+\frac{3\alpha_kM_fL_y}{2}+b_2(\alpha_k)\big)\frac{\bar{\beta}^2}{N}\alpha_k^2\sigma_g^2\nonumber
\\\leq&-\frac{\alpha_k}{4}\E[\|\nabla f(x_k)\|^2]+\frac{2\alpha_kL_g^2M^2(1-\lambda\mu)^{2N+2}}{\mu^2}(3+\frac{\alpha_k}{6})+c_0\alpha_k^2M^2+c_1\alpha_k^2\sigma_h^2+c_2\alpha_k^2\sigma_f^2+c_3\alpha_k^2\sigma_g^2
\end{align}
where $c_0, c_1, c_2, c_3$ are defined in \Cref{Theorem2}. Finally, telescoping \cref{eq:ggminmis}  yields 
\begin{align}\label{estimator2final}
\frac{1}{K}\sum_{k=0}^{K-1}&\E[\|\nabla f(x_k)\|^2] \nonumber
\\\leq&\frac{4{\W}^0}{\sum_{k=0}^{K-1}\alpha_k}+\frac{4c_0\sum_{k=0}^{K-1}\alpha_k^2}{\sum_{k=0}^{K-1}\alpha_k}M^2+\frac{4c_1\sum_{k=0}^{K-1}\alpha_k^2}{\sum_{k=0}^{K-1}\alpha_k}\sigma_h^2+\frac{4c_2\sum_{k=0}^{K-1}\alpha_k^2}{\sum_{k=0}^{K-1}\alpha_k}\sigma_f^2\nonumber
\\&+\frac{4c_3\sum_{k=0}^{K-1}\alpha_k^2}{\sum_{k=0}^{K-1}\alpha_k}\sigma_g^2+\frac{8L_g^2M^2\sum_{k=0}^{K-1}\alpha_k(1-\lambda\mu)^{2N+2}}{\mu^2\sum_{k=0}^{K-1}\alpha_k}(3+\frac{\alpha_k}{6})\nonumber
\\\leq&\frac{4{\W}^0}{\min\{\bar{\alpha}_1,\bar{\alpha}_2,\bar{\alpha}_3,\frac{\bar{\alpha}}{\sqrt{K}}\}K}+\frac{4c_0\bar{\alpha}}{\sqrt{K}}M^2+\frac{4c_1\bar{\alpha}}{\sqrt{K}}\sigma_h^2+\frac{4c_2\bar{\alpha}}{\sqrt{K}}\sigma_f^2+\frac{4c_3\bar{\alpha}}{\sqrt{K}}\sigma_g^2\nonumber
\\&+\frac{8L_g^2M^2(1-\lambda\mu)^{2N+2}}{\mu^2}(3+\frac{\alpha_k}{6})\nonumber
\\\leq&\frac{4{\W}^0}{\min\{\bar{\alpha}_1,\bar{\alpha}_2,\bar{\alpha}_3\}K}+\frac{4{\W}^0}{\bar{\alpha}\sqrt{K}}+\frac{4c_0\bar{\alpha}}{\sqrt{K}}M^2+\frac{4c_1\bar{\alpha}}{\sqrt{K}}\sigma_h^2+\frac{4c_2\bar{\alpha}}{\sqrt{K}}\sigma_f^2+\frac{4c_3\bar{\alpha}}{\sqrt{K}}\sigma_g^2\nonumber
\\&+\frac{8L_g^2M^2(1-\lambda\mu)^{2N+2}}{\mu^2}(3+\frac{\alpha_k}{6})\nonumber
\\=&\mathcal{O}\bigg(\frac{1}{\min\{\bar{\alpha}_1,\bar{\alpha}_2,\bar{\alpha}_3\}{K}}+\frac{1}{\bar\alpha\sqrt{K}}+\frac{\bar{\alpha}\max\{c_0,c_1\sigma_h^2,c_2,c_3\}}{\sqrt{K}}+(1-\lambda\mu)^{2N}\bigg).
\end{align}
The proof is complete.
\end{proof}
\subsection{Proof of \Cref{corollary2}}
\begin{proof}
Let $\eta=\frac{M_f}{L_y}=\mathcal{O}(\kappa_g)$. It follows from \Cref{lemma1} and  \Cref{Theorem2} that
\begin{align}\label{orderofalphaprime}
\begin{split}
L_y=&\mathcal{O}(\kappa_g),\
L_{yx}=\mathcal{O}(\kappa_g^3),\
M_f=\mathcal{O}(\kappa_g^2),\
L_f^\prime=\mathcal{O}(\kappa_g^3),\
\sigma_h^2=\mathcal{O}(N\kappa_g),\
\\\bar{\alpha}_1=&\mathcal{O}(\kappa_g^{-3}),\ 
\bar{\alpha}_2=\mathcal{O}(\kappa_g^{-4}),\ 
\bar{\alpha}_3=\mathcal{O}(N\kappa_g^{-4}+\kappa
_g^{-3}),\ 
\bar{\beta}=\mathcal{O}(\kappa_g^4+N\kappa_g^3),\
\\c_0=&\mathcal{O}(\kappa_g^3),\
c_1=\mathcal{O}(\kappa_g^3),\ c_2=\mathcal{O}(\kappa_g^3),\
c_3=\mathcal{O}\Big(\big(\frac{\kappa_g^8}{N}+N\kappa_g^6\big)\big(\kappa_g+N\kappa_g^{-1}\big)\Big).
\end{split}
\end{align}
Now, if we select $N=\mathcal{O}(\kappa_g)$, $\bar{\alpha}=\mathcal{O}(\kappa_g^{-4})$, we obtain from \cref{estimator2final} that
\begin{align*}
&\bar{\beta}=\mathcal{O}(\kappa_g^4),\; c_3=\mathcal{O}(\kappa_g^8),\; \frac{1}{K}\sum_{k=0}^{K-1}\E[\|\nabla f(x_k)\|^2]=\mathcal{O}(\frac{\kappa_g^4}{K}+\frac{\kappa_g^{4}}{\sqrt{K}}).
\end{align*}
To achieve an $\epsilon$-stationary point, it requires 
$K=\mathcal{O}(\kappa_g^8\epsilon^{-2})$ and  the number of samples in $\xi$ and $\zeta$ are both $\mathcal{O}(\kappa_g^9\epsilon^{-2})$. Then the proof is complete. 
\end{proof}
\end{document}